\documentclass[english]{article}
\usepackage{geometry}
\usepackage[T1]{fontenc}
\usepackage[latin9]{inputenc}
\usepackage{bm}
\usepackage{amsmath}
\usepackage{amssymb} 
\usepackage[unicode=true,
 bookmarks=false,
 breaklinks=false,pdfborder={0 0 1},colorlinks=false]
 {hyperref}
\hypersetup{ 
 colorlinks,citecolor=blue,filecolor=blue,linkcolor=blue,urlcolor=blue}
\usepackage{xcolor,colortbl}
\definecolor{Gray}{gray}{0.85}
\usepackage{enumitem}

\makeatletter

\usepackage{amsthm}
\usepackage{cite}  
\usepackage{comment}
\usepackage{natbib}
\usepackage{booktabs,mathtools}
\usepackage{graphicx}
\usepackage[linesnumbered,ruled,vlined]{algorithm2e}

\usepackage{algorithmic}

\usepackage{multirow}
\usepackage{dsfont}
\usepackage{color}

\definecolor{yxc}{RGB}{255,0,0}
\definecolor{yjc}{RGB}{125,0,0}
\definecolor{ytw}{RGB}{255,69,0}
\definecolor{gen}{RGB}{0,0,200}
\definecolor{lxs}{RGB}{138,43,226}

\definecolor{own_pink}{RGB}{217,25,169}
\definecolor{own_blue}{RGB}{0,100,223}

\newcommand{\Datab}{\mathcal{D}_{\mu}}
\newcommand{\pib}{\mu}

\allowdisplaybreaks

\DeclareMathOperator{\ind}{\mathds{1}}  

\newcommand{\defn}{\coloneqq}

\newcommand{\Var}{\mathsf{Var}}

\newcommand{\cb}{c_{\mathrm{b}}}

\newcommand{\cA}{\mathcal{A}}

\newcommand{\cJ}{\mathcal{J}}

\newcommand{\cS}{{\mathcal{S}}}

\newcommand{\mymid}{\,|\,} 

\newcommand{\LCBQ}{{\sf LCB-Q}\xspace}
\newcommand{\VILCB}{{\sf VI-LCB}\xspace}
\newcommand{\LCB}{{\sf LCB}\xspace}

\newcommand{\LCBQR}{{\sf LCB-Q-Advantage}\xspace}

\newcommand{\sumb}{B}
\newcommand{\re}{\mathsf{ref}}
\newcommand{\adv}{\mathsf{adv}}
\newcommand{\refmu}{\mu^{\mathsf{ref}}} 
\newcommand{\refsg}{\sigma^{\mathsf{ref}}}
\newcommand{\advmu}{\mu^{\mathsf{adv}}}
\newcommand{\advsg}{\sigma^{\mathsf{adv}}}
\newcommand{\nextb}{B^{\mathsf{next}}}

\newcommand{\nnext}{\mathsf{next}}
\newcommand{\epo}{\mathsf{epo}}

\usepackage{scalerel,stackengine}
\stackMath
\newcommand\reallywidehat[1]{%
\savestack{\tmpbox}{\stretchto{%
  \scaleto{%
    \scalerel*[\widthof{\ensuremath{#1}}]{\kern-.6pt\bigwedge\kern-.6pt}%
    {\rule[-\textheight/2]{1ex}{\textheight}}
  }{\textheight}%
}{0.5ex}}%
\stackon[1pt]{#1}{\tmpbox}%
}
\newcommand\reallywidecheck[1]{%
\savestack{\tmpbox}{\stretchto{%
  \scaleto{
    \scalerel*[\widthof{\ensuremath{#1}}]{\kern-.6pt\bigwedge\kern-.6pt}%
    {\rule[-\textheight/2]{1ex}{\textheight}}
  }{\textheight}%
}{0.5ex}}%
\stackon[1pt]{#1}{\scalebox{-1}{\tmpbox}}%
}

\title{Pessimistic Q-Learning for Offline Reinforcement Learning: Towards Optimal Sample Complexity }

\author{
	Laixi Shi\thanks{Department of Electrical and Computer Engineering, Carnegie Mellon University, Pittsburgh, PA 15213, USA.}\\
	CMU\\
	\and
	Gen Li\thanks{Department of Statistics and Data Science, The Wharton School, University of Pennsylvania, Philadelphia, PA 19104, USA.}    \\
	UPenn
	\and
	Yuting Wei\footnotemark[2] \\
	UPenn   \\
	\and
	Yuxin Chen\footnotemark[2] \\
	UPenn\\
	\and
	Yuejie Chi\footnotemark[1] \\
 	CMU  
	}

\date{\today}

\begin{document}

\theoremstyle{plain} \newtheorem{lemma}{\textbf{Lemma}}
\newtheorem{proposition}{\textbf{Proposition}}\newtheorem{theorem}{\textbf{Theorem}} \newtheorem{assumption}{Assumption}

\theoremstyle{remark}\newtheorem{remark}{\textbf{Remark}}

\maketitle

\begin{abstract}

Offline or batch reinforcement learning seeks to learn a near-optimal policy using history data without active exploration of the environment. To counter the insufficient coverage and sample scarcity of many offline datasets, the principle of pessimism has been recently introduced to mitigate high bias of the estimated values. While pessimistic variants of model-based algorithms (e.g., value iteration with lower confidence bounds) have been theoretically investigated, their model-free counterparts --- which do not require explicit model estimation --- have not been adequately studied, especially in terms of sample efficiency. To address this inadequacy, we study a pessimistic variant of Q-learning in the context of finite-horizon Markov decision processes, and characterize its sample complexity under the single-policy concentrability assumption which does not require the full coverage of the state-action space. In addition, a variance-reduced pessimistic Q-learning algorithm is proposed to achieve near-optimal sample complexity. Altogether, this work highlights the efficiency of model-free algorithms in offline RL when used in conjunction with pessimism and variance reduction. 
\end{abstract}


\allowdisplaybreaks

\setcounter{tocdepth}{2}
\tableofcontents

\section{Introduction}
  
Reinforcement Learning (RL) has achieved remarkable success in recent years, including matching or surpassing human performance in robotics control and strategy games \citep{silver2017mastering,mnih2015human}. 
Nevertheless, these success stories often come with nearly prohibitive cost, where an astronomical number of samples are required to train the learning algorithm to a satisfactory level. Scaling up and replicating the RL success in many real-world problems face considerable challenges, due to limited access to large-scale simulation data. In applications such as online advertising and clinical trials, real-time data collection could be expensive, time-consuming, or constrained in sample sizes as a result of experimental limitations.

On the other hand, it is worth noting that tons of samples might have already been accumulated and stored --- albeit not necessarily with the desired quality --- during previous data acquisition attempts. It is therefore natural to wonder whether such history data can be leveraged to improve performance in future deployments. 
In reality, the history data was often obtained by executing some (possibly unknown) behavior policy, 
which is typically not the desired policy. 
This gives rise to the problem of offline RL or batch RL \citep{lange2012batch,levine2020offline},\footnote{Throughout this paper, we will be using the term offline RL (resp. dataset) or batch RL (resp. dataset) interchangeably.} namely, 
how to make the best use of history data to learn an improved or even optimal policy, without further exploring the environment. 
In stark contrast to online RL that relies on active interaction with the environment, 
the performance of offline RL depends critically not only on the quantity, but also the quality of history data (e.g., coverage over the space-action space), 
given that the agent is no longer collecting new samples for the purpose of exploring the unknown environment.

Recently, the principle of pessimism (or conservatism) --- namely, being conservative in Q-function estimation when there are not enough samples --- has been put forward as an effective way to solve offline RL \citep{buckman2020importance,kumar2020conservative}. 
This principle has been implemented in, for instance,  
a model-based offline value iteration algorithm, which modifies classical value iteration \citep{azar2017minimax} by subtracting a penalty term in the estimated Q-values and has been shown to achieve appealing sample efficiency \citep{jin2021pessimism,rashidinejad2021bridging,xie2021policy}. It is noteworthy that the model-based approach is built upon the construction of an empirical transition kernel, and therefore, requires specific representation of the environment (see, e.g.~\cite{agarwal2019optimality,li2020breaking}). It remains unknown whether the pessimism principle can be incorporated into model-free algorithms --- another class of popular algorithms that performs learning without model estimation --- in a provably effective fashion for offline RL.

\newcommand{\topsepremove}{\aboverulesep = 0mm \belowrulesep = 0mm} \topsepremove

\begin{table}[t]
	\begin{center}
\begin{tabular}{c|c|c}
\toprule

	Algorithm &Type & Sample  complexity   \tabularnewline
\toprule
\hline 
VI-LCB \vphantom{$\frac{1^{7}}{1^{7^{7}}}$} & \multirow{2}{*}{model-based}  & \multirow{2}{*}{$\frac{H^6SC^{\star}}{\varepsilon^2} $}   \tabularnewline
\citep{xie2021policy} &    \tabularnewline
\hline
PEVI-Adv \vphantom{$\frac{1^{7}}{1^{7^{7}}}$} & \multirow{2}{*}{model-based}  & \multirow{2}{*}{$ \frac{H^4SC^{\star}}{\varepsilon^2} $
}  \tabularnewline
\citep{xie2021policy} & \tabularnewline \hline
\rowcolor{Gray} 
Q-LCB \vphantom{$\frac{1^{7}}{1^{7^{7}}}$} &     &  \tabularnewline
\rowcolor{Gray}
	{\bf (this work)} &   \multirow{-2}{*}{\cellcolor{Gray}\textcolor{red}{model-free}} &  \multirow{-2}{*}{\cellcolor{Gray} $ \frac{H^6SC^{\star}}{\varepsilon^2} $}  \tabularnewline
\hline
\rowcolor{Gray} 
Q-LCB-Adv \vphantom{$\frac{1^{7}}{1^{7^{7}}}$}  &   &   \tabularnewline
\rowcolor{Gray}
	{\bf (this work)} &   \multirow{-2}{*}{\cellcolor{Gray}\textcolor{red}{model-free}} & \multirow{-2}{*}{\cellcolor{Gray} $ \frac{H^4SC^\star}{\varepsilon^2} 
$}  \tabularnewline
\toprule
lower bound \vphantom{$\frac{1^{7}}{1^{7^{7}}}$} & \multirow{2}{*}{n/a}  &  \multirow{2}{*}{$ \frac{H^4SC^{\star}}{\varepsilon^2} $} \tabularnewline
\citep{xie2021policy} & & \tabularnewline
\toprule
\end{tabular}

	\end{center}

	\caption{Comparisons between our results and prior art for finding an $\varepsilon$-optimal policy in finite-horizon non-stationary MDPs. The sample complexities included in the table are valid for sufficiently small $\varepsilon$, with all logarithmic factors omitted.  
	\label{tab:prior-work}}  
	
\end{table}

\subsection{Main contributions}

In this paper, we consider finite-horizon non-stationary Markov decision processes (MDPs) with $S$ states, $A$ actions, and horizon length $H$. The focal point is to pin down the sample efficiency for pessimistic variants of model-free algorithms, under the mild single-policy concentrability assumption (cf. Assumption~\ref{assumption}) of the batch dataset introduced in \citet{rashidinejad2021bridging,xie2021policy} (in short, this assumption captures how close the batch dataset is to an expert dataset, and will be formally introduced in Section~\ref{sec:offline-concentrability}). Given $K$ episodes of history data each of length $H$ (which amounts to a total number of $T=KH$ samples), our main contributions are summarized as follows. 
\begin{itemize}
	
	\item We first study a natural pessimistic variant of the Q-learning algorithm, which simply modifies the classical Q-learning update rule by subtracting a penalty term (via certain lower confidence bounds). We prove that pessimistic Q-learning finds an $\varepsilon$-optimal policy as soon as the sample size $T$ exceeds the order of (up to log factor)
$$  
	\frac{H^6SC^{\star}}{\varepsilon^2}, 
$$
where $C^{\star}$ denotes the single-policy concentrability coefficient of the batch dataset. 
In comparison to the minimax lower bound  $\Omega \big( \frac{H^4SC^{\star}}{\varepsilon^2}\big)$ developed in \citet{xie2021policy}, 
the sample complexity of pessimistic Q-learning is at most a factor of $H^2$ from optimal (modulo some log factor).

	\item To further improve the sample efficiency of pessimistic model-free algorithms, we introduce a variance-reduced variant of pessimistic Q-learning. This algorithm is guaranteed to find an $\varepsilon$-optimal policy as long as the sample size $T$ is above the order of
$$  
	\frac{H^4SC^{\star}}{\varepsilon^2} + \frac{H^5SC^{\star}}{\varepsilon}
$$
		up to some log factor. In particular, this sample complexity is minimax-optimal (namely, as low as $\frac{H^4SC^{\star}}{\varepsilon^2}$ up to log factor) for small enough $\varepsilon$ (namely, $\varepsilon \leq \left(0, 1/H\right]$). The $\varepsilon$-range that enjoys near-optimality is much larger compared to $\varepsilon \leq \left(0, 1/H^{2.5}\right]$  in \citet{xie2021policy} for model-based algorithms.

\end{itemize}
Both of the proposed algorithms achieve low computation cost (i.e., $O(T)$) and low memory complexities (i.e., $O(\min\{T, SAH\})$). Additionally, more complete comparisons with prior sample complexities of pessimistic model-based algorithms \citep{xie2021policy} are provided in Table~\ref{tab:prior-work}. 
In comparison with model-based algorithms, model-free algorithms require drastically different technical tools to handle the complicated statistical dependency between the estimated Q-values at different time steps.

\subsection{Related works}

In this section, we discuss several lines of works which are related to ours, with an emphasis on value-based algorithms for tabular settings with finite state and action spaces.

\paragraph{Offline RL.} One of the key challenges in offline RL lies in the insufficient coverage of the batch dataset, due to lack of interaction with the environment \citep{levine2020offline,liu2020provably}. To address this challenge, most of the recent works can be divided into two lines: 1) regularizing the policy to avoid visiting under-covered state and action pairs \citep{fujimoto2019off,dadashi2021offline}; 2) penalizing the estimated values of the under-covered state-action pairs \citep{buckman2020importance,kumar2020conservative}. Our work follows the latter line (also known as the principle of pessimism), which has garnered significant attention recently. In fact, pessimism has been incorporated into recent development of various offline RL approaches, such as policy-based approaches \citep{rezaeifar2021offline,xie2021bellman,zanette2021provable}, model-based approaches \citep{rashidinejad2021bridging,uehara2021pessimistic,jin2021pessimism,yu2020mopo,kidambi2020morel,xie2021policy,yin2021towards,uehara2021representation,yan2022model,yu2021combo,yin2022near}, and model-free approaches \citep{kumar2020conservative,yu2021conservative,yan2022efficacy}.

\paragraph{Finite-sample guarantees for pessimistic approaches.}
While model-free approaches with pessimism \citep{kumar2020conservative,yu2021conservative} have achieved considerable empirical successes in offline RL, prior theoretical guarantees of pessimistic schemes have been confined almost exclusively to model-based approaches. Under the same single-policy concentrability assumption used in prior analyses of model-based approaches \citep{rashidinejad2021bridging,xie2021policy,yin2021near_double}, the current paper provides the first finite-sample guarantees for model-free approaches with pessimism in the tabular case without explicit model construction. In addition, \citet{yin2021towards} directly employed the occupancy distributions of the behavior policy and the optimal policy in bounding the performance of a model-based approach, rather than the worst-case upper bound of their ratios as done under the single-policy concentrability assumption.

\paragraph{Non-asymptotic guarantees for variants of Q-learning.} 
Q-learning, which is among the most famous model-free RL algorithms \citep{watkins1989learning,jaakkola1994convergence,watkins1992q}, has been adapted in a multitude of ways to deal with different RL settings. Theoretical analyses for Q-learning and its variants have been established in, for example, the online setting via regret analysis \citep{jin2018q,bai2019provably,zhang2020almost,li2021breaking,dong2019q,zhang2020reinforcement,zhang2020model,jafarnia2020model,yang2021q}, and the simulator setting via probably approximately correct (PAC) bounds \citep{chen2020finite,wainwright2019variance,li2021tightening}. The variant that is most closely related to ours is asynchronous Q-learning, which aims to find the optimal Q-function from Markovian trajectories following some behavior policy \citep{even2003learning,beck2012error,qu2020finite,li2021sample,yin2021near,yin2021near_double,yin2021near_double}. Different from ours, these works typically require full coverage of the state-action space by the behavior policy, a much stronger assumption than the single-policy concentrability assumed in our offline RL setting.

\paragraph{Variance reduction in RL.}
Variance reduction, originally proposed to accelerate stochastic optimization (e.g., the SVRG algorithm proposed by \citet{johnson2013accelerating}), has been successfully leveraged to improve the sample efficiency of various RL algorithms, including but not limited to policy evaluation \citep{du2017stochastic,wai2019variance,xu2019reanalysis,khamaru2020temporal},
planning \citep{sidford2018near,sidford2018variance}, Q-learning and its variants \citep{wainwright2019variance,zhang2020almost,li2021breaking,li2021sample}, and offline RL \citep{xie2021policy,yin2021near_double}.

\subsection{Notation and paper organization}

Let us introduce a set of notation that will be used throughout.  
We denote by $\Delta(\cS)$ the probability simplex over a set $\cS$, and introduce the notation $[N]\coloneqq \{1,\cdots,N\}$ for any integer $N>0$. For any vector $x \in \mathbb{R}^{SA}$ (resp.~$x\in \mathbb{R}^S$) that constitutes certain values for each of the state-action pairs (resp.~state),
we shall often use $x(s,a)$ (resp.~$x(s)$) to denote the entry associated with the $(s,a)$ pair (resp. state $s$). 
Similarly, we shall denote by $x \defn \{x_h\}_{h\in [H]}$ the set composed of certain vectors for each of the time step $h\in[H]$. We let $e_i$ represent the $i$-th standard basis vector, with the only non-zero element being in the $i$-th entry.

Let $\mathcal{X}\defn ( S, A,  H , T )$. 
The notation $f(\mathcal{X})\lesssim g(\mathcal{X})$ (resp.~$f(\mathcal{X})\gtrsim g(\mathcal{X})$) means that there exists a universal constant $C_{0}>0$ such that $|f(\mathcal{X})|\leq C_{0}|g(\mathcal{X})|$ (resp.~$|f(\mathcal{X})|\geq C_{0}|g(\mathcal{X})|$). In addition, we often overload scalar functions and expressions to take vector-valued arguments, with the interpretation that they are applied in an entrywise manner. For example, for a vector $x=[x_i]_{1\leq i\leq n}$, we have $x^2 =[x_i^2]_{1\leq i\leq n}$. For any two vectors $x=[x_i]_{1\leq i\leq n}$ and $y=[y_i]_{1\leq i\leq n}$, the notation $ {x}\leq {y}$ (resp.~$ {x}\geq {y}$) means
$x_{i}\leq y_{i}$ (resp.~$x_{i}\geq y_{i}$) for all $1\leq i\leq n$.

\paragraph{Paper organization.} The rest of this paper is organized as follows. Section~\ref{sec:problem-formulation} introduces the backgrounds on finite-horizon MDPs and formulates the offline RL problem. Section~\ref{sec:algorithm-theory-lcbq} starts by introducing a natural pessimistic variant of Q-learning along with its sample complexity bound, and further enhances the sample efficiency via variance reduction in  Section~\ref{sec:pessimistic-Q-vr}. 
Section~\ref{sec:analysis} presents the proof outline and key lemmas. 
Finally, we conclude in Section~\ref{sec:discussions} with a discussion and defer the proof details to the supplementary material.

\section{Background and problem formulation}
\label{sec:problem-formulation}

\subsection{Tabular finite-horizon MDPs}

\paragraph{Basics.}  
This work focuses on an episodic finite-horizon MDP as represented by 
$$\mathcal{M}= \big(\mathcal{S},\mathcal{A},H, \{P_h\}_{h=1}^H, \{r_h\}_{h=1}^H \big),$$ where $H$ is the horizon length, $\mathcal{S}$ is a finite state space of cardinality $S$, 
$\mathcal{A}$ is a finite action space of cardinality $A$, 
and $P_h : \cS \times \cA \rightarrow \Delta (\cS) $ (resp.~$r_h: \cS \times \cA \rightarrow [0,1]$)
represents the probability transition kernel (resp.~reward function) at the $h$-th time step $(1\leq h\leq H)$.
Throughout this paper, we shall adopt the following convenient notation
\begin{equation} \label{eq:transition_vector}
P_{h,s,a} \coloneqq P_h(\cdot \mymid s,a ) \in [0,1]^{1\times S},
\end{equation}
which stands for the transition probability vector given the current state-action pair $(s,a)$ at time step $h$. 
The parameters $S$, $A$ and $H$ can all be quite large, allowing one to capture the challenges arising in MDPs with large state/action space and long horizon.

A policy (or action selection rule) of an agent is represented by $\pi =\{\pi_h\}_{h=1}^H$, where $\pi_h: \mathcal{S} \rightarrow \Delta(\mathcal{A})$ specifies the associated selection probability over the action space at time step $h$ (or more precisely, we let $\pi_h(a\mymid s)$ represent the probability of selecting action $a$ in state $s$ at step $h$). When $\pi$ is a deterministic policy, we abuse the notation and let $\pi_h(s)$ denote the action selected by policy $\pi$ in state $s$ at step $h$. 
In each episode,  the agent generates an initial state $s_1\in \cS$ drawn from an initial state distribution $\rho\in \Delta(\cS)$, and rolls out a trajectory over the MDP by executing a policy $\pi$ as follows:
\begin{equation}\label{eq:single_trajectory}
\{s_h, a_h, r_h\}_{h=1}^H =\{ s_1,\, a_1, \, r_1,\,   \ldots, s_H,\, a_H, \, r_H \},
\end{equation}
where at time step $h$, $a_h \sim \pi_h(\cdot \mymid s_h)$ indicates the action selected in state $s_h$, $r_h= r_h(s_h,a_h) $ denotes the deterministic immediate reward, and $s_{h+1} $ denotes the next state drawn from the transition probability vector $P_{h,s_h,a_h}\coloneqq P_h(\cdot \mymid s_h,a_h )$. 
In addition, let $d_h^{\pi}(s)$ and $d_h^{\pi}(s,a)$ denote respectively the occupancy distribution induced by $\pi$ at time step $h\in [H]$, namely, 

\begin{align} \label{eq:visitation_dist}
	d_h^{\pi}(s) & \defn \mathbb{P}(s_h = s \mymid s_1 \sim \rho, \pi),  \qquad
	d_h^{\pi}(s, a)  \defn \mathbb{P}(s_h = s \mymid s_1 \sim \rho, \pi) \, \pi_h(a \mymid s); 
\end{align} 
here and throughout, we denote $[H]\coloneqq \{1,\cdots,H\}$. 
Given that the initial state $s_1$ is drawn from $\rho$, the above definition gives
\begin{align}
	d_1^{\pi}(s) = \rho(s) \qquad \text{for any policy }\pi. 
	\label{eq:d1-pi-s-deterministic}
\end{align}

\paragraph{Value function, Q-function, and optimal policy.}
The value function $V^{\pi}_{h}(s)$ of policy $\pi$ in state $s$ at step $h$ is defined as the expected cumulative rewards when  this policy is executed starting from state $s$ at step $h$, i.e., 
\begin{align}
	\label{eq:def_Vh}
	V^{\pi}_{h}(s) &\defn  \mathbb{E} 
	\left[  \sum_{t=h}^{H} r_{t}\big(s_{t},a_t \big) \,\Big|\, s_{h}=s \right], 
\end{align}
where the expectation is taken over the randomness of the trajectory \eqref{eq:single_trajectory} induced by the policy $\pi$ as well as the MDP transitions. Similarly, the Q-function $Q^{\pi}_h(\cdot,\cdot)$ of a policy $\pi$ at step $h$ is defined as
\begin{align} 
	\label{eq:def_Qh}
	Q^{\pi}_{h}(s,a) & \defn r_{h}(s,a)+ \mathbb{E} \left[  \sum_{t=h +1}^{H} r_t (s_t, a_t ) \,\Big|\, s_{h}=s, a_h  = a\right] ,
\end{align}
where the expectation is again over the randomness induced by $\pi$ and the MDP except that the state-action pair at step $h$ is now conditioned to be $(s,a)$. 
By convention, we shall also set 
\begin{equation}\label{eq:value-H-1-zero}
	V^{\pi}_{H+1}(s)= Q^{\pi}_{H+1}(s,a)=0  ~~\text{ for any }\pi\text{ and }(s,a)\in \cS\times \cA.
\end{equation}

A policy $\pi^{\star} =\{\pi_h^{\star}\}_{h=1}^H$ is said to be an optimal policy if it maximizes the value function (resp.~Q-function) {\em simultaneously} for all states (resp.~state-action pairs) among all policies, whose existence is always guaranteed \citep{puterman2014markov}. 
The resulting optimal value function $V^{\star} =\{ V_h^{\star} \}_{h=1}^H $ and optimal Q-functions $Q^{\star} =\{ Q_h^{\star} \}_{h=1}^H $ are denoted respectively by  
\begin{align*}
	V_h^{\star}(s) & \coloneqq V_h^{\pi^{\star}}(s) = \max_{\pi} V_h^{\pi}(s), \qquad
	Q_h^{\star}(s,a)  \coloneqq Q_h^{\pi^{\star}}(s,a)  = \max_{\pi} Q_h^{\pi}(s,a) 
\end{align*}
for any $(s,a,h)\in \cS\times \cA \times [H]$. Throughout this paper, we assume that $\pi^{\star}$ is a {\em deterministic optimal policy}, which always exists \citep{puterman2014markov}. 

Additionally, when the initial state is drawn from a given distribution $\rho$, 
the expected value of a given policy $\pi$ and that of the optimal policy at the initial step are defined respectively by
\begin{align}\label{eq:defn-V-rho}
	V_1^{\pi}(\rho) \coloneqq \mathop{\mathbb{E}}\limits_{s_1\sim \rho} \big[ V_1^\pi (s_1) \big] 
	 \qquad \text{and}  \qquad
	V_1^{\star}(\rho) \coloneqq \mathop{\mathbb{E}}\limits_{s_1\sim \rho} \big[ V_1^\star (s_1) \big].
\end{align}

\paragraph{Bellman equations.} The Bellman equations play a fundamental role in dynamic programming \citep{bertsekas2017dynamic}. 
Specifically, the value function and the Q-function of any policy $\pi$ satisfy the following Bellman consistency equation:
\begin{align}
	\label{eq:bellman}
 Q^{\pi}_{h}(s,a)=r_{h}(s,a)+ \mathop{\mathbb{E}}\limits_{s'\sim P_{h,s,a}} \big[V^{\pi}_{h+1}(s') \big]
\end{align}
for all $(s,a,h)\in \cS\times \cA\times [H]$. Moreover, the optimal value function and the optimal Q-function satisfy the Bellman optimality equation:
\begin{align} \label{eq:bellman_optimality}
	Q^{\star}_{h}(s,a)=r_{h}(s,a)+ \mathop{\mathbb{E}}\limits_{s'\sim P_{h,s,a}} \big[V^{\star}_{h+1}(s') \big]
\end{align}
for all $ (s,a,h)\in \cS\times \cA\times [H]$.

\subsection{Offline RL under single-policy concentrability}
	\label{sec:offline-concentrability}

Offline RL assumes the availability of a history dataset $\Datab$ containing $K$ episodes each of length $H$.
These episodes are independently generated based on a certain policy $\pib = \{ \pib_{h} \}_{h=1}^H$ --- called the {\em behavior policy}, resulting in a dataset 
\[
	\Datab \coloneqq  \Big\{  \big( s_1^k,\, a_1^k, \, r_1^k,\, \ldots, s_H^k,\, a_H^k, \, r_H^k  \big) \Big\}_{k=0}^{K-1} . 
\]
Here, the initial states $\{s_1^k \}_{k=1}^K$ are independently drawn from $\rho \in \Delta(\cS)$ such that $s_1^k \overset{\mathrm{i.i.d.}}{\sim} \rho $, 
while the remaining states and actions are generated by the MDP induced by the behavior policy $\mu$. 
The total number of samples is thus given by 
$$T= KH. $$ 
With the notation \eqref{eq:defn-V-rho} in place, 
the goal of offline RL amounts to finding an $\varepsilon$-optimal policy $\widehat{\pi} =\{ \widehat{\pi}_h \}_{h=1}^H$ satisfying
\[   
	V_1^{\star}(\rho) - V_1^{\widehat{\pi}}(\rho) 
	\leq \varepsilon 
\]
with as few samples as possible, and ideally, in a computationally fast and memory-efficient manner.

Obviously, efficient offline RL cannot be accomplished without imposing proper assumptions on the behavior policy, which also provide means to gauge the difficulty of the offline RL task through the quality of the history dataset. Following the recent works \citet{rashidinejad2021bridging,xie2021policy}, we assume that the behavior policy $\mu$ satisfies the following property called {\em single-policy concentrability}.

\begin{assumption}[single-policy concentrability]
\label{assumption}
The single-policy concentrability coefficient $C^\star \in [1, \infty)$ of a behavior policy $\mu$ is defined to be the smallest quantity that satisfies 
\begin{equation}\label{equ:concentrability-assumption}
    \max_{(h, s, a) \in [H] \times \mathcal{S} \times \cA} \frac{d^{\pi^\star}_h(s,a)}{d^{\mu}_h(s,a)} \leq C^\star,
\end{equation}
where we adopt the convention $0/0=0$. 
\end{assumption}

Intuitively, the single-policy concentrability coefficient measures the discrepancy between the optimal policy $\pi^{\star}$ and the behavior policy $\mu$ in terms of the resulting density ratio of the respective occupancy distributions. 
It is noteworthy that a finite $C^{\star}$ does not necessarily require $\mu$ to cover 
the entire state-action space; instead, it can be attainable when its coverage subsumes that of the optimal policy $\pi^{\star}$. 
This is in stark contrast to, and in fact much weaker than, other assumptions that require either full coverage of the behavior policy (i.e., $\min_{(h, s, a) \in [H] \times \mathcal{S} \times \cA}  d^{\mu}_h(s,a)>0$ \citep{li2021sample,yin2021near,yin2021near_double}), or uniform concentrability over all possible policies \citep{chen2019information}. Additionally, the single-policy concentrability coefficient is minimized (i.e., $C^{\star}=1$) when the behavior policy $\mu$ coincides with the optimal policy $\pi^{\star}$, a scenario closely related to imitation learning or behavior cloning \citep{rajaraman2020toward}.

\section{Pessimistic Q-learning: algorithms and theory}
\label{sec:algorithm-theory-lcbq}

In the current paper, we present two model-free algorithms --- namely, \LCBQ and {\LCBQR} --- for offline RL, along with their respective theoretical guarantees. The first algorithm can be viewed as a pessimistic variant of the classical Q-learning algorithm, 
while the second one further leverages the idea of variance reduction to boost the sample efficiency. In this section, we begin by introducing \LCBQ.

\subsection{\LCBQ: a natural pessimistic variant of Q-learning}

Before proceeding, we find it convenient to first review the classical Q-learning algorithm \citep{watkins1989learning,watkins1992q}, which can be regarded as a stochastic approximation scheme to solve the Bellman optimality equation \eqref{eq:bellman_optimality}. 
Upon receiving a sample transition $(s_h,a_h,r_h,s_{h+1})$ at time step $h$, 
Q-learning updates the corresponding entry in the Q-estimate as follows
\begin{align}    \label{eq:classical-Q-update}
    Q_h(s_h,a_h) ~\leftarrow~ & (1-\eta )Q_h(s_h,a_h)  + \eta \Big \{ r_h(s_h,a_h) + 
  V_{h+1}(s_{h+1}) \Big \} ,
\end{align}
where $Q_h$ (resp.~$V_h$) indicates the running estimate of $Q_h^{\star}$ (resp.~$V_h^{\star}$), and $0<\eta<1$ is the learning rate. 
In comparison to model-based algorithms that require estimating the probability transition kernel based on all the samples, Q-learning, as a popular kind of model-free algorithms, is simpler and enjoys more flexibility without explicitly constructing the model of the environment.
The wide applicability of Q-learning motivates one to adapt it to accommodate offline RL.

Inspired by recent advances in incorporating the pessimism principle for offline RL \citep{rashidinejad2021bridging,jin2021pessimism}, 
we study a pessimistic variant of Q-learning called \LCBQ, which modifies the Q-learning update rule as follows
\begin{align}
    \label{equ:lcb-q-update}
    Q_h(s_h, a_h) \leftarrow & (1-\eta_n )Q_h(s_h, a_h) + \eta_n \Big\{ r_h(s_h,a_h) +  V_{h+1}(s_{h+1})  - b_n \Big\}, 
\end{align}
where $\eta_n$ is the learning rate depending on the number of times $n$ that the state-action pair $(s_h,a_h)$ has been visited at step $h$, and the penalty term $b_n>0$ (cf.~line~\ref{line:2-nok} of Algorithm~\ref{algo:lcb-index}) reflects the uncertainty of the corresponding Q-estimate and implements pessimism in the face of uncertainty. 
The entire algorithm, which is a {\em single-pass} algorithm that only requires reading the offline dataset once, 
is summarized in Algorithm~\ref{algo:lcb-index}.

\begin{algorithm}[h]
 \textbf{Parameters:} some constant $\cb>0$, target success probability $1-\delta \in(0,1)$, and $\iota = \log\big(\frac{SAT}{\delta}\big)$. \\
\textbf{Initialize}  $Q_h(s, a) \leftarrow 0$, $N_h(s, a)\leftarrow 0$, and $V_h(s) \leftarrow 0$ for all $(s,h)\in \cS \times[H+1]$; $\widehat{\pi}$ s.t. $\widehat{\pi}_h(s) = 1$ for all $(h,s)\in [H]\times\cS$.  \\
\For{Episode $ k = 1$ \KwTo $K$}{
	Sample a new trajectory $\{s_h, a_h, r_h\}_{h=1}^H$ from $\mathcal{D}_{\mu}$. {\small\color{own_blue}\tcp{sampling from batch dataset}} 
   {\small\color{own_blue}\tcp{update the policy} }  \normalsize
  \For{Step $ h = 1$ \KwTo $H$}{

         $N_h(s_h, a_h) \leftarrow N_h(s_h, a_h) + 1$. {\small\color{own_blue}\tcp{update the counter} }
   
         $n \leftarrow N_h(s_h, a_h)$; 
        $\eta_n \leftarrow \frac{H + 1}{H + n}$.  
        {\small\color{own_blue}\tcp{update the learning rate}}      
        $b_n \leftarrow \cb \sqrt{\frac{H^3\iota^2}{n}}$.  \label{line:2-nok}  {\small\color{own_blue}\tcp{update the bonus term}}
        {\small\color{own_blue}\tcp{run the Q-learning update with LCB} }  \normalsize
  $Q_h(s_h, a_h) \leftarrow Q_h(s_h, a_h) + \eta_n\Big\{ r_h(s_h, a_h) + V_{h+1}(s_{h+1})-Q_h(s_h, a_h) - b_n \Big\} 
  $. \label{line:lcb-9-nok} \\
          
     {\small\color{own_blue}\tcp{update the value estimates}} \normalsize
      
    $V_{h}(s_h) \leftarrow \max\Big\{ V_{h}(s_h), \,  \max_a Q_{h}(s_h, a)  \Big\}$. \label{line:lcb_v_update}\\
    \label{line:lcb-10}
    If $V_{h}(s_h) = \max_a Q_{h}(s_h, a)$: update $\widehat{\pi}_h(s) \leftarrow \arg\max_a Q_{h}(s, a)$.
 
 }
 }
 \textbf{Output:} the policy $\widehat{\pi}$.
\caption{\LCBQ for offline RL}
\label{algo:lcb-index}
\end{algorithm}

\subsection{Theoretical guarantees for \LCBQ{}}
 
The proposed {\LCBQ{}} algorithm manages to achieve an appealing sample complexity as formalized by the following theorem. 
\begin{theorem}\label{thm:lcb}
  Consider any $\delta \in (0,1)$.
  Suppose that the behavior policy $\mu$ satisfies Assumption~\ref{assumption} with single-policy concentrability coefficient $C^{\star}\geq 1$. Let $\cb>0$ be some sufficiently large constant, and take $\iota \coloneqq \log\big(\frac{SAT}{\delta}\big)$. Assume that $T>SC^\star \iota$, then the policy $\widehat{\pi}$ returned by Algorithm~\ref{algo:lcb-index} satisfies
\begin{align}\label{equ:lcb-result}
    V_1^\star(\rho) - V_1^{\widehat{\pi}}(\rho) 
  &\leq c_\mathrm{a} \sqrt{\frac{H^6 SC^\star \iota^3}{T}} 
\end{align}
with probability at least $1-\delta$, where $c_\mathrm{a}>0$ is some universal constant. 
\end{theorem}

 As asserted by Theorem~\ref{thm:lcb}, the {\LCBQ{}} algorithm is guaranteed to find an $\varepsilon$-optimal policy with high probability, 
 as long as the total sample size $T=KH$ exceeds
 \begin{equation}
  \widetilde{O}\left( \frac{H^6 S C^\star}{\varepsilon^2} \right),
   \label{eq:lcb-sample-complexity}
\end{equation} 
where $\widetilde{O}(\cdot)$ hides logarithmic dependencies. When the behavior policy is close to the optimal policy, the single-policy concentrability coefficient $C^{\star}$ is closer to 1; if this is the case, then our bound indicates that the sample complexity does not depend on the size $A$ of the action space, which can be a huge saving when the action space is enormous.

\paragraph{Comparison with model-based pessimistic approaches.} 

A model-based approach --- called Value Iteration with Lower Confidence Bounds  (\VILCB) ---
has been recently proposed for offline RL \citep{rashidinejad2021bridging,xie2021policy}. 
In the finite-horizon case, \VILCB{} incorporates an additional LCB penalty into the classical value iteration algorithm, and updates {\em all} the entries in the Q-estimate simultaneously as follows
\begin{align}
    Q_h(s,a) ~\leftarrow~ r_h(s,a) + \widehat{P}_{h,s,a}V_{h+1} - b_h(s,a), \label{equ:VI-UCB}
\end{align}
with the aim of tuning down the confidence on those state-action pairs that have only been visited infrequently. 
Here, $\widehat{P}_{h,s,a}$ represents the empirical estimation of the transition kernel $P_{h,s,a}$, 
and $b_h(s,a) >0$ is chosen to capture the uncertainty level of $(\widehat{P}_{h,s,a} -P_{h,s,a})V_{h+1}$. 
Working backward, the algorithm estimates the Q-value $Q_h$ recursively over the time steps $h = H, H-1, \cdots, 1$.  In comparison with \VILCB, our sample complexity bound for  {\LCBQ{}} matches the bound developed for \VILCB{} by \citet{xie2021policy}, while enjoying enhanced flexibility without the need of specifying the transition kernel of the environment (as model estimation might potentially incur a higher memory burden).

\subsection{\LCBQR for near-optimal offline RL}
\label{sec:pessimistic-Q-vr}

\begin{algorithm}[t]
\textbf{Parameters:} number of epochs $M$, universal constant $\cb>0$, probability of failure $\delta \in(0,1)$, and $\iota = \log\big(\frac{SAT}{\delta}\big)$;\\
\textbf{Initialize:} \\
$Q_h(s, a), Q_h^{\LCB}(s, a), \overline{Q}_h(s, a), \overline{\mu}_h(s,a), \overline{\mu}^{\nnext}_h(s,a), N_h(s, a)  \leftarrow 0$ for all $(s,a,h) \in \cS\times\cA \times [H]$;\\
$V_h(s), \overline{V}_h(s), \overline{V}_h^{\nnext}(s) \leftarrow 0$ for all $(s,h) \in \cS \times [H+1]$;\\
$\refmu_h(s, a), \refsg_h(s, a), \advmu_h(s, a)$, $\advsg_h(s, a)$, $\overline{\delta}_h(s, a), \overline{B}_h(s, a) \leftarrow 0$ for all $(s, a, h) \in \cS\times \cA \times [H]$.\\[0.4em]

\For{Epoch $ m = 1$ \KwTo $M$}{

  $L_m= 2^m$; \small {\color{own_blue}\tcp{specify the number of episodes in the current epoch} }\normalsize

  $\widehat{N}_h(s,a) = 0$ for all $(h, s,a) \in [H] \times \cS\times \cA.$ \small {\color{own_blue}\tcp{reset the epoch-wise counter} }\normalsize

  \small {\color{own_blue}\tcc{Inner-loop: update value-estimates $V_h(s,a)$ and Q-estimates $Q_h(s,a)$ }} \normalsize

  \For{In-epoch Episode $ t = 1$ \KwTo $L_m$}{
    Sample a new trajectory $\{s_h, a_h,r_h\}_{h=1}^H$. {\small\color{own_blue}\tcp{sampling from batch dataset}}
  
    \For{Step $ h = 1$ \KwTo $H$}{
        $N_h(s_h, a_h) \leftarrow N_h(s_h, a_h) + 1$; $n \leftarrow N_h(s_h,a_h)$. 
        {\small\color{own_blue}\tcp{update the overall counter} }
        $\eta_n \leftarrow \frac{H + 1}{H + n}$;  \small {\color{own_blue}\tcp{update the learning rate}} \normalsize
           
          {\small\color{own_blue}\tcp{run the Q-learning update rule with LCB}} \normalsize 
        $Q_h^{\LCB}(s_h, a_h) \leftarrow \texttt{update-lcb-q()}$. 
     
             {\small\color{own_blue}\tcp{update the Q-estimate with LCB and reference-advantage}} \normalsize 
             
        $\overline{Q}_h(s_h, a_h) \leftarrow \texttt{update-lcb-q-ra()}$.

    \small {\color{own_blue}\tcp{update the Q-estimate $Q_h$ and value estimate $V_h$}} \normalsize
    $Q_h(s_h,a_h) \leftarrow \max \big\{Q_h^{\LCB}(s_h, a_h), \overline{Q}_h(s_h, a_h), Q_h(s_h, a_h) \big\}.$ \label{eq:line-q-update}\\
    $V_{h}(s_h) \leftarrow \max_a Q_h(s_h, a)$. \label{eq:line-v-update}\\

    \small {\color{own_blue}\tcp{update the epoch-wise counter and $\overline{\mu}^{\nnext}_h$ for the next epoch}} \normalsize
    $\widehat{N}_h(s_h, a_h) \leftarrow \widehat{N}_h(s_h, a_h) + 1$;\\
    $\overline{\mu}_h^{\nnext}(s_h, a_h) \leftarrow \left(1-\frac{1}{\widehat{N}_h(s_h,a_h)}\right)\overline{\mu}_h^{\nnext}(s_h, a_h) + \frac{1}{\widehat{N}_h(s_h,a_h)} \overline{V}^{\nnext}_{h+1}(s_{h+1})$;  \label{line:ref-mean-update}
    }
 }

 \small {\color{own_blue}\tcc{Update the reference ($\overline{V}_h$, $\overline{V}^\nnext_h$) and ($\overline{\mu}_h$, $\overline{\mu}^\nnext_h$)} } \normalsize
 \For{$(s,a,h) \in \cS\times\cA\times [H+1]$}{
$\overline{V}_h(s) \leftarrow \overline{V}_h^{\nnext}(s)$; $\overline{\mu}_h(s,a)\leftarrow \overline{\mu}_h^{\nnext}(s,a)$. \label{eq:update-mu-reference-v}   {\small\color{own_blue}\tcp{set $\overline{V}_h$ and $\overline{\mu}_h$ for the next epoch} }
$\overline{V}_h^\nnext(s) \leftarrow V_h(s)$; $\overline{\mu}_h^\nnext(s,a) \leftarrow {0}$. {\small\color{own_blue}\tcp{restart $\overline{\mu}_h^\nnext$ and set $\overline{V}_h^\nnext$ for the next epoch} }
}
  
 }
\KwOut{the policy $\widehat{\pi}$ s.t.~$\widehat{\pi}_h(s) = \arg\max_a Q_{h}(s, a)$ for any $(s,h)\in \cS\times[H]$.}

\caption{Offline \LCBQR RL}
\label{alg:lcb-advantage-per-epoch}
\end{algorithm}

\begin{figure}[t]
\begin{center}
\includegraphics[width=0.8\linewidth]{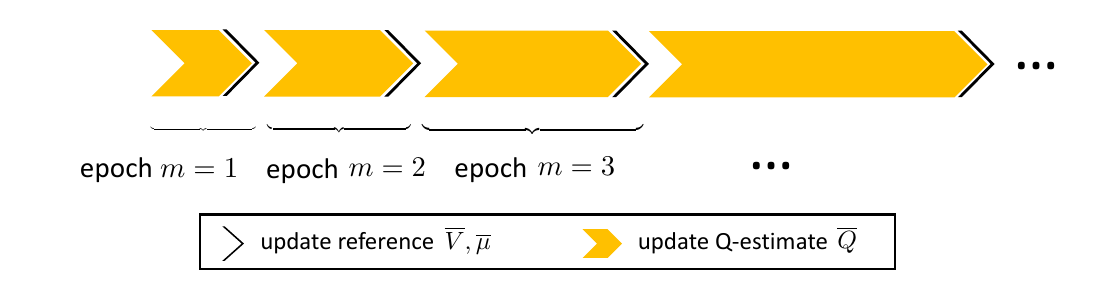}
\end{center}
\caption{An illustration of the epoch-based \LCBQR algorithm.}
\label{alg:illu}
\end{figure}

The careful reader might notice that the sample complexity \eqref{eq:lcb-sample-complexity} derived for \LCBQ remains a factor of $H^2$ away from the minimax lower bound (see Table~\ref{tab:prior-work}). 
To further close the gap and improve the sample complexity, we propose a new variant called \LCBQR, which leverages the idea of variance reduction to accelerate convergence \citep{johnson2013accelerating,sidford2018variance,wainwright2019variance,zhang2020almost,xie2021policy,li2021sample,li2021breaking}.

Inspired by the reference-advantage decomposition adopted in \citep{zhang2020almost,li2021breaking} for online Q-learning, 
\LCBQR maintains a collection of reference values $\{\overline{V}_h\}_{h=1}^H$, which serve as running proxy for the optimal values $\{V^\star_h\}_{h=1}^H$ and allow for reduced variability in each iteration. To be more specific, the \LCBQR algorithm (cf.~Algorithm~\ref{alg:lcb-advantage-per-epoch} as well as the subroutines in Algorithm~\ref{algo:subroutine} that closely resemble \citet{li2021breaking}) proceeds in an epoch-based style (the $m$-th epoch consists of $L_{m}=2^m$ episodes of samples), where the reference values are updated at the end of each epoch to be used in the next epoch, and the Q-estimates are iteratively updated during the remaining time of each epoch. By maintaining two auxiliary sequences of {\em pessimistic} Q-estimates --- that is, $Q^{\LCB}$ constructed by the pessimistic Q-learning update, and $\overline{Q}$ constructed by the pessimistic Q-learning update based on the reference-advantage decomposition --- the Q-estimate is updated by taking the maximum over the three candidates (cf.~line~\ref{eq:line-q-update} of Algorithm~\ref{alg:lcb-advantage-per-epoch})
\begin{equation}
Q_h(s,a) \leftarrow \max \{Q_h^{\LCB}(s, a),\, \overline{Q}_h(s, a),\, Q_h(s, a)\}
\end{equation}
when the state-action pair $(s,a)$ is visited at the step $h$. We now take a moment to discuss the key ingredients of the proposed algorithm in further detail.

\paragraph{Updating the references $\overline{V}_h$ and $\overline{\mu}_h$.} 
At the end of each epoch, the reference values $\{\overline{V}_h\}_{h=1}^H$, as well as the associated running average $\{\overline{\mu}_h\}_{h=1}^H$, are determined using what happens during the current epoch. 
More specifically, the following update rules for $\overline{V}_h$ and $\overline{\mu}_h$ are carried out at the end of the $m$-th epoch: 
\begin{subequations}
\begin{align}
  \overline{V}_h(s) &\leftarrow \overline{V}^{\nnext}_h(s) , \\
  \overline{\mu}_h(s,a) &\leftarrow \frac{\sum_{t=1}^{L_{m}} \ind(s_h^t = s, a_h^t = a) \overline{V}_{h+1}(s_{h+1}^t)}{ \max\Big\{ \big\{ \sum_{t=1}^{L_{m}} \ind(s_h^t = s, a_h^t = a) \big\} , 1 \Big\} } 
  \label{eq:update-mu-principle}
\end{align}
\end{subequations}
\normalsize
for all $(h,s,a)\in [H]\times \cS\times\cA$. 
Here, $\overline{V}_h(s)$ is assigned by $\overline{V}^{\nnext}_h(s)$, which is maintained as the value estimate $V_h(s)$ at the end of the $(m-1)$-th epoch, and the update of $\overline{\mu}_h(s,a) $ is implemented in a recursive manner in the current $m$-th epoch. See also line~\ref{eq:update-mu-reference-v} and line~\ref{line:ref-mean-update} of Algorithm~\ref{alg:lcb-advantage-per-epoch}. 

\paragraph{Learning Q-estimate $\overline{Q}_h$ based on the reference-advantage decomposition.}
Armed with the references $\overline{V}_h$ and $\overline{\mu}_h$ updated at the end of the previous $(m-1)$-th epoch, \LCBQR iteratively updates the Q-estimate $\overline{Q}_h$ in all episodes during the $m$-th epoch.
At each time step $h$ in any episode, whenever $(s,a)$ is visited, \LCBQR updates the reference Q-value as follows:
\begin{align}
  &\overline{Q}_h(s, a) \leftarrow (1-\eta)\overline{Q}_h(s, a) + \eta\Big\{ r_h(s,a) + \underset{\text{estimate of } P_{h,s,a} (V_{h+1} - \overline{V}_{h+1})}{\underbrace{\widehat{P}_{h,s,a} \big(V_{h+1}- \overline{V}_{h+1}\big)}} \hspace{-0.3em} + \hspace{-0.3em} \underset{\text{estimate of } P_{h,s,a}\overline{V}_{h+1}}{\underbrace{\overline{\mu}_h}}  \hspace{-1.0em}   - \overline{b}_h(s, a) \Big\}. \label{eq:lcb-adv-principle-rule}
\end{align}
Intuitively, we decompose the target $P_{h,s,a} V_{h+1}$ into a reference part $P_{h,s,a}\overline{V}_{h+1}$ and an advantage part $P_{h,s,a} (V_{h+1} - \overline{V}_{h+1})$, and cope with the two parts separately.  
In the sequel, let us take a moment to discuss three essential ingredients of the update rule \eqref{eq:lcb-adv-principle-rule}, which shed light on the design rationale of our algorithm. 
\begin{itemize}
\item
Akin to \LCBQ, the term $\widehat{P}_{h,s,a} \big(V_{h+1}- \overline{V}_{h+1}\big)$ serves as an unbiased stochastic estimate of $P_{h,s,a} \left(V_{h+1}- \overline{V}_{h+1}\right)$ if a sample transition $(s,a,s_{h+1})$ at time step $h$ is observed. 
If $V_{h+1}$ stays close to the reference $\overline{V}_{h+1}$ as the algorithm proceeds, the variance of this stochastic term can be lower than that of the stochastic term $\widehat{P}_{h,s,a} V_{h+1}$ in \eqref{equ:lcb-q-update}.

\item 
The auxiliary estimate $\overline{\mu}_h$ introduced in \eqref{eq:update-mu-principle} serves as a running estimate of the reference part $P_{h,s,a} \overline{V}_{h+1}$. 
    Based on the update rule \eqref{eq:update-mu-principle}, we design $\overline{\mu}_h(s,a)$ to estimate  the running mean of the reference part $\big[P_{h,s,a} \overline{V}_{h+1}\big]$ using a number of previous samples. As a result, we expect the variability of this term to be well-controlled, particularly as the number of samples in each epoch grows exponentially (recall that $L_m=2^m$).

\item In each episode, the term $\overline{b}_h(s,a)$ serves as the additional confidence bound on the error between the estimates of the reference/advantage and the ground truth. More specifically,  $\refmu_h(s, a)$ and $\refsg_h(s, a)$ are respectively the running mean and 2nd moment of the reference part $\big[P_{h,s,a} \overline{V}_{h+1}\big]$  (cf.~lines~\ref{line:refmu_h}-\ref{line:refsigma_h} of Algorithm~\ref{algo:subroutine}); $\advmu_h(s, a)$ and $\advsg_h(s, a)$ represent respectively the running mean and 2nd moment of the advantage part $\big[P_{h,s,a}(V_{h+1} - \overline{V}_{h+1})\big]$ (cf.~lines~\ref{line:advmu_h}-\ref{line:advsigma_h} of Algorithm~\ref{algo:subroutine});
   $\overline{B}_h(s, a)$ aggregates the empirical standard deviations of the reference and the advantage parts. The LCB penalty term $ \overline{b}_h(s, a)$ is updated using $ \overline{B}_h(s, a)$ and $\overline{\delta}_h(s_h, a_h)$ (cf.~lines~\ref{line:bonus_1}-\ref{line:bonus_2} of Algorithm~\ref{algo:subroutine}), taking into account the confidence bounds for both the reference and the advantage.

\end{itemize}

In a nutshell, the auxiliary sequences of the reference values are designed to help reduce the variance of the stochastic Q-learning updates, 
which taken together with the principle of pessimism play a crucial role in the improvement of sample complexity for offline RL.

\subsection{Theoretical guarantees for \LCBQR}
Encouragingly, the proposed \LCBQR algorithm provably achieves near-optimal sample complexity for sufficiently small $\varepsilon$, as demonstrated by the following theorem. 
\begin{theorem}\label{thm:lcb-adv} 
Consider any $\delta \in (0,1)$, and recall that $\iota=\log\big( \frac{SAT}{\delta} \big)$ and $T=KH$. 
Suppose that $\cb>0$ is chosen to be a sufficiently large constant, and that the behavior policy $\mu$ satisfies Assumption~\ref{assumption}. Then there exists some universal constant $c_{\mathrm{g}}>0$ such that with probability at least $1-\delta$, the policy $\widehat{\pi}$ output by Algorithm~\ref{alg:lcb-advantage-per-epoch} satisfies
\begin{align}\label{equ:lcb-result}
    V_1^\star(\rho) - V_1^{\widehat{\pi}}(\rho) 
    &\leq c_{\mathrm{g}} \bigg(\sqrt{ \frac{H^4SC^\star  \iota^5}{T}} + \frac{H^5 SC^\star \iota^4}{T} \bigg).
\end{align}
\end{theorem}
As a consequence, Theorem~\ref{thm:lcb-adv} reveals that
the  \LCBQR algorithm is guaranteed to find an $\varepsilon$-optimal policy (i.e., $V_1^\star(\rho) - V_1^{\widehat{\pi}}(\rho)\leq \varepsilon$) as long as the total sample size $T$ exceeds 
\begin{equation}
  \widetilde{O}\left(\frac{H^4 S C^\star}{\varepsilon^2} + \frac{H^5 S C^\star}{\varepsilon}\right).
  \label{eq:sample-complexity-LCB-ADV-full}
\end{equation}
For sufficiently small accuracy level $\varepsilon$ (i.e., $\varepsilon\leq 1/H$), this results in a sample complexity of 
\small
\begin{equation}
  \widetilde{O}\left(\frac{H^4 S C^\star}{\varepsilon^2} \right),
  \label{eq:sample-complexity-LCB-ADV}
\end{equation}
\normalsize
thereby matching the minimax lower bound developed in \citet{xie2021policy} up to logarithmic factor. Compared with the minimax lower bound $\Omega \big( \frac{H^4SA}{\varepsilon^2}\big)$ in the online RL setting \citep{domingues2021episodic}, this suggests that offline RL can be fairly sample-efficient when the behavior policy closely mimics the optimal policy in terms of the resulting state-action occupancy distribution (a scenario where $C^{\star}$ is potentially much smaller than the size of the action space).

\paragraph{Comparison with offline model-based approaches.}
  In the same offline finite-horizon setting, the state-of-art model-based approach called {\sf PEVI-Adv} has been proposed by \citet{xie2021policy}, which also leverage the idea of reference-advantage decomposition. In comparison with {\sf PEVI-Adv}, \LCBQR not only enjoys the flexibility of model-free approaches, but also achieves optimal sample complexity for a broader range of target accuracy level $\varepsilon$. More precisely, the $\varepsilon$-range for which the algorithm achieves sample optimality can be compared as follows:
  \begin{align}
    \underbrace{\varepsilon \leq \left(0, H^{-1}\right]}_{\text{ (Our \LCBQR)}} \quad \text{vs.} \quad \underbrace{\varepsilon \leq \left(0, H^{-2.5}\right]}_{\text{ ({\sf PEVI-Adv})}},
  \end{align}
  offering an improvement by a factor of $H^{1.5}$.

\section{Analysis}
\label{sec:analysis}

In this section, we outline the main steps needed to establish the main results in Theorem~\ref{thm:lcb} and Theorem~\ref{thm:lcb-adv}.
Before proceeding,
let us first recall the following rescaled learning rates
\begin{equation}
	\eta_n = \frac{H+1}{H+n}
	\label{eq:eta-n-definition}
\end{equation}
for the $n$-th visit of a given state-action pair at a given time step $h$, which are adopted in both \LCBQ and \LCBQR.
For notational convenience, we further introduce two sequences of related quantities defined for any integers $N\geq 0$ and $n\geq 1$:  
\begin{equation}
	\label{equ:learning rate notation}
	\eta_0^N \defn \begin{cases} \prod_{i=1}^N(1-\eta_i) =0 , & \text{if }N>0, \\ 1, & \text{if }N=0,\end{cases}
		\qquad \text{and} \qquad   \eta_n^N \defn \begin{cases} \eta_n \prod_{i = n+1}^N(1-\eta_i), & \text{if }N>n, \\ \eta_n,  & \text{if }N=n,\\
	0,&\text{if } N<n.  \end{cases} 
\end{equation}
The following identity can be easily verified:  

\begin{align}
	\sum_{n=0}^N \eta_n^N = 1.
	\label{eq:sum-eta-n-N}
\end{align}

\begin{algorithm}[t]

        \SetKwFunction{FMain}{update-lcb-q}
  \SetKwProg{Fn}{Function}{:}{}
  \Fn{\FMain{}}{

        $Q_h^{\LCB}(s_h, a_h) \leftarrow (1 - \eta_n)Q_h^{\LCB}(s_h, a_h) + \eta_n\big(r(s_h, a_h) + V_{h+1}(s_{h+1}) - \cb \sqrt{\frac{H^3\iota^2}{n}} \big)$. \label{line:dadv4} \\
        
        }    

        \SetKwFunction{FMain}{update-lcb-q-ra}
  \SetKwProg{Fn}{Function}{:}{}
  \Fn{\FMain{}}{
    \small {\color{own_blue}\tcc{update the moment statistics of the interested terms}} \normalsize
   $[\refmu_h, \refsg_h, \advmu_h, \advsg_h](s_h,a_h) \leftarrow \texttt{update-moments()}$;  \\     
          \small {\color{own_blue}\tcc{update the bonus difference and accumulative bonus} }\normalsize
          $[\overline{\delta}_h , \overline{B}_h](s_h, a_h) \leftarrow \texttt{update-bonus()}$; \label{line:bonus_1}    

          $\overline{b}_h(s_h, a_h) \leftarrow \overline{B}_h(s_h, a_h) + (1-\eta_n) \frac{\overline{\delta}_h(s_h,a_h)}{\eta_n} + \cb \frac{H^{7/4}\iota}{n^{3/4}} + \cb\frac{H^2\iota}{n} $; \label{line:bonus_2}\\
           \small {\color{own_blue}\tcc{update the Q-estimate based on reference-advantage} }\normalsize   
          $\overline{Q}_h(s_h, a_h) \leftarrow (1 - \eta_n) \overline{Q}_h(s_h, a_h) + \eta_n\big(r_h(s_h, a_h) + V_{h+1}(s_{h+1}) - \overline{V}_{h+1}(s_{h+1}) + \overline{\mu}_h(s_h, a_h) - \overline{b}_h \big);$  \label{line:ref-q-update} \\

        }

    \SetKwFunction{FMain}{update-moments}
  \SetKwProg{Fn}{Function}{:}{}
  \Fn{\FMain{}}{

        ${\refmu_h(s_h, a_h) \leftarrow  (1-\tfrac{1}{n}) \refmu_h(s_h, a_h) + \tfrac{1}{n} \overline{V}^{\nnext}_{h+1}(s_{h+1})}$; \label{eq:line-mu-mean}
         {\small\color{own_blue}\tcp{mean of the reference} }\normalsize   \label{line:refmu_h}

        ${\refsg_h(s_h, a_h) \leftarrow  (1-\tfrac{1}{n}) \refsg_h(s_h, a_h) + \tfrac{1}{n}\big(\overline{V}^{\nnext}_{h+1}(s_{h+1}) \big)^2}$;
        \small {\color{own_blue}\tcp{$2^{\text{nd}}$ moment of the reference}} \normalsize  \label{line:refsigma_h}

        $\advmu_h(s_h, a_h) \leftarrow (1-\eta_n)\advmu_h(s_h, a_h) + \eta_n \big( V_{h+1}(s_{h+1}) - \overline{V}_{h+1}(s_{h+1}) \big)$;
        \small {\color{own_blue}\tcp{mean of the advantage}} \normalsize \label{line:advmu_h}
        
        $\advsg_h(s_h, a_h) \leftarrow  (1-\eta_n)\advsg_h (s_h, a_h) + \eta_n \big( V_{h+1}(s_{h+1}) - \overline{V}_{h+1}(s_{h+1}) \big)^2$. 
        \small {\color{own_blue}\tcp{$2^{\text{nd}}$ moment of the advantage} }\normalsize \label{line:advsigma_h}

    }

     \SetKwFunction{FMain}{update-bonus}
  \SetKwProg{Fn}{Function}{:}{}
  \Fn{\FMain{}}{
        
      {$\nextb_h(s_h, a_h) \leftarrow \cb\sqrt{\frac{\iota}{n}}\Big(\sqrt{ \refsg_h(s_h, a_h) - \big( \refmu_h(s_h, a_h) \big)^2} + \sqrt{H}\sqrt{\advsg_h(s_h, a_h) - \big( \advmu_h(s_h, a_h) \big)^2} \,\Big)$;} \label{eq:line-number-19} \\

        $\overline{\delta}_h(s_h,a_h) \leftarrow \nextb_h(s_h, a_h) - \overline{B}_h(s_h, a_h) ;$ \label{eq:line-delta} \\ $\overline{B}_h(s_h, a_h)  \leftarrow \nextb_h(s_h, a_h) .$
        }

\caption{Auxiliary functions}
\label{algo:subroutine}
\end{algorithm}


\subsection{Analysis of \LCBQ} 
\label{sec:LCBQ_analysis}

\begin{algorithm}[t]
 \textbf{Parameters:} some constant $\cb>0$, target success probability $1-\delta \in(0,1)$, and $\iota = \log\big(\frac{SAT}{\delta}\big)$. \\
\textbf{Initialize}  $Q^1_h(s, a) \leftarrow 0$; $N^1_h(s, a)\leftarrow 0$ for all $(s, a, h) \in \cS\times \cA \times [H]$; $V^1_h(s) \leftarrow 0$ for all $(s,h)\in \cS\times[H+1]$; $\pi^1$ s.t. $ \pi^1_h(s)= 1$ for all $(s,h)\in \cS\times [H]$. \\
\For{Episode $ k = 1$ \KwTo $K$}{
Sample the $k$-th trajectory $\{s_h^k, a_h^k, r_h^k\}_{h=1}^H$ from $\mathcal{D}_{\mu}$. {\small\color{own_blue}\tcp{sampling from batch dataset}} 
  \For{Step $ h = 1$ \KwTo $H$ }{
    \For{$(s,a)\in\cS\times \cA$}{   
   {\small\color{own_blue}\tcp{carry over the estimates and policy}} \normalsize
    $N_h^{k+1}(s, a) \leftarrow N_h^k(s, a)$;  ~
    $Q_h^{k+1}(s, a) \leftarrow Q_h^k(s, a)$; ~$V_h^{k+1}(s) \leftarrow V_h^k(s)$; $\pi_h^{k+1}(s) \leftarrow \pi_h^{k}(s)$.
    }
         $N_h^{k+1}(s_h^k, a_h^k) \leftarrow N_h^k(s_h^k, a_h^k) + 1$. {\small\color{own_blue}\tcp{update the counter} }
   
        $n \leftarrow N_h^{k+1}(s_h^k, a_h^k)$; 
        $\eta_n \leftarrow \frac{H + 1}{H + n}$.  
        {\small\color{own_blue}\tcp{update the learning rate}}      
        $b_n \leftarrow \cb \sqrt{\frac{H^3\iota^2}{n}}$.  \label{line:2}  {\small\color{own_blue}\tcp{update the bonus term}}
        {\small\color{own_blue}\tcp{update the Q-estimates with LCB} }  \normalsize
  $Q_h^{k+1}(s_h^k, a_h^k) \leftarrow Q_h^k(s_h^k, a_h^k) + \eta_n\Big\{ r_h(s_h^k, a_h^k) + V_{h+1}^k(s_{h+1}^k)-Q_h^k(s_h^k, a_h^k) - b_n \Big\} $. \label{line:lcb-9} \\
    
     {\small\color{own_blue}\tcp{update the value estimates}} \normalsize
      
    $V_{h}^{k+1}(s_h^k) \leftarrow \max\Big\{ V_{h}^{k}(s_h^k ), \,  \max_a Q_{h}^{k+1}(s_h^k, a)  \Big\}$. \label{line:lcb_v_update}\\
    {\small\color{own_blue}\tcp{update the policy} }  \normalsize
  If $V_{h}^{k+1}(s_h^k) =\max_a Q_{h}^{k+1}(s_h^k, a)$: update $\pi_h^{k+1}(s_h^k) =  \arg\max_a Q_{h}^{k+1}(s_h^k, a)$. \label{line:lcb-policy-update}
    \\
    \label{line:lcb-10}

 }
 }

\caption{\LCBQ for offline RL (a rewrite of Algorithm~\ref{algo:lcb-index} to specify dependency on $k$)}
\label{algo:lcb-index-k}
\end{algorithm}

To begin with, we intend to derive a recursive formula concerning the update rule of $Q_h^k$ --- the estimate of the Q-function at step $h$ at the beginning of the $k$-th episode. Note that we have omitted the dependency of all quantities on the episode index $k$ in Algorithm~\ref{algo:lcb-index}. 
For notational convenience and clearness, we rewrite Algorithm~\ref{algo:lcb-index} as Algorithm~\ref{algo:lcb-index-k} by specifying the dependency on the episode index $k$ and shall often use the following set of short-hand notation when it is clear from context.
\begin{itemize}
    \item $N_h^k(s, a)$, or the shorthand $N_h^k$:  the number of episodes that has visited $(s,a)$ at step $h$ before the beginning of the $k$-th episode.
    \item $k_h^n(s,a)$, or the shorthand $k^n$:  the index of the episode in which the state-action pair $(s,a)$ is visited  at step $h$ for the $n$-th times. We also adopt the convention that $k^0 = 0$.
    \item 
    $P_h^k\in\{0,1\}^{1\times S}$:  a row vector corresponding to the empirical transition at step $h$ of the $k$-th episode, namely,
    \begin{align} \label{eq:Phk_def}
	    P_h^k(s) = \ind\big( s= s_{h+1}^k \big) \qquad \text{for all }s\in \cS.
    \end{align}
    \item $\pi^{k}  = \{\pi_h^{k}\}_{h=1}^H$ with $\pi_h^{k}(s) \defn \arg\max_a Q_h^{k}(s,a), \forall (h,s)\in [H] \times \cS$: the deterministic greedy policy at the beginning of the $k$-th episode.
    \item $\widehat{\pi}$: the final output $\widehat{\pi}$ of Algorithms~\ref{algo:lcb-index} corresponds to $\pi^{K+1}$ defined above; for notational simplicity, we shall treat $\widehat{\pi}$ as $\pi^{K}$ in our analysis, which does not affect our result at all. 
\end{itemize}

Consider any state-action pair $(s,a)$. 
According to the update rule in line~\ref{line:lcb-9} of Algorithm~\ref{algo:lcb-index-k}, we can express (with the assistance of the above notation)
\begin{equation}
 Q_h^k(s, a) = Q_h^{k^{N_h^k} + 1} (s, a) =
	\big(1-\eta_{N_h^k} \big)Q_{h}^{k^{N_h^k}}(s, a) + \eta_{N_h^k} \Big\{ r_h(s,a) + V_{h+1}^{k^{N_h^k}}\big( s_{h+1}^{k^{N_h^k}} \big) -b_{N_h^k} \Big\},
	\label{eq:Qhk-sa-expansion-135}
 \end{equation}
where the first identity holds since $k^{N_h^k}$ denotes the latest episode prior to $k$ that visits $(s,a)$ at step $h$, 
and the learning rate is defined in \eqref{eq:eta-n-definition}. 
Note that it always holds that $k>k^{N_h^k}$. Applying the above relation \eqref{eq:Qhk-sa-expansion-135} recursively and using the notation \eqref{equ:learning rate notation} lead to
\begin{equation}\label{equ:Q-update}
    Q_h^k(s,a) = \eta_0^{N_h^k} Q_h^1(s, a) + \sum_{n=1}^{N_h^k}\eta_n^{N_h^k} \left(r_h(s,a) + V_{h+1}^{k^n}\big(s_{h+1}^{k^n} \big) -b_n \right). 
\end{equation}

As another important fact, the value estimate $V_h^k$ is monotonically non-decreasing in $k$,  i.e.,  
\begin{equation}\label{equ:monotone-lcb}
		V_h^{k +1}(s) \ge V_h^{k}(s)\qquad\text{for all }( s,k,h)\in \cS \times [K] \times [H],
\end{equation}
which is an immediate consequence of the update rule in line~\ref{line:lcb_v_update} of Algorithm~\ref{algo:lcb-index-k}. 
Crucially,  we observe that the iterate $V^{k}_h$ forms a ``pessimistic view'' of $V^{\pi^k}_h$ --- and in turn $V^{\star}_h$ --- resulting from suitable design of the penalty term. This observation is formally stated in the following lemma, with the proof postponed to Section~\ref{sec:proof-lem:vk-lower}. 
\begin{lemma} \label{lem:Vk-lower}
Consider any $\delta \in (0, 1)$, and suppose that $\cb >0$ is some sufficiently large constant. Then with probability at least $1-\delta$, 
\begin{align}
\label{equ:lcb-concentration-main}
  \Bigg|\sum_{n = 1}^{N_h^k(s,a)} \eta^{N_h^k(s,a)}_n \Big(P_{h, s,a} - P_h^{k^n(s,a)} \Big) V^{k^n(s,a)}_{h+1} \Bigg| 
	\leq \sum_{n=1}^{N_h^k(s,a)} \eta_n^{N_h^k(s,a)} b_n  
\end{align}
holds simultaneously for all  $(k, h, s,a) \in [K] \times [H] \times \cS\times \cA$, 
and
\begin{equation} \label{eq:Vk-lower}
     V_{h}^{k}(s)  \leq V_{h}^{\pi^k}(s) \leq V_h^{\star}(s)
\end{equation}
holds simultaneously for all $(k, h, s) \in [K] \times [H] \times \cS$. 
\end{lemma}
In a nutshell, the result \eqref{eq:Vk-lower} in Lemma~\ref{lem:Vk-lower} reveals that $V_h^k$ is a pointwise lower bound on $V_h^{\pi^k}$ and $ V_h^{\star}$, thereby forming a pessimistic estimate of the optimal value function. 
In addition,  the property \eqref{equ:lcb-concentration-main} in Lemma~\ref{lem:Vk-lower} essentially tells us that the weighted sum of the penalty terms dominates the weighted sum of the uncertainty terms, which plays a crucial role in ensuring the aforementioned pessimism property. 
As we shall see momentarily, Lemma~\ref{lem:Vk-lower} forms the basis of the subsequent proof.

We are now ready to embark on the analysis for \LCBQ, which is divided into multiple steps as follows.

\paragraph{Step 1: decomposing estimation errors.}

With the aid of Lemma~\ref{lem:Vk-lower},
we can develop an upper bound on the performance difference of interest in \eqref{equ:lcb-result} as follows
\begin{align}
	V_1^\star(\rho) - V_1^{\widehat{\pi}}(\rho) &= \mathop{\mathbb{E}}\limits _{s_{1}\sim\rho}\big[ V_1^\star(s_1) \big] - \mathop{\mathbb{E}}\limits _{s_{1}\sim\rho}\big[ V_1^{\pi^{K}}(s_1)  \big] \nonumber\\
	& \overset{\mathrm{(i)}}{\leq} \mathop{\mathbb{E}}\limits _{s_{1}\sim\rho}\big[ V_1^\star(s_1) \big] - \mathop{\mathbb{E}}\limits _{s_{1}\sim\rho}\big[ V_1^{K}(s_1) \big]\nonumber \\
	& \overset{\mathrm{(ii)}}{\leq}  \frac{1}{K}\sum_{k=1}^K \left(\mathop{\mathbb{E}}\limits _{s_{1}\sim\rho}\big[ V_1^\star(s_1) \big] - \mathop{\mathbb{E}}\limits _{s_{1}\sim\rho}\big[ V_1^{k}(s_1) \big] \right) \nonumber\\
	& = \frac{1}{K}\sum_{k=1}^K\sum_{s\in \cS} d_1^{\pi^\star}(s) \left(V_1^\star(s) - V_1^{k}(s)\right),\label{equ:regret2pac}
\end{align}
where (i) results from Lemma~\ref{lem:Vk-lower} (i.e., $V_1^{\pi^{K}}(s) \geq V_1^{K}(s) $ for all $s\in \cS$), 
 (ii) follows from the monotonicity property in \eqref{equ:monotone-lcb}, and the
last equality holds since $d_1^{\pi^{\star}}(s) = \rho(s)$ (cf.~\eqref{eq:d1-pi-s-deterministic}).

We then attempt to bound the quantity on the right-hand side of \eqref{equ:regret2pac}. 
Given that $\pi^{\star}$ is assumed to be a deterministic policy, we have $d_h^{\pi^\star}(s) = d_h^{\pi^\star}(s, \pi^{\star}(s))$. Taking this together with the relations  $V_h^{k}(s) \geq \max_a Q_h^k(s,a) \geq Q_h^k(s, \pi_h^\star(s))$ (see line~\ref{line:lcb_v_update} of Algorithm~\ref{algo:lcb-index-k}) and $V_h^\star(s)=Q_h^\star(s, \pi_h^\star(s))$, we obtain
\begin{align}
\sum_{k=1}^K \sum_{s\in \cS} d_h^{\pi^\star}(s) \left(V_h^\star(s) - V_h^{k}(s)\right) 
   & = \sum_{k=1}^K \sum_{s\in \cS} d_h^{\pi^\star}(s, \pi_h^\star(s)) \left(V_h^\star(s) - V_h^{k}(s)\right) \nonumber\\
    &\leq \sum_{k=1}^K \sum_{s\in \cS} d_h^{\pi^\star}(s, \pi_h^\star(s)) \Big(Q_h^\star\big(s, \pi_h^\star(s)\big) - Q_h^k\big(s, \pi_h^\star(s)\big)\Big) \nonumber\\
    &= \sum_{k=1}^K \sum_{(s, a)\in \cS \times \cA} d_h^{\pi^\star}(s, a) \left(Q_h^\star(s, a) - Q_h^{k}(s, a)\right) \label{eq:llama}
\end{align}
for any $h\in [H]$,
where the last identity holds since $\pi^\star$ is deterministic and hence 
\begin{equation} \label{eq:d_h_pi_star}
	d_h^{\pi^\star}(s,a) = 0 \qquad \text{for any }a\neq \pi^\star_h(s).
\end{equation}

In view of \eqref{eq:llama}, we need to properly control $Q_h^\star(s, a) - Q_h^{k}(s, a)$. 
By virtue of \eqref{eq:sum-eta-n-N}, we can rewrite $Q_h^\star(s, a)$ as follows
\begin{align}
Q_h^\star(s, a) & =    \sum_{n=0}^{N_h^k}\eta_n^{N_h^k}   Q_h^\star(s, a)  =  \eta_0^{N_h^k}  Q_h^{\star}(s, a) + \sum_{n=1}^{N_h^k}\eta_n^{N_h^k} Q_h^{\star}(s, a) \nonumber \\
& =  \eta_0^{N_h^k}  Q_h^{\star}(s, a) + \sum_{n=1}^{N_h^k}\eta_n^{N_h^k} \left(r_h(s,a) + P_{h,s, a} V_{h+1}^\star \right), \label{eq:Q_opt_decom}
\end{align}
where the second line follows from Bellman's optimality equation \eqref{eq:bellman_optimality}. 
Combining \eqref{equ:Q-update} and \eqref{eq:Q_opt_decom} leads to
\begin{align}
    &Q_h^\star(s, a) - Q_h^k(s, a) \nonumber\\
    & = \eta_0^{N_h^k}\left(Q_h^\star(s,a) - Q_h^1(s, a)\right) + \sum_{n=1}^{N_h^k}\eta_n^{N_h^k}\left( P_{h,s, a} V_{h+1}^\star - V_{h+1}^{k^n}(s_{h+1}^{k^n}) + b_n \right) \nonumber\\
    & =  \eta_0^{N_h^k}\left(Q_h^\star(s,a) - Q_h^1(s, a)\right)  + \sum_{n=1}^{N_h^k} \eta_n^{N_h^k} b_n + \sum_{n=1}^{N_h^k} \eta_n^{N_h^k} P_{h,s, a}\big( V_{h+1}^\star - V_{h+1}^{k^n}\big)   + \sum_{n=1}^{N_h^k} \eta_n^{N_h^k} \big(P_{h,s, a} - P_h^{k^n}\big)V_{h+1}^{k^n} 
	\label{eq:Qstar-Qk-decomposition-1}\\
    &\leq \eta_0^{N_h^k}H + 2\sum_{n=1}^{N_h^k} \eta_n^{N_h^k} b_n + \sum_{n=1}^{N_h^k} \eta_n^{N_h^k} P_{h,s,a}\big( V_{h+1}^\star - V_{h+1}^{k^n}\big), \label{eq:Qstar-Qk-decomposition-2}
\end{align}
where we have made use of the definition in \eqref{eq:Phk_def} by recognizing $P_h^{k^n}V_{h+1}^{k^n} = V_{h+1}^{k^n}(s_{h+1}^{k^n}) $ in \eqref{eq:Qstar-Qk-decomposition-1}, and the last inequality follows from the fact $Q_h^\star(s,a) - Q_h^1(s, a) = Q_h^\star(s,a) - 0\leq H$ and the bound \eqref{equ:lcb-concentration-main} in Lemma~\ref{lem:Vk-lower}.   
Substituting the above bound into \eqref{eq:llama},
we arrive at
\begin{align}
    \sum_{k=1}^K \sum_{s\in \cS} d_h^{\pi^\star}(s) \left(V_h^\star(s) - V_h^{k}(s)\right) & \leq 
    \underbrace{\sum_{k=1}^K \sum_{(s, a) \in \cS \times \cA} d_h^{\pi^\star}(s, a)\eta_0^{N_h^k(s,a)}H + 2\sum_{k=1}^K \sum_{(s, a) \in \cS \times \cA} d_h^{\pi^\star}(s, a) \sum_{n=1}^{N_h^k(s,a)} \eta_n^{N_h^k(s,a)} b_n}_{\eqqcolon\, I_h} \nonumber\\
    & \qquad + \sum_{k=1}^K  \sum_{(s, a) \in \cS \times \cA} d_h^{\pi^\star}(s, a) P_{h, s, a}\sum_{n=1}^{N_h^k(s,a)} \eta_n^{N_h^k(s,a)}\big(V_{h+1}^\star - V_{h+1}^{k_h^n(s,a)}\big).
	\label{equ:lcb-decompose-terms}
\end{align}

\paragraph{Step 2: establishing a crucial recursion.} 
As it turns out, the last term on the right-hand side of \eqref{equ:lcb-decompose-terms} can be used to derive a recursive relation that connects step $h$ with step $h+1$, 
as summarized in the next lemma.
\begin{lemma} 
	\label{lemma:recursion}
	With probability at least $1-\delta$, the following recursion holds:  
\begin{align}
	&\sum_{k=1}^K  \sum_{(s, a) \in \cS \times \cA} d_h^{\pi^\star}(s, a) P_{h, s, a}\sum_{n=1}^{N_h^k(s,a)} \eta_n^{N_h^k(s,a)}\big(V_{h+1}^\star - V_{h+1}^{k_h^n(s,a)}\big) \nonumber\\
	&\leq \left(1+\frac{1}{H}\right)\sum_{k=1}^K \sum_{s\in \cS} d_{h+1}^{\pi^\star}(s) \left(V_{h+1}^\star(s) - V_{h+1}^{k}(s)\right) + 24\sqrt{H^2C^\star K \log\frac{2H}{\delta}} + 12HC^\star\log\frac{2H}{\delta}. \label{equ:lcb-recursion}
\end{align} 
\end{lemma}
Lemma~\ref{lemma:recursion} taken together with \eqref{equ:lcb-decompose-terms} implies that
\begin{align}
    \sum_{k=1}^K \sum_{s\in \cS} d_h^{\pi^\star}(s) \left(V_h^\star(s) - V_h^{k}(s)\right) & \leq 
	\left(1+\frac{1}{H}\right)\sum_{k=1}^K \sum_{s\in \cS} d_{h+1}^{\pi^\star}(s) \left(V_{h+1}^\star(s) - V_{h+1}^{k}(s)\right) \notag \\
	& \qquad + I_h + 24\sqrt{H^2C^\star K \log\frac{2H}{\delta}} + 12HC^\star\log\frac{2H}{\delta} .
	\label{equ:lcb-decompose-terms-123}
\end{align}
Invoking \eqref{equ:lcb-decompose-terms-123} recursively over the time steps $h=H, H-1,\cdots, 1$ with the terminal condition $V^{k}_{H+1} = V_{H+1}^{\star} = 0$, we reach
\begin{align}\label{equ:summary_of_terms}
	\sum_{k=1}^K \sum_{s\in \cS} d_1^{\pi^\star}(s) \left(V_1^\star(s) - V_1^{k}(s)\right) 
	&\leq \max_{h\in[H]} \sum_{k=1}^K \sum_{s\in \cS} d_h^{\pi^\star}(s) \left(V_h^\star(s) - V_h^{k}(s)\right) \nonumber \\
	&\leq \sum_{h=1}^H \left(1+\frac{1}{H}\right)^{h - 1}\left( I_h + 24\sqrt{H^2C^\star K \log\frac{2H}{\delta}} + 12HC^\star\log\frac{2H}{\delta}\right) ,
\end{align} 
which captures the estimation error resulting from the use of pessimism principle.

\paragraph{Step 3: controlling the right-hand side of \eqref{equ:summary_of_terms}.} The right-hand side  of  \eqref{equ:summary_of_terms} can be bounded through the following lemma, which will be proved in Appendix~\ref{proof:lemma:lemma:lcb-bound-terms}.

\begin{lemma}\label{lemma:lcb-bound-terms}
Consider any $\delta \in (0,1)$. With probability at least $1-\delta$, we have  
\begin{align}
\sum_{h=1}^H \left(1+\frac{1}{H}\right)^{h - 1}\left( I_h + 24\sqrt{H^2C^\star K \log\frac{2H}{\delta}} + 12HC^\star\log\frac{2H}{\delta}\right)   & \lesssim   H^2 SC^\star \iota + \sqrt{H^5SC^\star K \iota^3}, \label{equ:lcb-final-a}
\end{align} 
where we recall that $\iota \coloneqq \log\big(\frac{SAT}{\delta}\big)$. 
\end{lemma}

Combining Lemma~\ref{lemma:lcb-bound-terms} with \eqref{equ:summary_of_terms} and \eqref{equ:regret2pac} yields
\begin{align}
V_1^\star(\rho) - V_1^{\widehat{\pi}}(\rho) &\leq  \frac{1}{K}\sum_{k=1}^K\sum_{s\in \cS} d_1^{\pi^\star}(s) \left(V_1^\star(s) - V_1^{k}(s)\right) \nonumber \\
&\leq \frac{1}{K}\max_{h\in[H]} \sum_{k=1}^K \sum_{s\in \cS} d_h^{\pi^\star}(s) \left(V_h^\star(s) - V_h^{k}(s)\right) \notag \\
	& \leq \frac{c_\mathrm{a}}{2} \sqrt{\frac{H^5SC^\star \iota^3}{K}} +  \frac{c_\mathrm{a}}{2} \frac{H^2 SC^\star \iota}{K} 
	 = \frac{c_\mathrm{a}}{2} \sqrt{\frac{H^6SC^\star \iota^3}{T}} +  \frac{c_\mathrm{a}}{2}\frac{H^3 SC^\star \iota}{T} \notag\\
	& \leq c_\mathrm{a} \sqrt{\frac{H^6SC^\star \iota^3}{T}}
	\label{eq:lcb-final-result-H-layers}
\end{align}
for some sufficiently large constant $c_{\mathrm{a}}>0$, where the last inequality is valid as long as $T> SC^{\star} \iota$. 
This concludes the proof of Theorem~\ref{thm:lcb}.

\subsection{Analysis of \LCBQR}

\begin{algorithm}[!t]
\textbf{Parameters:} number of epochs $M$, universal constant $\cb>0$, target success probability $1-\delta \in(0,1)$, and $\iota = \log\big(\frac{SAT}{\delta}\big)$. \\
\textbf{Initialize:}  \\
$Q_h^{1}(s, a), Q_h^{\LCB,1}(s, a), \overline{Q}^{1}_h(s, a), \overline{\mu}^{1}_h(s,a), \overline{\mu}^{\nnext,1}_h(s,a), N_h^1(s, a) \leftarrow 0$ for all $(s,a,h) \in \cS\times\cA \times [H]$;\\
$V^{1}_h(s), \overline{V}^{1}_h(s), \overline{V}_h^{\nnext,1}(s) \leftarrow 0$ for all $(s,h) \in \cS \times [H+1]$;\\
$\mu^{\re, 1}_h(s, a), \sigma^{\re, 1}_h(s, a), \mu^{\adv, 1}_h(s, a)$, $\sigma^{\adv, 1}_h(s, a)$, $\overline{\delta}^{1}_h(s, a), \overline{B}^{1}_h(s, a) \leftarrow 0$ for all $(s, a, h) \in \cS\times \cA \times [H]$.\\[0.4em]

\For{Epoch $ m = 1$ \KwTo $M$}{

  $L_m= 2^m$. \small {\color{own_blue}\tcp{specify the number of episodes in the current epoch} }\normalsize

  $\widehat{N}^{(m,1)}_h(s,a) = 0$ for all $(h, s,a) \in [H] \times \cS\times \cA.$ \small {\color{own_blue}\tcp{reset the epoch-wise counter} }\normalsize
  
  \small {\color{own_blue}\tcc{Inner-loop: update value-estimates $V_h(s,a)$ and Q-estimates $Q_h(s,a)$ }} \normalsize
  \For{In-epoch Episode $ t = 1$ \KwTo $L_m$}{
  	Set  $k \leftarrow \sum_{i=1}^{m-1} L_i + t.$ {\small\color{own_blue}\tcp{set the episode index}} 
  
    Sample the $k$-th trajectory $\{s_h^k, a_h^k, r_h^k\}_{h=1}^H$. {\small\color{own_blue}\tcp{sampling from batch dataset}}

	Compute $\pi^k$ s.t.~$\pi_h^k(s) = \arg\max_a Q^{k}_h(s,a)$ for all $(s,h)\in \cS\times [H]$.  \small {\color{own_blue}\tcp{update the policy}}\normalsize

    \For{Step $ h = 1$ \KwTo $H$}{

	    \For{$(s,a)\in\cS\times \cA$}{   
   {\small\color{own_blue}\tcp{carry over the estimates}} \normalsize
    $N_h^{k+1}(s, a) \leftarrow N_h^k(s, a)$; $\widehat{N}_h^{k+1}(s, a) \leftarrow \widehat{N}_h^k(s, a)$;  ~$V_h^{k+1}(s) \leftarrow V_h^k(s)$;\\
    $Q_h^{\LCB, k+1}(s,a) \leftarrow Q_h^{\LCB, k}(s,a)$ ~$\overline{Q}^{k+1}_h(s,a) \leftarrow \overline{Q}^{k}_h (s,a)$; ~$Q_h^{k+1}(s, a) \leftarrow Q_h^k(s, a)$;
    $\overline{V}^{k+1}_h(s) \leftarrow \overline{V}^{k}_h(s)$ ~$\overline{V}^{\nnext, k+1}_h(s) \leftarrow \overline{V}^{\nnext, k}_h(s)$; ~$\overline{\mu}^{k+1}(s,a) \leftarrow \overline{\mu}^{k}(s,a)$.
    }
    
        $N_h^{k+1}(s_h^k, a_h^k) \leftarrow N^k_h(s_h^k, a_h^k) + 1$; $n \leftarrow N^{k+1}_h(s_h^k, a_h^k)$. 
        {\small\color{own_blue}\tcp{update the overall counter} }

        $\eta_n \leftarrow \frac{H + 1}{H + n}$.  \small {\color{own_blue}\tcp{update the learning rate}} \normalsize

        {\small\color{own_blue}\tcp{update the Q-estimate with LCB}} \normalsize 
        $Q_h^{\LCB, k+1}(s_h^k, a_h^k) \leftarrow \texttt{update-lcb-q()}$. \\ 

       {\small\color{own_blue}\tcp{update the Q-estimate with LCB and reference-advantage}} \normalsize 
        $\overline{Q}^{k+1}_h(s_h^k, a_h^k) \leftarrow \texttt{update-lcb-q-ra()}$.

    \small {\color{own_blue}\tcp{update the Q-estimate $Q_h$ and value estimate $V_h$}} \normalsize
    $Q^{k+1}_h(s_h^k, a_h^k) \leftarrow \max \big\{Q_h^{\LCB,k+1}(s_h^k, a_h^k), \overline{Q}^{k+1}_h(s_h^k, a_h^k), Q^k_h(s_h^k, a_h^k) \big\}.$ \label{eq:line-q-update-k}\\
    $V^{k+1}_{h}(s_h^k) \leftarrow \max_a Q^{k+1}_h(s_h^k, a)$. \label{eq:line-v-update-k}\\

    \small {\color{own_blue}\tcp{update epoch-wise counter and $\overline{\mu}^{\nnext}_h(s,a)$ for the next epoch}} \normalsize
    $\widehat{N}^{(m,t+1)}_h(s_h^k, a_h^k) \leftarrow \widehat{N}^{(m,t)}_h(s_h^k, a_h^k) + 1$.\\
    $\overline{\mu}_h^{\nnext, k+1}(s_h^k, a_h^k) \leftarrow \left(1-\frac{1}{\widehat{N}^{(m,t+1)}_h(s_h^k, a_h^k)}\right)\overline{\mu}_h^{\nnext,k}(s_h, a_h) + \frac{1}{\widehat{N}^{(m,t+1)}_h(s_h^k, a_h^k)} \overline{V}^{\nnext,k}_{h+1}(s_{h+1})$.  \label{line:ref-mean-update-k} 
    }
 }
 \small {\color{own_blue}\tcc{Update the reference ($\overline{V}_h$, $\overline{V}^\nnext_h$) and ($\overline{\mu}_h$, $\overline{\mu}^\nnext_h$)} } \normalsize

 \For{$(s,a,h) \in \cS\times\cA\times [H+1]$}{

$\overline{V}^{k+1}_h(s) \leftarrow \overline{V}_h^{\nnext, k+1}(s)$; $\overline{\mu}^{k+1}_h(s,a)\leftarrow \overline{\mu}_h^{\nnext, k+1}(s,a)$. \label{eq:update-mu-reference-v-k}   {\small\color{own_blue}\tcp{set $\overline{V}_h$ and $\overline{\mu}_h$ for the next epoch} }

$\overline{V}_h^{\nnext, k+1}(s) \leftarrow V^{k+1}_h(s)$; $\overline{\mu}_h^{\nnext, k+1}(s,a) \leftarrow 0$. {\small\color{own_blue}\tcp{set $\overline{\mu}_h^\nnext$ and $\overline{V}_h^\nnext$ for the next epoch} } \label{eq:update-mu-reference-v-next-k}
}
 
 }
\KwOut{the policy $\widehat{\pi} = \pi^{K}$ with $K=\sum_{m=1}^{M} L_m$.}

\caption{\LCBQR (a rewrite of Algorithm~\ref{alg:lcb-advantage-per-epoch} that specifies dependency on $k$ or $(m,t)$.)}
\label{alg:lcb-advantage-per-epoch-k}
\end{algorithm}
We now turn to the analysis of \LCBQR. Thus far, we have omitted the dependency of all quantities on the epoch number $m$ and the in-epoch episode number $t$ in Algorithms~\ref{alg:lcb-advantage-per-epoch} and \ref{algo:subroutine}. 
While it allows for a more concise description of our algorithm, it might hamper the clarity of our proofs.
In the following, we introduce the notation $k$ to denote the current episode as follows: 
\begin{align}
\label{eqn:k-to-m-t}
	k \defn \sum_{i=1}^{m-1} L_i + t,
\end{align}
which corresponds to the $t$-th in-epoch episode in the $m$-th epoch; here, 
 $L_m=2^m $ stands for the total number of in-epoch episodes in the $m$-th epoch.  
With this notation in place, we can rewrite Algorithm~\ref{alg:lcb-advantage-per-epoch} as Algorithm~\ref{alg:lcb-advantage-per-epoch-k} 
in order to make clear the dependency on the episode index $k$, epoch number $m$, and in-epoch episode index $t$.

Before embarking on our main proof, we make two crucial observations which play important roles in our subsequent analysis. 
First, similar to the property \eqref{equ:monotone-lcb} for \LCBQ, the update rule (cf.~lines~\ref{eq:line-q-update}-\ref{eq:line-v-update} of Algorithm~\ref{alg:lcb-advantage-per-epoch-k}) ensures the monotonic non-decreasing property of $V_h(s)$ such that for all $k\in[K]$,
		\begin{equation}\label{equ:monotone-lcb-adv}
		V_h^{k +1}(s) \ge V_h^{k}(s),\qquad\text{for all } (k,s,h)\in[K]  \times \cS \times [H].
		\end{equation}
Secondly, $V^{k}_h$ forms a ``pessimistic view'' of $V^{\star}_h$, which is formalized in the lemma below; the proof is deferred to Appendix~\ref{proof:lemma-mono-lcb-adv}. 
\begin{lemma} \label{lem:lcb-adv-lower}
Let $\delta \in (0, 1)$. Suppose that $\cb >0$ is some sufficiently large constant. Then with probability at least $1-\delta$, the value estimates produced by Algorithm~\ref{alg:lcb-advantage-per-epoch} satisfy
\begin{equation} \label{eq:lcb-adv-lower}
     V_{h}^{k}(s)  \leq V_{h}^{\pi^{k}}(s) \leq V^{\star}(s)
\end{equation}
for all $(k,h,s) \in [K] \times [H+1] \times \cS$.
\end{lemma}

With these two observations in place, we can proceed to present the analysis for \LCBQR.
To begin with,  
the performance difference of interest can be controlled similar to \eqref{equ:regret2pac} as follows:
\begin{align}
	V_1^\star(\rho) - V_1^{\widehat{\pi}}(\rho) &= \mathop{\mathbb{E}}\limits _{s_{1}\sim\rho}\big[ V_1^\star(s_1) \big] - \mathop{\mathbb{E}}\limits _{s_{1}\sim\rho}\big[ V_1^{\pi^{K}}(s_1)  \big] \notag\\
	& \overset{\mathrm{(i)}}{\leq} \mathop{\mathbb{E}}\limits _{s_{1}\sim\rho}\big[ V_1^\star(s_1) \big] - \mathop{\mathbb{E}}\limits _{s_{1}\sim\rho}\big[ V_1^{K}(s_1) \big]\nonumber \\
	& \overset{\mathrm{(ii)}}{\leq}  \frac{1}{K}\sum_{k=1}^K \left(\mathop{\mathbb{E}}\limits _{s_{1}\sim\rho}\big[ V_1^\star(s_1) \big] - \mathop{\mathbb{E}}\limits _{s_{1}\sim\rho}\big[ V_1^{k}(s_1) \big] \right) \nonumber\\
	& = \frac{1}{K}\sum_{k=1}^K\sum_{s\in \cS} d_1^{\pi^\star}(s) \left(V_1^\star(s) - V_1^{k}(s)\right),
	\label{equ:regret2pac-2}
\end{align}
where (i) follows from Lemma~\ref{lem:lcb-adv-lower} (i.e., $V_1^{\pi^{K}}(s) \geq V_1^{K}(s)$ for all $s\in\cS$), (ii) holds due to the monotonicity in \eqref{equ:monotone-lcb-adv} and the
last equality holds since $d_1^{\pi^{\star}}(s) = \rho(s)$ (cf.~\eqref{eq:d1-pi-s-deterministic}). 
It then boils down to controlling the right-hand side of \eqref{equ:regret2pac-2}.
Towards this end, it turns out that one can control a more general counterpart, i.e.,  
\begin{align}
    \sum_{k=1}^K \sum_{s\in \cS} d_h^{\pi^\star}(s) \left(V_h^\star(s) - V_h^{k}(s)\right)
\end{align}
for any $h\in [H]$. 
This is accomplished via the following
lemma, whose proof is postponed to Appendix~\ref{proof:lem:lcb-adv-decompose}. 

\begin{lemma}\label{lem:lcb-adv-each-h}
Let $\delta \in (0, 1)$, and recall that $\iota \coloneqq \log\big(\frac{SAT}{\delta}\big)$. Suppose that $c_\mathrm{a}, \cb  >0$ are some sufficiently large constants. Then with probability at least $1-\delta$, one has
	\begin{align}
	\sum_{k=1}^K \sum_{s\in \cS} d_h^{\pi^\star}(s) \left(V_h^\star(s) - V_h^{k}(s)\right) \leq J_h^1 + J_h^2 + J_h^3,
	\end{align}
where
	\begin{align}
		J_h^1 &:=\sum_{k=1}^K \sum_{s, a \in \cS \times \cA} d_h^{\pi^\star}(s, a)\left[\eta_0^{N_h^k(s,a)}H + \frac{4\cb H^{7/4} \iota}{\left(N_h^k(s,a) \vee 1 \right)^{3/4} } + \frac{4\cb H^{2} \iota}{N_h^k(s,a) \vee 1  }\right], \nonumber\\
		J_h^2 &:= 2\sum_{k=1}^K \sum_{s, a \in \cS \times \cA} d_h^{\pi^\star}(s, a)\overline{\sumb}_h^{k}(s,a), \nonumber \\
		J_h^3 &\defn  \left(1+\frac{1}{H}\right)\sum_{k=1}^K \sum_{s\in \cS} d_{h+1}^{\pi^\star}(s) \left(V_{h+1}^\star(s) - V_{h+1}^{k}(s)\right) + 48\sqrt{HC^\star K \log\frac{2H}{\delta} } + 28 c_\mathrm{a} H^{3}  C^\star \sqrt{S} \iota^2. \label{eq:Jh123}
		\end{align}
\end{lemma}
As a  direct consequence of Lemma~\ref{lem:lcb-adv-each-h}, one arrives at a recursive relationship between time steps $h$ and $h+1$ as follows:
\begin{align}
	&\sum_{k=1}^K \sum_{s\in \cS} d_h^{\pi^\star}(s) \left(V_h^\star(s) - V_h^{k}(s)\right) \nonumber \\
	&\leq \left(1+\frac{1}{H}\right)\sum_{k=1}^K \sum_{s\in \cS} d_{h+1}^{\pi^\star}(s) \left(V_{h+1}^\star(s) - V_{h+1}^{k}(s)\right) + 48\sqrt{HC^\star K \log\frac{2H}{\delta} } + 28 c_\mathrm{a} H^{3}  C^\star \sqrt{S} \iota^2 + J_h^1 + J_h^2. \label{eq:lcb-adv-recursive-equ}
\end{align}
Recursing over time steps $h=H, H-1,\cdots,1$ with the terminal condition $V^{k}_{H+1} = V^{\star}_{H+1} = 0$,
we can upper bound the performance difference at $h=1$ as follows
\begin{align}
	\sum_{k=1}^K \sum_{s\in \cS} d_1^{\pi^\star}(s) \left(V_1^\star(s) - V_1^{k}(s)\right) 	&\leq \max_{h\in[H]} \sum_{k=1}^K \sum_{s\in \cS} d_h^{\pi^\star}(s) \left(V_h^\star(s) - V_h^{k}(s)\right) \nonumber \\
	& \leq \sum_{h=1}^H \left(1+\frac{1}{H}\right)^{h-1} \left(48\sqrt{HC^\star K \log\frac{2H}{\delta} } + 28 c_\mathrm{a} H^{3}  C^\star \sqrt{S} \iota^2 + J_h^1 + J_h^2 \right). \label{lem:lcb-adv-decompose}
\end{align}

To finish up, it suffices to upper bound each term in \eqref{lem:lcb-adv-decompose} separately. 
We summarize their respective upper bounds as follows; the proof is provided in Appendix~\ref{proof:lemma:lcb-adv-bound-terms}.

\begin{lemma}\label{lemma:lcb-adv-bound-terms}
Fix $\delta \in (0,1)$, and recall that $\iota \coloneqq \log\big(\frac{SAT}{\delta}\big)$. With probability at least $1-\delta$, we have 
\begin{subequations} \label{eq:lemma6}
\begin{align}
& \sum_{h=1}^H \left(1+\frac{1}{H}\right)^{h-1} J_h^1 \lesssim H^{2.75}(SC^\star)^{\frac{3}{4}}  K^{\frac{1}{4}}\iota^2 +  H^3 SC^\star \iota^3 , \label{eq:lemma6-a}\\
&\sum_{h=1}^H \left(1+\frac{1}{H}\right)^{h-1} J_h^2  \lesssim  \sqrt{H^4SC^\star \iota^3  \max_{h\in [H] } \sum_{k=1}^K \sum_{s\in \cS} d_{h}^{\pi^\star}(s) \left(V^{\star}_{h}(s) - V^{k}_{h}(s)\right) } + \sqrt{H^3S C^\star K \iota^5} + H^{4} S C^\star \iota^4, \label{eq:lemma6-b}\\
&\sum_{h=1}^H \left(1+\frac{1}{H}\right)^{h-1} \left(48\sqrt{HC^\star K \log\frac{2H}{\delta} } + 28 c_\mathrm{a} H^{3}  C^\star \sqrt{S} \iota^2 \right) \lesssim \sqrt{H^3C^\star K \log\frac{2H}{\delta}} + H^{4}C^\star \sqrt{S} \iota^2 .\label{eq:lemma6-c}
\end{align}
\end{subequations}
\end{lemma}

Substituting the above upper bounds into \eqref{equ:regret2pac-2} and \eqref{lem:lcb-adv-decompose} and recalling that $T=HK$, we arrive at
\begin{align}
&V_{1}^{\star}(\rho)-V_{1}^{\widehat{\pi}}(\rho) 
 \lesssim \frac{1}{K}\max_{h\in[H]} \sum_{k=1}^K \sum_{s\in \cS} d_h^{\pi^\star}(s) \left(V_h^\star(s) - V_h^{k}(s)\right) \notag \\
& \lesssim \frac{1}{K} \left(\sqrt{H^4SC^\star \iota^3  \max_{h\in [H] } \sum_{k=1}^K \sum_{s\in \cS} d_{h}^{\pi^\star}(s) \left(V^{\star}_{h}(s) - V^{k}_{h}(s)\right) } + \left(\sqrt{H^{3}SC^{\star}K\iota^{5}}+H^{4}SC^{\star}\iota^{4}+H^{2.75}(SC^{\star})^{\frac{3}{4}}K^{\frac{1}{4}}\iota^{2}\right) \right) \notag\\
 & \overset{\mathrm{(i)}}{\asymp}\frac{1}{K}\left(\sqrt{H^4SC^\star \iota^3  \max_{h\in [H] } \sum_{k=1}^K \sum_{s\in \cS} d_{h}^{\pi^\star}(s) \left(V^{\star}_{h}(s) - V^{k}_{h}(s)\right) } + \sqrt{H^{3}SC^{\star}K\iota^{5}}+H^{4}SC^{\star}\iota^{4}\right) \nonumber\\
 & \overset{\mathrm{(ii)}}{\lesssim} \frac{1}{K}\left(\sqrt{H^{3}SC^{\star}K\iota^{5}}+H^{4}SC^{\star}\iota^{4}\right)  \nonumber\\
 & \asymp\sqrt{\frac{H^{4}SC^{\star}\iota^{5}}{T}}+\frac{H^{5}SC^{\star}\iota^{4}}{T}, \nonumber
\end{align}
where (i) has made use of the AM-GM inequality:
\[
2H^{2.75}(SC^{\star})^{\frac{3}{4}}K^{\frac{1}{4}}\leq\left(H^{0.75}(SC^{\star})^{\frac{1}{4}}K^{\frac{1}{4}}\right)^{2}+\left(H^{2}(SC^{\star})^{\frac{1}{2}}\right)^{2}=\sqrt{H^{3}SC^{\star}K}+H^{4}SC^{\star},
\]
and (ii) holds by letting $x \defn \max_{h\in [H] } \sum_{k=1}^K \sum_{s\in \cS} d_{h}^{\pi^\star}(s) \left(V^{\star}_{h}(s) - V^{k}_{h}(s)\right)$ and solving the inequality $x \lesssim \sqrt{H^4SC^\star \iota^3 x} + \sqrt{H^{3}SC^{\star}K\iota^{5}}+H^{4}SC^{\star}\iota^{4}$. 
This concludes the proof.

\section{Discussions}
\label{sec:discussions}
Focusing on model-free paradigms, this paper has developed near-optimal sample complexities for some variants of pessimistic Q-learning algorithms --- armed with lower confidence bounds and variance reduction --- for offline RL. These sample complexity results, taken together with the analysis framework developed herein, open up a few exciting directions for future research.
 For example, the pessimistic Q-learning algorithms can be deployed in conjunction with their optimistic counterparts (e.g., \citet{jin2018q,li2021breaking,zhang2020almost}), when additional online data can be acquired to fine-tune the policy \citep{xie2021policy}. In addition, the $\varepsilon$-range for \LCBQR to attain sample optimality remains somewhat limited (i.e., $\varepsilon \in (0,1/H])$). Our recent work \citet{li2022settling} suggests that a new variant of pessimistic model-based algorithm is sample-optimal for a broader range of $\varepsilon$, which in turn motivates further investigation into whether model-free algorithms can accommodate a broader $\varepsilon$-range too without compromising sample efficiency.   
Moving beyond the tabular setting, it would be of great importance to extend the algorithmic and theoretical framework to accommodate low-complexity function approximation \citep{nguyen2021sample}.

\section*{Acknowledgements}
 
L.~Shi and Y.~Chi are supported in part by the grants ONR N00014-19-1-2404, NSF CCF-2106778, CCF-2007911, and the CAREER award ECCS-1818571. Y.~Wei is supported in part by the NSF grants CCF-2106778, DMS-2147546/2015447, and the CAREER award DMS-2143215. 
Y.~Chen is supported in part by the Alfred P.~Sloan Research Fellowship, the Google Research Scholar Award,
the AFOSR grants FA9550-19-1-0030 and FA9550-22-1-0198, the ONR grant N00014-22-1-2354,  
and the NSF grants CCF-2221009, CCF-1907661, IIS-2218713, DMS-2014279 and IIS-2218773. 
L.~Shi is also gratefully supported by the Leo Finzi Memorial Fellowship, Wei Shen and Xuehong Zhang Presidential Fellowship, and
Liang Ji-Dian Graduate Fellowship at Carnegie Mellon University.
 Part of this work was done while G.~Li, Y.~Wei and Y.~Chen were visiting the Simons Institute for the Theory of Computing.


\appendix

\section{Technical lemmas}

\subsection{Preliminary facts}
Our results rely heavily on proper choices of the learning rates. In what follows, 
we make note of several useful properties concerning the learning rates, which have been established in \cite{jin2018q,li2021breaking}.
\begin{lemma}[Lemma 1 in \citep{li2021breaking}]
	\label{lemma:property of learning rate} 
	For any integer $N>0$, the following properties hold:
\begin{subequations}
\label{eq:properties-learning-rates}
\begin{align}
	& \frac{1}{N^a}  \le\sum_{n=1}^{N}\frac{\eta_{n}^{N}}{n^a}\le\frac{2}{N^a} \qquad \mbox{for all}\quad  \frac{1}{2} \leq a \leq 1,
	\label{eq:properties-learning-rates-12}\\
\max_{1\le n\le N} & \eta_{n}^{N}\le\frac{2H}{N},\qquad\sum_{n=1}^{N}(\eta_{n}^{N})^{2}\le\frac{2H}{N}, \qquad
	 \sum_{N=n}^{\infty}\eta_{n}^{N}\le1+\frac{1}{H}.
	 \label{eq:properties-learning-rates-345}
\end{align}
\end{subequations}
 
\end{lemma}

In addition, we gather a few elementary properties about the Binomial distribution, which will be useful throughout the proof.
The lemma below is adapted from \citet[Lemma A.1]{xie2021policy}. 
\begin{lemma}\label{lem:binomial} 
  Suppose $N\sim \mathsf{Binomial}(n,p)$, where $n\geq 1$ and $p\in [0,1]$. For any $\delta\in (0,1)$, we have
  \begin{align}
    \frac{p}{N\vee 1} &\leq \frac{8\log\left(\frac{1}{\delta}\right)}{n},  \label{equ:binomial}
  \end{align}
and 
\begin{subequations}
  \label{equ:binomial-all}
  \begin{align}
    N&\geq \frac{np}{8\log\left(\frac{1}{\delta}\right)} \qquad \text{ if } np \geq 8\log\left(\frac{1}{\delta}\right), \label{equ:binomial2}\\
    N &\leq \begin{cases}
 e^2np & \text{ if } np \geq \log\left(\frac{1}{\delta}\right), \\ 
2e^2 \log\left(\frac{1}{\delta}\right) & \text{ if } np \leq 2\log\left(\frac{1}{\delta}\right).
\end{cases} \label{equ:binomial3}
  \end{align}
\end{subequations}
  with probability at least $1 - 4\delta$. 
\end{lemma}
\begin{proof}
To begin with, we directly invoke \citet[Lemma A.1]{xie2021policy} which yields the results in \eqref{equ:binomial} and \eqref{equ:binomial2}. Regarding \eqref{equ:binomial3}, invoking the Chernoff bound \citep[Theorem 2.3.1]{vershynin2018high} with $\mathbb{E}[N] = np$, when $np \geq \log\left(\frac{1}{\delta}\right)$, it satisfies
\begin{align*}
\mathbb{P}(N \geq e^2np) &\leq e^{-np}\left(\frac{enp}{e^2np}\right)^{ e^2np} \leq e^{-np} \leq \delta.
\end{align*}
Similarly, when $ np \leq 2\log\left(\frac{1}{\delta}\right)$, we have
\begin{align*}
\mathbb{P}\left(N \geq 2e^2\log\left(\frac{1}{\delta}\right)\right)& \overset{\mathrm{(i)}}{\leq} 
e^{-np}\left(\frac{enp}{2e^2\log\left(\frac{1}{\delta}\right)}\right)^{ 2e^2\log(\frac{1}{\delta})} \\
&\overset{\mathrm{(ii)}}{\leq} e^{-np}\left(\frac{enp}{e^2np}\right)^{ 2e^2\log(\frac{1}{\delta})} \leq e^{-2e^2 \log\left(\frac{1}{\delta}\right)} \leq \delta,
  \end{align*}
where (i) results from \citet[Theorem 2.3.1]{vershynin2018high}, and (ii) follows from the basic fact $e^2\log\left(\frac{1}{\delta}\right) \geq 2\log\left(\frac{1}{\delta}\right) \geq np$. 
Taking the union bound thus completes the proof. 
\end{proof}

\subsection{Freedman's inequality and its consequences}

Both the samples collected within each episode and the algorithms analyzed herein exhibit certain Markovian structure. 
As a result, concentration inequalities tailored to martingales become particularly effective for our analysis. 
In this subsection, we collect a few useful concentration results that will be applied multiple times in the current paper. These results might be of independent interest.

To begin with, the following theorem provides a user-friendly version of Freedman's
inequality \citep{freedman1975tail}; 
see \citet[Section C]{li2021tightening} for more details.
\begin{theorem}[Freedman's inequality]\label{thm:Freedman}
Consider a filtration $\mathcal{F}_0\subset \mathcal{F}_1 \subset \mathcal{F}_2 \subset \cdots$,
and let $\mathbb{E}_{k}$ stand for the expectation conditioned
on $\mathcal{F}_k$. 
Suppose that $Y_{n}=\sum_{k=1}^{n}X_{k}\in\mathbb{R}$,
where $\{X_{k}\}$ is a real-valued scalar sequence obeying

\[
  \left|X_{k}\right|\leq R\qquad\text{and}\qquad\mathbb{E}_{k-1} \big[X_{k}\big]=0\quad\quad\quad\text{for all }k\geq1
\]
for some quantity $R<\infty$. 
We also define
\[
W_{n}\coloneqq\sum_{k=1}^{n}\mathbb{E}_{k-1}\left[X_{k}^{2}\right].
\]
In addition, suppose that $W_{n}\leq\sigma^{2}$ holds deterministically for some given quantity $\sigma^2<\infty$.
Then for any positive integer $m \geq1$, with probability at least $1-\delta$
one has
\begin{equation}
\left|Y_{n}\right|\leq\sqrt{8\max\Big\{ W_{n},\frac{\sigma^{2}}{2^{m}}\Big\}\log\frac{2m}{\delta}}+\frac{4}{3}R\log\frac{2m}{\delta}.\label{eq:Freedman-random}
\end{equation}
\end{theorem}

We shall also record some immediate consequence of Freedman's inequality tailored to our problem.  
Recall that $N_h^i(s,a)$ denotes the number of times that $(s,a)$ has been visited at step $h$ before the beginning of the $i$-th episode, 
and $k^n(s,a)$ stands for the index of the episode in which $(s,a)$ is visited for the $n$-th time. The following concentration bound has been established in \citet[Lemma 7]{li2021breaking}.

\begin{lemma}
  \label{lemma:martingale-union-all}
  Let $\big\{ W_{h}^{i} \in \mathbb{R}^S \mid 1\leq i\leq K, 1\leq h \leq H+1 \big\}$ and $\big\{u_h^i(s,a,N)\in \mathbb{R} \mid 1\leq i\leq K, 1\leq h \leq H+1 \big\}$  be a collections of vectors and scalars, respectively, and suppose that they obey the following properties:
  \begin{itemize}
    \item $W_{h}^{i}$ is fully determined by the samples collected up to the end of the $(h-1)$-th step of the $i$-th episode;  
    \item $\|W_h^i\|_{\infty}\leq C_\mathrm{w}$; 
    \item $u_h^i(s,a, N)$  is fully determined by the samples collected up to the end of the $(h-1)$-th step of the $i$-th episode, and a given positive integer $N\in[K]$;
    \item $0\leq u_h^i(s,a, N) \leq C_{\mathrm{u}}$;
    \item $0\leq \sum_{n=1}^{N_h^k(s,a)} u_h^{k_h^n(s,a)}(s,a, N) \leq 2$.  
    
  \end{itemize}
  In addition, consider the following sequence  
  \begin{align}
    X_i (s,a,h,N) &\defn u_h^i(s,a, N) \big(P_h^{i} - P_{h,s,a}\big) W_{h+1}^{i} 
    \ind\big\{ (s_h^i, a_h^i) = (s,a)\big\},
    \qquad 1\leq i\leq K,  
  \end{align}
  with $P_h^{i}$ defined in \eqref{eq:Phk_def}. 
  Consider any $\delta \in (0,1)$. Then with probability at least $1-\delta$, 
  \begin{align}
    & \left|\sum_{i=1}^k X_i(s,a,h,N) \right| \notag\\
    & \quad \lesssim \sqrt{C_{\mathrm{u}} \log^2\frac{SAT}{\delta}}\sqrt{\sum_{n = 1}^{N_h^k(s,a)} u_h^{k_h^n(s,a)}(s,a,N) \Var_{h, s,a} \big(W_{h+1}^{k_h^n(s,a)}  \big)} + \left(C_{\mathrm{u}} C_{\mathrm{w}} + \sqrt{\frac{C_{\mathrm{u}}}{N}} C_{\mathrm{w}}\right) \log^2\frac{SAT}{\delta}
  \end{align}
  holds simultaneously for all $(k, h, s, a, N) \in [K] \times [H] \times \cS \times \cA \times [K]$.  
\end{lemma}

Next, we make note of an immediate consequence of Lemma~\ref{lemma:martingale-union-all} as follows.
\begin{lemma}\label{lemma:azuma-hoeffding}
Let $\big\{ W_{h}^{i} \in \mathbb{R}^S \mid 1\leq i\leq K, 1\leq h \leq H+1 \big\}$ be a collection of vectors satisfying the following properties: 
  \begin{itemize}
    \item $W_{h}^{i}$ is fully determined by the samples collected up to the end of the $(h-1)$-th step of the $i$-th episode;  
    \item $\|W_h^i\|_{\infty}\leq C_\mathrm{w}$.
  \end{itemize}

For any positive $N \geq H$, we consider the following sequence 
\begin{align}
  X_i(s,a,h,N) &\defn \eta_{N_h^i(s,a)}^N \big(P_h^{i} - P_{h,s,a}\big) W_{h+1}^{i}
    \ind\big\{ (s_h^i, a_h^i) = (s,a)\big\},
    \qquad 1\leq i\leq K,
\end{align}
with $P_h^{i}$ defined in \eqref{eq:Phk_def}. 
Consider any $\delta\in (0,1)$. With probability at least $1-\delta$, 
  \begin{align}
    & \left|\sum_{i=1}^k X_i(s,a,h,N) \right| \lesssim \sqrt{\frac{H}{N} C_{\mathrm{w}}^2 }\log^2\frac{SAT}{\delta} 
  \end{align}
holds simultaneously for all $(k, h, s, a, N) \in [K] \times [H] \times \cS \times \cA \times [K]$.  
\end{lemma}
\begin{proof}
Taking $u_{h}^{i}(s,a,N)=\eta_{N_{h}^{i}(s,a)}^{N}$, one can see from \eqref{eq:properties-learning-rates-345} in Lemma~\ref{lemma:property of learning rate} that
\[
	\big|u_{h}^{i}(s,a,N)\big|\leq\frac{2H}{N}\eqqcolon C_{\mathrm{u}}.
\]
Recognizing the trivial bound $\Var_{h,s,a}\big(W_{h+1}^{k_{h}^{n}(s,a)}\big)\leq C_{\mathrm{w}}^{2}$,
we can invoke Lemma~\ref{lemma:martingale-union-all} to obtain that, with probability at least $1-\delta$, 
\begin{align*}
  \Bigg|\sum_{i=1}^{k}X_{i}(s,a,h,N)\Bigg|
 & \lesssim\sqrt{C_{\mathrm{u}}\log^{2}\frac{SAT}{\delta}}\sqrt{\sum_{n=1}^{N_{h}^{k}(s,a)}\eta_{n}^{N}C_{\mathrm{w}}^{2}}+\left(C_{\mathrm{u}}C_{\mathrm{w}}+\sqrt{\frac{C_{\mathrm{u}}}{N}}C_{\mathrm{w}}\right)\log^{2}\frac{SAT}{\delta}\\
 & \lesssim\sqrt{\frac{H}{N}\log^{2}\frac{SAT}{\delta}}\cdot C_{\mathrm{w}}+\frac{HC_{\mathrm{w}}}{N}\log^{2}\frac{SAT}{\delta}\lesssim\sqrt{\frac{HC_{\mathrm{w}}^{2}}{N}}\log^{2}\frac{SAT}{\delta}
\end{align*}
holds simultaneously for all $(k, h, s, a, N) \in [K] \times [H] \times \cS \times \cA \times [K]$, where the last line applies \eqref{eq:properties-learning-rates-345} in Lemma~\ref{lemma:property of learning rate} once again.
\end{proof}

Finally, we introduce another lemma by invoking Freedman's inequality in Theorem~\ref{thm:Freedman}.
\begin{lemma}
  \label{lemma:martingale-union-recursion}
  Let $\big\{ W_{h}^{k}(s,a) \in \mathbb{R}^S \mid (s,a)\in\cS\times \cA, 1\leq k\leq K, 1\leq h \leq H+1 \big\}$ be a collection of vectors satisfying the following properties: 
  \begin{itemize}
    \item $W_{h}^{k}(s,a)$ is fully determined by the given state-action pair $(s,a)$ and the samples collected up to the end of the $(k-1)$-th episode;  
    \item $\|W_h^k(s,a)\|_{\infty}\leq C_\mathrm{w}$.
  \end{itemize}

For any positive $ C_{\mathrm{d}} \geq 0$, we consider the following sequences 
\begin{align}
  X_{h,k} &\defn C_{\mathrm{d}} \left[\frac{d_{h}^{\pi_{\star}}(s_h^k,a_h^k)}{d_{h}^{\mu}(s_h^k,a_h^k)} P_{h,s_h^k,a_h^k}W_{h+1}^{k}(s_h^k,a_h^k) - \sum_{(s,a)\in\cS\times\cA} d_{h}^{\pi_{\star}}(s,a) P_{h,s,a}W_{h+1}^{k}(s,a)\right],
    \qquad 1\leq k\leq K,\\
  \overline{X}_{h,k} &\defn C_{\mathrm{d}} \left[\frac{d_{h}^{\pi_{\star}}(s_h^k,a_h^k)}{d_{h}^{\mu}(s_h^k,a_h^k)} P_{h}^k W_{h+1}^{k}(s_h^k,a_h^k) - \sum_{(s,a)\in\cS\times\cA} d_{h}^{\pi_{\star}}(s,a) P_{h,s,a}W_{h+1}^{k}(s,a)\right],
    \qquad 1\leq k\leq K.
\end{align}
  Consider any $\delta \in (0,1)$. Then with probability at least $1-\delta$, 
  \begin{align}
    & \left|\sum_{k=1}^K X_{h,k}\right| \leq \sqrt{\sum_{k=1}^K 8 C_{\mathrm{d}}^2 C^\star \sum_{(s,a)\in\cS\times\cA} d_{h}^{\pi_{\star}}(s,a) \left[P_{h,s,a}W_{h+1}^{k}(s,a)\right]^2 \log\frac{2H}{\delta}} + 2C_{\mathrm{d}} C^\star C_\mathrm{w}\log\frac{2H}{\delta} \\
  &\left|\sum_{k=1}^K \overline{X}_{h,k}\right| \leq \sqrt{\sum_{k=1}^K 8 C_{\mathrm{d}}^2 C^\star \sum_{(s,a)\in\cS\times\cA} d_{h}^{\pi_{\star}}(s,a) P_{h,s,a}\left[W_{h+1}^{k}(s,a)\right]^2 \log\frac{2H}{\delta}} + 2C_{\mathrm{d}} C^\star C_\mathrm{w}\log\frac{2H}{\delta}
  \end{align}
  hold simultaneously for all $h\in[H]$.  
\end{lemma}
\begin{proof}
We intend to apply Freedman's inequality (cf.~Theorem~\ref{thm:Freedman}) to control $\sum_{k=1}^K X_{h,k}$. Considering any given time step $h$, it is easily verified that
  \begin{align*}
    \mathbb{E}_{k-1}[X_{h,k}] =0, \qquad \mathbb{E}_{k-1}[\overline{X}_{h,k}] =0,
  \end{align*}
where $\mathbb{E}_{k-1}$ denotes the expectation conditioned on everything happening up to the end of the $(k-1)$-th episode. To continue, we observe that
\begin{align}
  |X_{h,k}| \leq C_{\mathrm{d}}\left(\frac{d_{h}^{\pi_{\star}}(s_h^k,a_h^k)}{d_{h}^{\mu}(s_h^k,a_h^k)} + 1\right) \left\|W_{h+1}^k(s,a)\right\|_\infty  \leq 2C_{\mathrm{d}} C^\star C_\mathrm{w}, \label{eq:X-R}\\
  |\overline{X}_{h,k}| \leq C_{\mathrm{d}}\left(\frac{d_{h}^{\pi_{\star}}(s_h^k,a_h^k)}{d_{h}^{\mu}(s_h^k,a_h^k)} + 1\right) \left\|W_{h+1}^k(s,a)\right\|_\infty  \leq 2C_{\mathrm{d}} C^\star C_\mathrm{w}, \label{eq:X-R-2}
\end{align}
	where we use the assumptions $\frac{d_{h}^{\pi_{\star}}(s,a)}{d_{h}^{\mu}(s,a)} \leq C^\star$ for all $(h,s,a)\in[H]\times \cS\times \cA$ (cf.~Assumption~\ref{assumption}) and $\left\|W_{h+1}^k(s_h^k,a_h^k)\right\|_\infty  \leq C_{\mathrm{w}}$. 

Recall that $\Delta(\cS \times \cA)$ is the probability simplex over the set $\cS\times \cA$ of all state-action pairs, and we denote by $d_h^{\mu} \in \Delta(\cS \times \cA)$  the state-action visitation distribution induced by the behavior policy $\mu$ at time step $h\in[H]$. 
With this in hand, we obtain
\begin{align}
  \sum_{k=1}^K \mathbb{E}_{k-1}[|X_{h,k}|^2] &\leq \sum_{k=1}^K C_{\mathrm{d}}^2 \mathbb{E}_{k-1}\left[\frac{d_{h}^{\pi_{\star}}(s_h^k,a_h^k)}{d_{h}^{\mu}(s_h^k,a_h^k)} P_{h,s_h^k,a_h^k}W_{h+1}^{k}(s_h^k,a_h^k) - \sum_{(s,a)\in\cS\times\cA} d_{h}^{\pi_{\star}}(s,a) P_{h,s,a}W_{h+1}^{k}(s,a)\right]^2 \nonumber\\
  & \leq \sum_{k=1}^K C_{\mathrm{d}}^2 \mathbb{E}_{(s_h^k,a_h^k) \sim d_h^\mu}\left[\frac{d_{h}^{\pi_{\star}}(s_h^k,a_h^k)}{d_{h}^{\mu}(s_h^k,a_h^k)} P_{h,s_h^k,a_h^k}W_{h+1}^{k}(s_h^k,a_h^k)\right]^2 \nonumber\\
  & = \sum_{k=1}^K C_{\mathrm{d}}^2 \sum_{(s,a)\in\cS\times\cA} \frac{d_{h}^{\pi_{\star}}(s,a)}{d_{h}^{\mu}(s,a)}d_{h}^{\pi_{\star}}(s,a) \left[P_{h,s,a}W_{h+1}^{k}(s,a)\right]^2 \nonumber\\
  & \overset{\mathrm{(i)}}{\leq} \sum_{k=1}^K C_{\mathrm{d}}^2 C^\star\sum_{(s,a)\in\cS\times\cA} d_{h}^{\pi_{\star}}(s,a) \left[P_{h,s,a}W_{h+1}^{k}(s,a)\right]^2 \label{eq:var-of-w}  \\
  &\leq \sum_{k=1}^K C_{\mathrm{d}}^2 \sum_{(s,a)\in \cS\times\cA}C^\star d_{h}^{\pi_{\star}}(s,a) \left\|W_{h+1}^k(s_h^k,a_h^k)\right\|_\infty^2 \leq C_{\mathrm{d}}^2 C^\star C_{\mathrm{w}}^2 K, \label{eq:var-of-w-crude}
\end{align}
where (i) follows from $\frac{d_{h}^{\pi_{\star}}(s,a)}{d_{h}^{\mu}(s,a)} \leq C^\star$ (see Assumption~\ref{assumption}) and the assumption $\left\|W_{h+1}^k(s_h^k,a_h^k)\right\|_\infty  \leq C_{\mathrm{w}}$.

Similarly, we can derive
\begin{align}
  \sum_{k=1}^K \mathbb{E}_{k-1}[|\overline{X}_{h,k}|^2] &\leq \sum_{k=1}^K C_{\mathrm{d}}^2 \mathbb{E}_{k-1}\left[\frac{d_{h}^{\pi_{\star}}(s_h^k,a_h^k)}{d_{h}^{\mu}(s_h^k,a_h^k)} P_h^k W_{h+1}^{k}(s_h^k,a_h^k) - \sum_{(s,a)\in\cS\times\cA} d_{h}^{\pi_{\star}}(s,a) P_{h,s,a}W_{h+1}^{k}(s,a)\right]^2 \nonumber\\
  & \leq \sum_{k=1}^K C_{\mathrm{d}}^2 \mathbb{E}_{(s_h^k,a_h^k) \sim d_h^\mu} \left[ \mathbb{E}_{P_h^k \sim P_{h,s_h^k,a_h^k}}  \left[\frac{d_{h}^{\pi_{\star}}(s_h^k,a_h^k)}{d_{h}^{\mu}(s_h^k,a_h^k)} P_h^k W_{h+1}^{k}(s_h^k,a_h^k)\right]^2 \right] \nonumber\\
  & = \sum_{k=1}^K C_{\mathrm{d}}^2 \sum_{(s,a)\in\cS\times\cA} \frac{d_{h}^{\pi_{\star}}(s,a)}{d_{h}^{\mu}(s,a)}d_{h}^{\pi_{\star}}(s,a) \mathbb{E}_{P_h^k \sim P_{h,s,a}}\left[P_h^k W_{h+1}^{k}(s,a)\right]^2 \nonumber\\
  & \overset{\mathrm{(i)}}{\leq} \sum_{k=1}^K C_{\mathrm{d}}^2 C^\star\sum_{(s,a)\in\cS\times\cA} d_{h}^{\pi_{\star}}(s,a) \mathbb{E}_{P_h^k \sim P_{h,s,a}}\left[P_h^k W_{h+1}^{k}(s,a)\right]^2  \label{eq:var-of-w-2}\\
  & = \sum_{k=1}^K C_{\mathrm{d}}^2 C^\star\sum_{(s,a)\in\cS\times\cA} d_{h}^{\pi_{\star}}(s,a) P_{h,s,a}\left[W_{h+1}^{k}(s,a)\right]^2  \label{eq:var-of-w-2}\\
  &\leq \sum_{k=1}^K C_{\mathrm{d}}^2 \sum_{(s,a)\in \cS\times\cA}C^\star d_{h}^{\pi_{\star}}(s,a) \left\|W_{h+1}^k(s,a)\right\|_\infty^2 \leq C_{\mathrm{d}}^2 C^\star C_{\mathrm{w}}^2 K, \label{eq:var-of-w-crude-2}
\end{align}
where (i) follows from $\frac{d_{h}^{\pi_{\star}}(s,a)}{d_{h}^{\mu}(s,a)} \leq C^\star$ (see Assumption~\ref{assumption}) and the assumption $\left\|W_{h+1}^k(s_h^k,a_h^k)\right\|_\infty  \leq C_{\mathrm{w}}$.

Plugging in the results in \eqref{eq:X-R} and \eqref{eq:var-of-w} (resps.~\eqref{eq:X-R-2} and \eqref{eq:var-of-w-2}) to control $\sum_{k=1}^K \left|X_{h,k}\right|$ (resps.~$\sum_{k=1}^K \left|\overline{X}_{h,k}\right|$), we invoke Theorem~\ref{thm:Freedman} with $m=\lceil \log_2 K \rceil$ and take the union bound over $h\in[H]$ to show that with probability at least $1-\delta$,
\begin{align}
  \left|\sum_{k=1}^K X_{h,k}\right|& \leq\sqrt{8\max\left\{\sum_{k=1}^K C_{\mathrm{d}}^2C^\star \sum_{(s,a)\in\cS\times\cA}d_{h}^{\pi_{\star}}(s,a) \left[P_{h,s,a}W_{h+1}^{k}(s,a)\right]^2, \frac{C_{\mathrm{d}}^2 C^\star C_{\mathrm{w}}^2 K}{2^m}\right\} \log\frac{2H}{\delta}} \nonumber \\
  &\qquad + \frac{8}{3}C_{\mathrm{d}} C^\star C_\mathrm{w}\log\frac{2H}{\delta} \nonumber \\
  &\leq \sqrt{\sum_{k=1}^K 8 C_{\mathrm{d}}^2 C^\star \sum_{(s,a)\in\cS\times\cA} d_{h}^{\pi_{\star}}(s,a) \left[P_{h,s,a}W_{h+1}^{k}(s,a)\right]^2 \log\frac{2H}{\delta}} + 6C_{\mathrm{d}} C^\star C_\mathrm{w}\log\frac{2H}{\delta} \nonumber
\end{align}
and
\begin{align}
  \left|\sum_{k=1}^K \overline{X}_{h,k}\right|& \leq\sqrt{8\max\left\{\sum_{k=1}^K C_{\mathrm{d}}^2C^\star \sum_{(s,a)\in\cS\times\cA}d_{h}^{\pi_{\star}}(s,a) P_{h,s,a}\left[W_{h+1}^{k}(s,a)\right]^2, \frac{C_{\mathrm{d}}^2 C^\star C_{\mathrm{w}}^2 K}{2^m}\right\} \log\frac{2H}{\delta}} \nonumber \\
  &\qquad + \frac{8}{3}C_{\mathrm{d}} C^\star C_\mathrm{w}\log\frac{2H}{\delta} \nonumber \\
  &\leq \sqrt{\sum_{k=1}^K 8 C_{\mathrm{d}}^2 C^\star \sum_{(s,a)\in\cS\times\cA} d_{h}^{\pi_{\star}}(s,a) P_{h,s,a}\left[W_{h+1}^{k}(s,a)\right]^2 \log\frac{2H}{\delta}} + 6C_{\mathrm{d}} C^\star C_\mathrm{w}\log\frac{2H}{\delta} \nonumber
\end{align}
holds simultaneously for all $h\in[H]$.
\end{proof}

\section{Proof of main lemmas for \LCBQ (Theorem~\ref{thm:lcb})}
\label{proof:lcb-lemmas}

\subsection{Proof of Lemma~\ref{lem:Vk-lower}}
\label{sec:proof-lem:vk-lower}

\subsubsection{Proof of inequality~\eqref{equ:lcb-concentration-main}}

To begin with, we shall control $\sum_{n=1}^{N_h^k(s,a)}\eta_n^{N_h^k(s,a)} \big(P_{h,s, a} - P_h^{k^n(s,a)}\big)V_{h+1}^{k^n(s,a)}$ by invoking Lemma~\ref{lemma:azuma-hoeffding}. Let
\begin{align*}
  W_{h+1}^i \coloneqq V_{h+1}^{i},
\end{align*}
which satisfies
\begin{align*}
  \|W_{h+1}^i\|_{\infty} \leq H \eqqcolon C_{\mathrm{w}}.
\end{align*}
Applying Lemma~\ref{lemma:azuma-hoeffding} with $N = N_h^k(s,a)$ reveals that, with probability at least $1- \delta$, 
\begin{subequations}
\label{equ:lcb-concentration-1}
\begin{align}
 \Bigg|\sum_{n = 1}^{N_h^k(s,a)} \eta^{N_h^k(s,a)}_n \Big(P_{h, s,a} - P_h^{k^n(s,a)} \Big)V^{k^n(s,a)}_{h+1} \Bigg| = \left|\sum_{i=1}^k X_i\big(s,a,h,N_h^k(s,a)\big)\right| \leq 
\cb\sqrt{\frac{H^3 \iota^2}{N_h^k(s,a)}} 
\end{align}
holds simultaneously for all $(s, a,k, h) \in  \cS\times \cA \times [K] \times [H]$, provided that the constant $\cb>0$ is large enough and that $N_h^k(s,a) > 0$. If $N_h^k(s,a) = 0$, then we have the trivial bound
\begin{align}
	\Bigg| \sum_{n = 1}^{N_h^k(s,a)} \eta^{N_h^k(s,a)}_n \Big(P_{h, s,a} - P_h^{k^n(s,a)} \Big)V^{k^n(s,a)}_{h+1} \Bigg| = 0 .
\end{align}
\end{subequations}
Additionally, from the definition $b_n = \cb \sqrt{\frac{H^3 \iota^2}{n}}$, we observe that
\begin{equation}\label{equ:lcb-bn-bound}
  \begin{cases}
 \sum_{n=1}^{N_h^k(s,a)} \eta_n^{N_h^k(s,a)} b_n \in \Big[\cb \sqrt{\frac{H^3\iota^2}{N_h^k(s,a) }}, 2\cb \sqrt{\frac{H^3\iota^2}{N_h^k(s,a) }} \,\Big], & \text{ if } N_h^k(s,a) > 0 \\ 
 \sum_{n=1}^{N_h^k(s,a)} \eta_n^{N_h^k(s,a)} b_n  = 0, & \text{ if } N_h^k(s,a) = 0
\end{cases}
\end{equation}
holds simultaneously for all $s, a, h, k \in \cS\times \cA\times [H]\times [K]$, which follows directly from the property \eqref{eq:properties-learning-rates-12} in Lemma~\ref{lemma:property of learning rate}.

Combining the above bounds \eqref{equ:lcb-concentration-1} and \eqref{equ:lcb-bn-bound}, we arrive at the advertised result
\begin{align*}
  	\Bigg|\sum_{n = 1}^{N_h^k(s,a)} \eta^{N_h^k(s,a)}_n \Big(P_{h, s,a} - P_h^{k^n(s,a)} \Big)V^{k^n(s,a)}_{h+1} \Bigg| 
	\leq \sum_{n=1}^{N_h^k(s,a)} \eta_n^{N_h^k(s,a)} b_n.
\end{align*}

\subsubsection{Proof of inequality~\eqref{eq:Vk-lower}}\label{sec:proof-lcb-induction}

Note that the second inequality of \eqref{eq:Vk-lower} holds straightforwardly as 
\begin{equation*}
  V_h^{\pi}(s) \leq V^\star(s) 
  \end{equation*}
 holds for any policy $\pi$. As a consequence, it suffices to establish the first inequality of \eqref{eq:Vk-lower}, namely, 
 \begin{equation} \label{eq:Vk-lower-first}
	 V_{h}^{k}(s)  \leq V_{h}^{\pi^k}(s)  \qquad \text{for all } (s,h,k)\in \cS\times[H]\times [K].
\end{equation}

Before proceeding,  let us introduce the following auxiliary index
\begin{align}
	k_o(h,k,s) \coloneqq \max\Big\{l: l< k \text{ and } V_h^{l}(s) = \max_a Q_h^{l}(s,a) \Big\}
	\label{eq:defn-ko-h-k-s}
\end{align}
for any $(h,k,s)\in[H]\times[K]\times\cS$, 
which denotes the index of the latest episode --- before the end of the $(k-1)$-th episode --- in which $V_h(s)$ has been updated. 
In what follows, we shall often abbreviate $k_o(h,k,s)$ as $k_o(h)$ whenever it is clear from the context.

Towards establishing the relation~\eqref{eq:Vk-lower-first}, we proceed by means of an inductive argument. 
In what follows, we shall first justify the desired inequality for the base case when $h +1 = H+1$ for all episodes $k\in[K]$, 
and then use induction to complete the argument for other cases. 
More specifically, 
consider any step $h\in[H]$ in any episode $k\in[K]$, and suppose that the first inequality of \eqref{eq:Vk-lower} is satisfied for all previous episodes as well as all steps $h'\geq h+1$ in the current episode, namely,

\begin{subequations}\label{eq:lcb-lower-indunction-assumption-12}
\begin{align}
	V_{h'}^{k'}(s) &\leq V_{h'}^{\pi^{k'}}(s) \qquad \text{for all } (k',h',s) \in [k-1]\times [H+1]\times \cS, \label{eq:lcb-lower-indunction-assumption2}\\
	V_{h'}^k(s) &\leq V_{h'}^{\pi^k}(s) \qquad \text{for all } h'\geq h+1 \text{ and } s\in\cS. \label{eq:lcb-lower-indunction-assumption}
\end{align}
\end{subequations}

We intend to justify that the following is valid
  \begin{equation}\label{eq:lcb-lower-induction}
	  V_{h}^k(s) \leq V_{h}^{\pi^k}(s) \qquad \text{for all }s\in \cS,
  \end{equation}
assuming that the induction hypothesis \eqref{eq:lcb-lower-indunction-assumption-12} holds.

\paragraph{Step 1:  base case.} 

Let us begin with the base case when $h + 1 = H+1$ for all episodes $k\in[K]$.
Recognizing the fact that $V_{H+1}^{\pi}=V_{H+1}^{k}=0$ for any $\pi$ and any $k\in[K]$, we directly arrive at
\begin{align}
  V_{H+1}^k(s) &\leq V_{H+1}^{\pi^k}(s) \qquad \text{for all } (k,s)\in [K] \times \cS.
\end{align}

\paragraph{Step 2: induction.} To justify \eqref{eq:lcb-lower-induction} under the induction hypothesis 
\eqref{eq:lcb-lower-indunction-assumption-12}, we decompose the difference term 
to obtain 
\begin{align}
V_{h}^{\pi^{k}}(s)-V_{h}^{k}(s) & =V_{h}^{\pi^{k}}(s) - \max\big\{\max_{a}Q_{h}^{k}(s,a), V_{h}^{k-1}(s) \big\} \nonumber\\
&= Q_{h}^{\pi^{k}}\big(s,\pi^{k}_h(s)\big) - \max\big\{\max_{a}Q_{h}^{k}(s,a), V_{h}^{k_{o}(h)}(s) \big\},
	\label{equ:lcb-lower-decompose-terms}
\end{align}
where the last line holds since $V_h(s)$ has not been updated during episodes $k_o(h), k_o(h) +1,\cdots, k-1$ (in view of the definition of $k_o(h)$ in \eqref{eq:defn-ko-h-k-s}).
We shall prove that the right-hand side of \eqref{equ:lcb-lower-decompose-terms} is non-negative by discussing the following two cases separately.
\begin{itemize}
 \item Consider the case where $V_{h}^{k}(s) = \max_{a}Q_{h}^{k}(s,a)$. Before continuing, it is easily observed from the update rule in line~\ref{line:lcb-policy-update} and line~\ref{line:lcb_v_update} of Algorithm~\ref{algo:lcb-index} that: $V_h(s)$ and $\pi_h(s)$ are updated hand-in-hand for every $h$. Thus, it implies that 
    \begin{align}
      \pi^k_h(s) = \arg \max_a Q_h^k(s,a), \qquad \text{ when } V_h^k(s) = \max_{a}Q_{h}^{k}(s,a)  \label{eq:pi-k-base}
    \end{align}
holds for all $(k,h) \in[K] \times [H]$.
As a result, we express the term of interest as follows: 
\begin{align}
 V_{h}^{\pi^{k}}(s)-V_{h}^{k}(s) =  Q_{h}^{\pi^{k}}\big(s,\pi^{k}_h(s)\big) - \max_{a}Q_{h}^{k}(s,a) = Q_{h}^{\pi^{k}}\big(s,\pi^{k}_h(s)\big) - Q_{h}^{k}\big(s,\pi^{k}_h(s)\big). \label{eq:lcb-induction-h-case1}
 \end{align}
To continue, we turn to controlling a more general term $Q_{h}^{\pi^k}(s, a) - Q_{h}^{  k}(s, a)$ for all $(s,a)\in\cS\times\cA$. Invoking the fact $\eta_0^{N_h^k} + \sum_{n = 1}^{N_h^k} \eta_n^{N_h^k} = 1$ (see \eqref{equ:learning rate notation} and \eqref{eq:sum-eta-n-N}) leads to 
\begin{align*}
  Q_{h}^{\pi^k}(s, a)  = \eta^{N_h^k}_0 Q_{h}^{\pi^k} (s,a) + \sum_{n = 1}^{N_h^k} \eta^{N_h^k}_n Q_{h}^{\pi^k} (s,a).
\end{align*}
This relation combined with \eqref{equ:Q-update} allows us to express the difference between $Q_{h}^{\pi^k}$ and $Q_{h}^{  k}$ as follows
  \begin{align}
Q_{h}^{\pi^k}(s, a) - Q_{h}^{  k}(s, a)      & =  \eta_0^{N_h^k} \left(Q_{h}^{\pi^k}(s, a) - Q_h^1 (s, a) \right) + \sum_{n=1}^{N_h^k}\eta_n^{N_h^k} \left[ Q_{h}^{\pi^k}(s, a)- r_h(s,a)- V_{h+1}^{k^n}(s_{h+1}^{k^n})  + b_n\right] \nonumber \\
    & \overset{\mathrm{(i)}}{=}  \eta_0^{N_h^k} \left(Q_{h}^{\pi^k}(s, a) - Q_h^1 (s, a) \right) + \sum_{n=1}^{N_h^k}\eta_n^{N_h^k} \left[ P_{h,s, a }V_{h+1}^{\pi^k} - V_{h+1}^{k^n}(s_{h+1}^{k^n})  + b_n\right] \nonumber \\
    & \overset{\mathrm{(ii)}}{{\geq}}  \sum_{n=1}^{N_h^k}\eta_n^{N_h^k} \left[ P_{h,s, a }V_{h+1}^{\pi^k} - V_{h+1}^{k^n}(s_{h+1}^{k^n})  + b_n\right]  \nonumber\\
 & \overset{\mathrm{(iii)}}{{=}}  \sum_{n=1}^{N_h^k} \eta_n^{N_h^k} P_{h,s, a} \left( V_{h+1}^{\pi^k} - V_{h+1}^{k^n}\right) + \sum_{n=1}^{N_h^k}\eta_n^{N_h^k} \left[\left(P_{h,s, a} - P_h^{k^n}\right)V_{h+1}^{k^n} + b_n\right] \nonumber\\	  
    & \overset{\mathrm{(iv)}}{\geq} \sum_{n=1}^{N_h^k}\eta_n^{N_h^k} \left[\left(P_{h,s, a} - P_h^{k^n}\right)V_{h+1}^{k^n} + b_n \right] .
	  \label{equ:lcb-lower-1}
  \end{align}
Here, (i) invokes the Bellman equation $Q_{h}^{\pi^k}(s, a)  = r_h(s, a) + P_{h,s, a }V_{h+1}^{\pi^k}$; 
(ii) holds since $Q_{h}^{\pi^k}(s, a) \geq 0 = Q_h^1 (s, a)$;  
(iii) relies on the notaion \eqref{eq:Phk_def}; and (iv) comes from the fact 
 \begin{equation*}
  V_{h+1}^{\pi^k} \geq V_{h+1}^{k} \geq V_{h+1}^{k^n},
\end{equation*}
owing to the induction hypothesis in \eqref{eq:lcb-lower-indunction-assumption-12} as well as the monotonicity of $V_{h+1}$ in \eqref{equ:monotone-lcb}.
Consequently, it follows from \eqref{equ:lcb-lower-1} that  
\begin{align}\label{equ:lcb-lower-induction-result}
 &Q_{h}^{\pi^k}(s, a) - Q_{h}^{  k}(s, a) \geq \sum_{n = 1}^{N_h^k(s,a)} \eta^{N_h^k(s,a)}_n \left(P_{h, s,a} - P_h^{k^n(s,a)} \right)V^{k^n(s,a)}_{h+1}  + \sum_{n = 1}^{N_h^k(s,a)} \eta^{N_h^k(s,a)}_n b_n \nonumber \\
 & \qquad \geq  \sum_{n = 1}^{N_h^k(s,a)} \eta^{N_h^k(s,a)}_n b_n - \Bigg|\sum_{n = 1}^{N_h^k(s,a)} \eta^{N_h^k(s,a)}_n \left(P_{h, s,a} - P_h^{k^n(s,a)} \right)V^{k^n(s,a)}_{h+1} \Bigg| \geq 0
\end{align}
for all state-action pair $(s,a)$, where the last inequality holds due to the bound \eqref{equ:lcb-concentration-main} in Lemma~\ref{lem:Vk-lower}.  Plugging the above result into \eqref{eq:lcb-induction-h-case1} directly establishes that
\begin{align}
  V_{h}^{\pi^{k}}(s)-V_{h}^{k}(s) = Q_{h}^{\pi^k}\big(s, \pi^k(s) \big) - Q_{h}^{k} \big(s,\pi^k(s) \big) \geq 0. \label{equ:lcb-lower-new-1}
\end{align}

\item When $V_h^k(s) = V_h^{k_o(h)}(s)$, it indicates that
  \begin{align}
    V^{k_o(h)}_h(s) = \max_a Q^{k_o(h)}_h(s, a) , \qquad \pi^{k_o(h)}_h(s) = \arg\max_a Q_h^{k_o(h)}(s,a),
  \end{align}
which follows from the definition of $k_o(h)$ in \eqref{eq:defn-ko-h-k-s} and the corresponding fact in \eqref{eq:pi-k-base}. 
We also make note of the fact that
\begin{align}
  \pi^k_h(s) = \pi^{k_o(h)}_h(s),
\end{align}
which holds since $V_h(s)$ (and hence $\pi_h(s)$) has not been updated during episodes $k_o(h), k_o(h) +1,\cdots, k-1$ (in view of the definition \eqref{eq:defn-ko-h-k-s}).
Combining the above two results, we can show that
\begin{align}
	V_{h}^{\pi^{k}}(s)-V_{h}^{k}(s) &=Q_{h}^{\pi^k}\big( s, \pi^k_h(s) \big) - V_{h}^{k_o(h)}(s) = Q_{h}^{\pi^k}\big( s, \pi^k_h(s) \big) - \max_a Q^{k_o(h)}_h(s, a)  
	\nonumber \\
  & = Q_{h}^{\pi^k}\big(s, \pi^{{k_o(h)}}_h(s)\big) - Q_{h}^{k_o(h)}\big(s, \pi^{{k_o(h)}}_h(s)\big) \nonumber \\
	&\geq 0, \label{equ:lcb-lower-2}
  \end{align}
where the final line can be verified using exactly the same argument as in the previous case to show \eqref{equ:lcb-lower-1} and then \eqref{equ:lcb-lower-new-1}. Here, we omit the proof of this step for brevity.

\end{itemize}

To conclude, substituting the relations \eqref{equ:lcb-lower-new-1}  and \eqref{equ:lcb-lower-2} in the above two cases back into \eqref{equ:lcb-lower-decompose-terms}, we arrive at 
\[
	V_{h}^{\pi^{k}}(s)-V_{h}^{k}(s) \geq 0
\]
as desired in~\eqref{eq:lcb-lower-induction}. This immediately completes the induction argument.

\subsection{Proof of Lemma~\ref{lemma:recursion}}\label{proof:lemma-lcb-revursion}

We make the observation that Lemma~\ref{lemma:recursion} would follow immediately if we could establish the following relation: 
\begin{align}
	A_h & \defn \sum_{k=1}^K  \underbrace{\sum_{(s, a) \in \cS \times \cA} d_h^{\pi^\star}(s, a)  P_{h, s, a}\sum_{n=1}^{N_h^k(s,a)} \eta_n^{N_h^k(s,a)} \left(V_{h+1}^\star - V_{h+1}^{k^n(s,a)}\right)}_{\eqqcolon A_{h,k}} \nonumber \\
& \leq \sum_{k=1}^K \underbrace{\left(1+\frac{1}{H}\right) \sum_{s\in \cS} d_{h+1}^{\pi^\star}(s) \left(V_{h+1}^\star(s) - V_{h+1}^{k}(s)\right)}_{\eqqcolon B_{h,k}} + 24\sqrt{H^2C^\star K \log\frac{2H}{\delta}} + 12HC^\star\log\frac{2H}{\delta}.
\label{equ:lcb-final-b}
\end{align}
The remainder of the proof is thus dedicated to proving \eqref{equ:lcb-final-b}. 

To continue, let us first consider two auxiliary sequences $\{Y_{h,k}\}_{k=1}^K$ and $\{Z_{h,k}\}_{k=1}^K$ which are the empirical estimates of $A_{h,k}$ and $B_{h,k}$, respectively. For any time step $h$ in episode $k$, $Y_{h,k}$ and $Z_{h,k}$ are defined as follows  

\begin{align*}
  Y_{h,k} &:= \frac{d_{h}^{\pi^{\star}}(s_{h}^{k},a_{h}^{k})}{d_{h}^{\mu}(s_{h}^{k},a_{h}^{k})} P_{h,s_{h}^{k},a_{h}^{k}} \sum_{n=1}^{N_{h}^{k}(s_{h}^{k},a_{h}^{k})}\eta_{n}^{N_{h}^{k}(s_{h}^{k},a_{h}^{k})}\left(V_{h+1}^{\star}-V_{h+1}^{k^{n}(s_{h}^{k},a_{h}^{k})}\right),\nonumber \\
  Z_{h,k} &:= \left(1+\frac{1}{H}\right) \frac{d_{h}^{\pi^{\star}}(s_{h}^{k},a_{h}^{k})}{d_{h}^{\mu}(s_{h}^{k},a_{h}^{k})} P_{h,s_{h}^{k},a_{h}^{k}} \left(V_{h+1}^{\star}-V_{h+1}^{k}\right).
\end{align*}
To begin with,  let us establish the relationship between $\{Y_{h,k}\}_{k=1}^K$ and $\{Z_{h,k}\}_{k=1}^K$:
\begin{align}
\sum_{k=1}^K Y_{h,k} & =\sum_{k=1}^{K}\frac{d_{h}^{\pi^{\star}}(s_{h}^{k},a_{h}^{k})}{d_{h}^{\mu}(s_{h}^{k},a_{h}^{k})} P_{h,s_{h}^{k},a_{h}^{k}} \sum_{n=1}^{N_{h}^{k}(s_{h}^{k},a_{h}^{k})}\eta_{n}^{N_{h}^{k}(s_{h}^{k},a_{h}^{k})}\left(V_{h+1}^{\star}-V_{h+1}^{k^{n}(s_{h}^{k},a_{h}^{k})}\right)\nonumber\\
 & \stackrel{(\mathrm{i})}{=}\sum_{l=1}^{K}\frac{d_{h}^{\pi^{\star}}(s_{h}^{l},a_{h}^{l})}{d_{h}^{\mu}(s_{h}^{l},a_{h}^{l})} P_{h,s_{h}^{l},a_{h}^{l}} \left\{ \sum_{N=N_h^{l}(s_{h}^{l},a_{h}^{l})}^{N_h^{K}(s_{h}^{l},a_{h}^{l})}\eta_{N_h^{l}(s_{h}^{l},a_{h}^{l})}^{N}\right\} \left(V_{h+1}^{\star}-V_{h+1}^{l}\right)\\
 & \leq\left(1+\frac{1}{H}\right)\sum_{k=1}^{K}\frac{d_{h}^{\pi^{\star}}(s_{h}^{k},a_{h}^{k})}{d_{h}^{\mu}(s_{h}^{k},a_{h}^{k})} P_{h,s_{h}^{k},a_{h}^{k}} \left(V_{h+1}^{\star}-V_{h+1}^{k}\right)= \sum_{k=1}^K Z_{h,k}. 
	\label{lcb-C-1}
\end{align}
Here, (i) holds by replacing $k^{n}(s_{h}^{k},a_{h}^{k})$ with $l$ and gathering all terms that involve $V_{h+1}^{\star}-V_{h+1}^{l}$; 
in the last line, we have invoked the property $\sum_{N=n}^{N_h^K(s,a)} \eta_n^N \leq \sum_{N=n}^{\infty} \eta_n^N = 1 + 1/H$ (see \eqref{eq:properties-learning-rates-345}) together with the fact $V_{h+1}^\star - V_{h+1}^{l}\geq 0$ (see Lemma~\ref{lem:Vk-lower}), and have further replaced $l$ with $k$.

With the above relation in hand, in order to verify \eqref{equ:lcb-final-b}, we further decompose $A_h$ into several terms 
\begin{align}
  A_h &= \sum_{k=1}^K A_{h,k} = \sum_{k=1}^K Y_{h,k} + \sum_{k=1}^K \left(A_{h,k} - Y_{h,k}\right) \overset{\mathrm{(i)}}{\leq} \sum_{k=1}^K Z_{h,k} + \sum_{k=1}^K \left(A_{h,k} - Y_{h,k}\right) \nonumber \\
  &= \sum_{k=1}^K B_{h,k} + \sum_{k=1}^K \left(Z_{h,k} - B_{h,k}\right) + \sum_{k=1}^K \left(A_{h,k} - Y_{h,k}\right) \label{eq:recursion-extra-error}
\end{align}
where (i) follows from \eqref{lcb-C-1}.

As a result, it remains to control $\sum_{k=1}^K \left(Z_{h,k} - B_{h,k}\right)$ and $\sum_{k=1}^K \left(A_{h,k} - Y_{h,k}\right)$ separately in the following.

\paragraph{Step 1: controlling $\sum_{k=1}^K \left(A_{h,k} - Y_{h,k}\right)$.}
We shall first control this term by means of Lemma~\ref{lemma:martingale-union-recursion}. 
Specifically, consider 
\begin{align}
  W_{h+1}^{k}(s,a) \coloneqq \sum_{n=1}^{N_h^k(s,a)} \eta_n^{N_h^k(s,a)} \left(V_{h+1}^\star - V_{h+1}^{k^n(s,a)}\right), \qquad C_{\mathrm{d}} \coloneqq 1 \label{eq:Y-defn-w}
\end{align}
which satisfies 
\begin{align}
  \left\|W_{h+1}^{k}(s,a)\right\|_\infty \leq  \sum_{n=1}^{N_h^k(s,a)} \eta_n^{N_h^k(s,a)} \left(\left\|V_{h+1}^\star\right\|_\infty + \left\|V_{h+1}^{k^n(s,a)}\right\|_\infty \right) \leq 2H \eqqcolon C_{\mathrm{w}}. \label{eq:Y-cw}
\end{align}
Here we use the fact that $\eta_0^{N_h^k} + \sum_{n = 1}^{N_h^k} \eta_n^{N_h^k} = 1$ (see \eqref{equ:learning rate notation} and \eqref{eq:sum-eta-n-N}).
Then, applying Lemma~\ref{lemma:martingale-union-recursion} with \eqref{eq:Y-defn-w}, we have with probability at least $1-\delta$, the following inequality holds true 
\begin{align}
  \left|\sum_{k=1}^K \left(A_{h,k} - Y_{h,k}\right)\right| &\eqqcolon \left|\sum_{k=1}^K X_{h,k}\right| \nonumber \\
  & \leq  \sqrt{\sum_{k=1}^K 8C_{\mathrm{d}}^2 C^\star \sum_{(s,a)\in\cS\times\cA} d_{h}^{\pi^\star}(s,a) \left[P_{h,s,a}W_{h+1}^{k}(s,a)\right]^2 \log\frac{2H}{\delta}} + 2C_{\mathrm{d}} C^\star C_\mathrm{w}\log\frac{2H}{\delta} \nonumber \\
  &\overset{\mathrm{(i)}}{\leq} \sqrt{8C^\star\sum_{k=1}^K \left\|W_{h+1}^{k}(s,a)\right\|_\infty^2 \log\frac{2H}{\delta}} + 4HC^\star\log\frac{2H}{\delta} \nonumber \\
  &\leq 8\sqrt{H^2C^\star K \log\frac{2H}{\delta}} + 4HC^\star\log\frac{2H}{\delta}, \label{eq:bound-of-Y}
\end{align}
where (i) holds since $\big| P_{h,s,a}W_{h+1}^{k}(s,a)\big| \leq  \big\| P_{h, s,a}\big\|_1 \|W_{h+1}^{k}(s,a) \|_{\infty} =\|W_{h+1}^{k}(s,a) \|_{\infty}$.

\paragraph{Step 2: controlling $\sum_{k=1}^K \left(Z_{h,k} - B_{h,k}\right)$.} Similarly, we shall control $\sum_{k=1}^K \left(Z_{h,k} - B_{h,k}\right)$ by invoking Lemma~\ref{lemma:martingale-union-recursion}.

Recall that
\begin{align}
  Z_{h,k} - B_{h,k} = \left(1+\frac{1}{H}\right) \frac{d_{h}^{\pi^\star}(s_{h}^{k},a_{h}^{k})}{d_{h}^{\mu}(s_{h}^{k},a_{h}^{k})} P_{h,s_{h}^{k},a_{h}^{k}} \left(V_{h+1}^{\star}-V_{h+1}^{k}\right) - \left(1+\frac{1}{H}\right) \sum_{s\in \cS} d_{h+1}^{\pi^\star}(s) \left(V_{h+1}^\star(s) - V_{h+1}^{k}(s)\right),
\end{align}
and let us consider 
\begin{align}
  W_{h+1}^{k}(s,a) \coloneqq V_{h+1}^\star - V_{h+1}^{k}, \qquad C_{\mathrm{d}} \coloneqq \left(1+\frac{1}{H}\right)\leq 2\label{eq:Z-defn-w}
\end{align}
which satisfies 
\begin{align}
  \left\|W_{h+1}^{k}(s,a)\right\|_\infty \leq \left\|V_{h+1}^\star\right\|_\infty + \left\|V_{h+1}^{k}\right\|_\infty \leq 2H \eqqcolon C_{\mathrm{w}} .
\end{align}
Again, in view of Lemma~\ref{lemma:martingale-union-recursion}, we can show that with probability at least $1-\delta$,
\begin{align}
  \left|\sum_{k=1}^K \left(B_{h,k} - Z_{h,k}\right)\right| &= \left|\sum_{k=1}^K X_{h,k}\right| \nonumber \\
  & \leq  \sqrt{\sum_{k=1}^K 8C_{\mathrm{d}}^2 C^\star \sum_{(s,a)\in\cS\times\cA} d_{h}^{\pi^\star}(s,a) \left[P_{h,s,a}W_{h+1}^{k}(s,a)\right]^2 \log\frac{2H}{\delta}} + 2C_{\mathrm{d}} C^\star C_\mathrm{w}\log\frac{2H}{\delta} \nonumber \\
  &\stackrel{(\mathrm{i})}{\leq} \sqrt{32C^\star\sum_{k=1}^K \left\|W_{h+1}^{k}(s,a)\right\|_\infty^2 \log\frac{2H}{\delta}} + 8HC^\star\log\frac{2H}{\delta}  \nonumber\\
  &\leq 16\sqrt{H^2C^\star K \log\frac{2H}{\delta}} + 8HC^\star\log\frac{2H}{\delta}, \label{eq:bound-of-Z}
\end{align}
where (i) holds due to the fact $\big\| P_{h, s,a}\big\|_1=1$.

\paragraph{Step 3: putting all this together.}
Substitution results in \eqref{eq:bound-of-Y} and \eqref{eq:bound-of-Z} back into \eqref{eq:recursion-extra-error} completes the proof of \eqref{equ:lcb-final-b} as follows
\begin{align*}
  A_h &\leq \sum_{k=1}^K B_{h,k} + \bigg|\sum_{k=1}^K \left(Z_{h,k} - B_{h,k}\right)\bigg| + \bigg|\sum_{k=1}^K \left(A_{h,k} - Y_{h,k}\right)\bigg| \\
	&\leq \sum_{k=1}^K B_{h,k} + 24\sqrt{H^2C^\star K \log\frac{2H}{\delta}} + 12HC^\star\log\frac{2H}{\delta}.
\end{align*}

This in turn concludes the proof of Lemma \ref{lemma:recursion}. 

\subsection{Proof of Lemma~\ref{lemma:lcb-bound-terms}}
\label{proof:lemma:lemma:lcb-bound-terms}

Recall that the term of interest in \eqref{equ:summary_of_terms} is given by
 \begin{align}\label{eq:lcb-lemma3-all}
  \sum_{h=1}^H \left(1+\frac{1}{H}\right)^{h - 1}\left( 24\sqrt{H^2C^\star K \log\frac{2H}{\delta}} + 12HC^\star\log\frac{2H}{\delta}\right) + \sum_{h=1}^H \left(1+\frac{1}{H}\right)^{h - 1}I_h.
 \end{align}
First, it is easily seen that
\begin{equation}\label{equ:algebra property}
  \left(1+\frac{1}{H}\right)^{h-1} \leq \left(1+\frac{1}{H}\right)^{H} \leq e \qquad \text{for every } h = 1, \cdots, H,
\end{equation}
which taken collectively with the expression of the first term in \eqref{eq:lcb-lemma3-all} yields
\begin{align}
  \sum_{h=1}^H \left(1+\frac{1}{H}\right)^{h - 1}\left( 24\sqrt{H^2C^\star K \log\frac{2H}{\delta}} + 12HC^\star\log\frac{2H}{\delta}\right) &\leq 24e \sum_{h=1}^H \left(\sqrt{H^2C^\star K \log\frac{2H}{\delta}} + HC^\star\log\frac{2H}{\delta}\right) \nonumber \\
  & \lesssim \sqrt{H^4C^\star K \log\frac{H}{\delta}} + H^2C^\star\log\frac{H}{\delta}. \label{eq:lcb-lemma3-first-term}
\end{align}

As a result, it remains to control the second term in \eqref{eq:lcb-lemma3-all}. 
Plugging the expression of $I_h$ (cf.~\eqref{equ:lcb-decompose-terms}) and invoking the fact \eqref{equ:algebra property} give
\begin{align}
    \sum_{h=1}^H \left(1+\frac{1}{H}\right)^{h - 1} I_h &=  \sum_{h=1}^H \left(1+\frac{1}{H}\right)^{h-1} \sum_{k=1}^K \sum_{ (s, a) \in \cS \times \cA} d_h^{\pi^\star}(s, a)\eta_0^{N_h^k(s,a)}H \nonumber\\
   & \qquad + 2 \sum_{h=1}^H \left(1+\frac{1}{H}\right)^{h-1}  \sum_{k=1}^K \sum_{ (s, a) \in \cS \times \cA} d_h^{\pi^\star}(s, a) \sum_{n=1}^{N_h^k(s,a)} \eta_n^{N_h^k(s,a)} b_n \nonumber \\
   &\leq \underbrace{e\sum_{h=1}^H \sum_{k=1}^K \sum_{(s, a) \in \cS \times \cA} d_h^{\pi^\star}(s, a)\eta_0^{N_h^k(s,a)}H}_{\eqqcolon A} + \underbrace{2e \sum_{h=1}^H \sum_{k=1}^K \sum_{(s, a) \in \cS \times \cA} d_h^{\pi^\star}(s, a) \sum_{n=1}^{N_h^k(s,a)} \eta_n^{N_h^k(s,a)} b_n}_{ \eqqcolon B}. \label{equ:lcb-I1}
\end{align}

\paragraph{Step 1: controlling the quantities $A$ and $B$ in \eqref{equ:lcb-I1}.} 

We first develop an upper bound on the quantity $A$ in \eqref{equ:lcb-I1}.
Recognizing the fact that $\eta_0^N =0$ for any $N> 0$ (see \eqref{equ:learning rate notation}), we have 
\begin{align}
A & =e\sum_{h=1}^{H}\sum_{k=1}^{K}\sum_{(s,a)\in\cS\times\cA}d_{h}^{\pi^{\star}}(s,a)\eta_{0}^{N_{h}^{k}(s,a)}H\nonumber\\
 & \leq eH \sum_{h=1}^H\sum_{(s,a)\in\cS\times\cA}d_{h}^{\pi^{\star}}(s,a)\sum_{k=1}^{K}\mathds{1}\big(N_{h}^{k}(s,a)<1\big)\nonumber\\
 & \leq eH \sum_{h=1}^H\sum_{(s,a)\in\cS\times\cA}d_{h}^{\pi^{\star}}(s,a)\frac{8\iota}{d_{h}^{\mu}(s,a)}+eH \sum_{h=1}^H \sum_{(s,a)\in\cS\times\cA}d_{h}^{\pi^{\star}}(s,a)\sum_{k=\lceil\frac{8\iota}{d_{h}^{\mu}(s,a)}\rceil}^{K}\mathds{1}\big(N_{h}^{k}(s,a)<1\big)\nonumber\\
 & =eH \sum_{h=1}^H \sum_{s\in\cS}d_{h}^{\pi^{\star}}\big(s,\pi^{\star}(s)\big)\frac{8\iota}{d_{h}^{\mu}\big(s,\pi^{\star}(s)\big)}+eH \sum_{h=1}^H\sum_{s\in\cS}d_{h}^{\pi^{\star}}\big(s,\pi^{\star}(s)\big)\sum_{k=\lceil \frac{8\iota}{d_{h}^{\mu}(s,\pi^{\star}(s))}\rceil}^{K}\mathds{1}\big(N_{h}^{k}\big(s,\pi^{\star}(s)\big)<1\big),\nonumber
\end{align}
where the last equality holds since $\pi^{\star}$ is a deterministic policy (so that $d_h^{\pi^\star}(s, a)\neq 0$ only when $a=\pi^\star(s)$).
Recalling $\frac{d_h^{\pi^\star}(s,a)}{d_h^{\mu}(s,a)} \leq C^\star$ under Assumption~\ref{assumption}, we can further bound $A$ by 
\begin{align}
A & \leq 8eH^{2}SC^{\star}\iota+eH \sum_{h=1}^H \sum_{s\in\cS}d_{h}^{\pi^{\star}}\big(s,\pi^{\star}(s)\big)\sum_{k=\lceil \frac{8\iota}{d_{h}^{\mu}(s,\pi^{\star}(s))}\rceil}^{K}\mathds{1}\big(N_{h}^{k}\big(s,\pi^{\star}(s)\big)<1\big)\nonumber\\
 & = 8eH^{2}SC^{\star}\iota,
  \label{equ:visit-cover-bound}
\end{align}
where the last inequality follows since when $k\geq \frac{8\iota}{d_h^{\mu}(s,a)}$, one has --- with probability at least $1-\delta$ --- that
\begin{align*}
    N_h^k(s,a) \geq \frac{k d_h^{\mu}(s,a)}{8\iota} \geq 1,
\end{align*}
holds simultaneously for all $(s,a, h, k)\in \cS\times \cA\times [K]\times [H]$ (as implied by \eqref{equ:binomial2}).

Turning to the quantity $B$ in \eqref{equ:lcb-I1}, one can deduce that
\begin{align}
B & =2e\sum_{h=1}^{H}\sum_{k=1}^{K}\sum_{(s,a)\in\cS\times\cA}d_{h}^{\pi^{\star}}(s,a)\sum_{n=1}^{N_{h}^{k}(s,a)}\eta_{n}^{N_{h}^{k}(s,a)}b_{n}\nonumber\\
 & \lesssim\sum_{h=1}^{H}\sum_{k=1}^{K}\sum_{(s,a)\in\cS\times\cA}d_{h}^{\pi^{\star}}(s,a)\sqrt{\frac{H^{3}\iota^{2}}{N_{h}^{k}(s,a)\vee1}}  =\sum_{h=1}^{H}\sum_{k=1}^{K}\sum_{s\in\cS}d_{h}^{\pi^{\star}}\big(s,\pi^{\star}(s)\big)
 \sqrt{\frac{H^{3}\iota^{2}}{N_{h}^{k}\big(s,\pi^{\star}(s)\big)\vee1}},  
\end{align}
where the inequality follows from inequality~\eqref{equ:lcb-bn-bound}, and the last equality is valid since $\pi^{\star}$ is a deterministic policy.

To further control the right hand side above, Lemma~\ref{lem:binomial}  provides an upper bound for 
$\sqrt{1/ \big( N_{h}^{k}\big(s,\pi^{\star}(s)\big)\vee1 \big)}$ which in turn leads to 
\begin{align}
B & \lesssim\sqrt{H^{3}\iota^{3}}\sum_{h=1}^{H}\sum_{k=1}^{K}\sum_{s\in\cS}d_{h}^{\pi^{\star}}\big(s,\pi^{\star}(s)\big)\sqrt{\frac{1}{kd_{h}^{\mu}\big(s,\pi^{\star}(s)\big)}}\nonumber\\
 & \lesssim\sqrt{H^{3}C^{\star}\iota^{3}}\sum_{h=1}^{H}\sum_{k=1}^{K}\sum_{s\in\cS}\sqrt{d_{h}^{\pi^{\star}}\big(s,\pi^{\star}(s)\big)}\sqrt{\frac{1}{k}}\nonumber\\
 & \lesssim\sqrt{H^{5}C^{\star}K\iota^{3}}\max_h\sum_{s\in\cS}\sqrt{d_{h}^{\pi^{\star}}\big(s,\pi^{\star}(s)\big)} \notag\\
 & \lesssim\sqrt{H^{5}C^{\star}K\iota^{3}}\cdot\Bigg(\sqrt{S}\cdot\sqrt{\sum_{s\in\cS}d_{h}^{\pi^{\star}}\big(s,\pi^{\star}(s)\big)}\Bigg)\asymp\sqrt{H^{5}SC^{\star}K\iota^{3}},
\end{align}
where the second inequality follows from the fact $\frac{d_h^{\pi^\star}(s,a)}{d_h^{\mu}(s,a)} \leq C^\star$ under Assumption~\ref{assumption}, and the last line invokes the Cauchy-Schwarz inequality.

Taking the upper bounds on both $A$ and $B$ collectively establishes
\begin{align}
\sum_{h=1}^{H}\left(1+\frac{1}{H}\right)^{h-1}I_{h}\leq A+B\lesssim H^{2}SC^{\star}\iota+\sqrt{H^{5}SC^{\star}K\iota^{3}}.  \label{eq:lcb-Ih-bound}
\end{align}

\paragraph{Step 2: putting everything together.} 
Combining \eqref{eq:lcb-lemma3-first-term} and \eqref{eq:lcb-Ih-bound} allows us to establish that 
\begin{align*}
  \sum_{h=1}^H \left(1+\frac{1}{H}\right)^{h - 1}\left( I_h + 16\sqrt{H^2C^\star K \log\frac{2H}{\delta}} + 8HC^\star\log\frac{2H}{\delta}\right) \lesssim H^{2}SC^{\star}\iota+\sqrt{H^{5}SC^{\star}K\iota^{3}}, 
\end{align*}
as advertised.

\section{Proof of lemmas for \LCBQR (Theorem~\ref{thm:lcb-adv})}\label{proof:lcb-adv-lemmas}

\paragraph{Additional notation for \LCBQR.}\label{para:lcb-adv-add-notation}
Let us also introduce, and remind the reader of, several notation of interest in Algorithm~\ref{alg:lcb-advantage-per-epoch-k} as follows.
\begin{itemize}
	\item $N_h^{k}(s,a)$ (resp.~$N_h^{(m,t)}(s,a)$) denotes the value of $N_h(s,a)$ --- the number of episodes that has visited $(s,a)$ at step $h$ at the {\em beginning} of the $k$-th episode (resp.~the {\em beginning} of $t$-th episode of the $m$-th epoch);
	for the sake of conciseness, we shall often abbreviate $N_h^k = N_h^k(s,a)$ (resp.~$N_h^{(m,t)} = N_h^{(m,t)}(s,a)$) when it is clear from context.

	\item $L_m=2^m $: the total number of in-epoch episodes in the $m$-th epoch.

	\item $k^n_h(s,a)$: the index of the episode in which $(s, a)$ is visited for the $n$-th time at time step $h$; $\left(m^n_h(s,a), t^n_h(s, a)\right)$ denote respectively the index of the epoch and that of the in-epoch episode in which $(s, a)$ is visited for the $n$-th time at step $h$; for the sake of conciseness, we shall often use the shorthand $k^n = k^n_h(s,a)$, $(m^n,k^n) = \left(m^n_h(s,a), k^n_h(s, a)\right)$ whenever it is clear from context.

	\item $Q_h^{k}(s, a)$, $Q_h^{\LCB, k}(s, a)$, $\overline{Q}_h^{k}(s, a)$ and $V_h^{k}(s)$ are used to denote $Q_h(s, a)$, $Q_h^{\LCB}(s, a)$, $\overline{Q}_h(s, a)$, and $V_h(s)$ at the {\em beginning} of the $k$-th episode, respectively.

	\item $\overline{V}_h^{k}(s), \overline{V}_h^{\nnext, k}(s), \overline{\mu}^{k}_h(s,a), \overline{\mu}^{\nnext,k}_h(s,a)$ denote the values of $\overline{V}_h(s), \overline{V}^{\nnext}_h(s), \overline{\mu}_h(s,a)$ and $\overline{\mu}^{\nnext}_h(s,a)$ at the {\em beginning} of the $k$-th episode, respectively.

	\item $\widehat{N}_h^{(m,t)}(s,a)$ represents $\widehat{N}_h(s,a)$ at the {\em beginning} of the $t$-th in-epoch episode in the $m$-th epoch.

	\item $\widehat{N}_h^{\epo, m}(s,a)$ denotes  $\widehat{N}_h^{(m, L_m+1)}(s,a)$, representing the number of visits to $(s,a)$ in the entire duration of the $m$-th epoch.
	\item $[\mu^{\re, k}_h, \sigma^{\re, k}_h, \mu^{\adv,k}_h, \sigma^{\adv, k}_h, \overline{\delta}^{k}_h, \overline{\sumb}^{k}_h, \overline{b}_h^{k}]$: the values of  $[\mu^{\re}_h, \sigma^{\re}_h, \mu^{\adv}_h, \sigma^{\adv}_h, \overline{\delta}_h, \overline{\sumb}_h, \overline{b}_h]$ at the {\em beginning} of the $k$-th episode, respectively.

\end{itemize}

In addition, for a fixed vector $V \in \mathbb{R}^{|\mathcal{S}|}$, let us define a variance parameter with respect to $P_{h,s,a}$ as follows
\begin{equation} 
\label{lemma1:equ2}
  \Var_{h, s, a}(V) \defn \mathop{\mathbb{E}}\limits_{s' \sim P_{h,s,a}} \Big [\big(V(s') -  P_{h,s, a}V \big)^2\Big] = P_{h,s,a} (V^{ 2}) - (P_{h,s,a}V)^2.
\end{equation} 
This notation will be useful in the subsequent proof. We remind the reader that there exists a one-to-one mapping between the index of the episode $k$ and the index pair  $(m,t)$ (i.e., the epoch $m$ and in-epoch episode $t$), as specified in \eqref{eqn:k-to-m-t}. 
 
In the following, for any episode $k$, we recall the expressions of $\overline{V}_{h+1}$ and $\overline{\mu}_h$ (which is the running mean of $\overline{V}_{h+1}$).

\begin{itemize}
  \item Recalling the update rule of $\overline{V}_h$ and $\overline{V}^{\nnext}_h$ in line~\ref{eq:update-mu-reference-v-k} and line~\ref{eq:update-mu-reference-v-next-k} of Algorithm~\ref{alg:lcb-advantage-per-epoch-k}, we observe that
    both the reference values for the current epoch $\overline{V}_h$ and for the next epoch $\overline{V}^{\nnext}_h$ remain unchanged within each epoch. Additionally, for any epoch $m$, $\overline{V}_h$ takes the value of $\overline{V}^{\nnext}_h$ in the previous $(m-1)$-th epoch; namely, for any episode $k$ happening in the $m$-th epoch, we have
    \begin{align}\label{eq:v-to-v-next}
      \overline{V}_h^{k} &= \overline{V}^{\nnext, k'}_h
    \end{align}
    for all episode $k'$  within the $(m-1)$-th epoch.

    \item $\overline{\mu}^k_h$ serves as the estimate of $ P_{h,s,a}\overline{V}^k_{h+1}$ constructed by the samples in the previous $(m-1)$-th epoch (collected by updating $\overline{\mu}^{\nnext}_h$). Recall the update rule of $\overline{\mu}_h$ in line~\ref{eq:update-mu-reference-v-k} and line~\ref{line:ref-mean-update-k} of Algorithm~\ref{alg:lcb-advantage-per-epoch-k}: for any $(s,a,h) \in \cS\times \cA\times [H]$, we can write $\overline{\mu}^k_h$ as 
\begin{align}
  \overline{\mu}_h^k(s,a) &= \overline{\mu}_h^{(m,1)} (s,a) = \overline{\mu}_h^{\nnext, (m,1)}(s,a) = \overline{\mu}_h^{\nnext, (m-1,L_{m-1})}(s,a) \nonumber \\
  &= \frac{\sum_{i= N_h^{(m-1,1)} + 1 }^{ N_h^{(m,1)}} \overline{V}^{\nnext, k^i}_{h+1}(s^{k^i}_{h+1})}{\widehat{N}_h^{\epo, m-1}(s,a) \vee 1} = \frac{\sum_{i= N_h^{(m-1,1)} + 1 }^{ N_h^{(m,1)}} \overline{V}^k_{h+1}(s^{k^i}_{h+1})}{\widehat{N}_h^{\epo, m-1}(s,a) \vee 1} ,\label{equ:definition-ref-refmean}
\end{align}
where the last equality follows from \eqref{eq:v-to-v-next} using the fact that the indices of episodes in which $(s,a)$ is visited within the $(m-1)$-th epoch are $\{i: i = N_h^{(m-1,1)}+1, N_h^{(m-1,1)}+2,\cdots,N_h^{(m,1)}\}$.

\end{itemize}

Finally, according to the update rules of $\mu^{\adv, k^{n+1}}_h(s^k_h, a^k_h)$ and $\sigma^{\adv, k^{n+1}}_h(s^k_h, a^k_h) $  in lines~\ref{line:advmu_h}-\ref{line:advsigma_h} of Algorithm~\ref{algo:subroutine}, we have
\begin{align*}
\mu^{\adv, k^{n+1}}_h(s^k_h, a^k_h) &= \mu^{\adv, k^n+1}_h(s^k_h, a^k_h) = (1-\eta_n)\mu^{\adv, k^n}_h(s^k_h, a^k_h) + \eta_n \big(V^{k^n}_{h+1}(s^{k^n}_{h+1}) - \overline{V}^{k^n}_{h+1}(s^{k^n}_{h+1}) \big) ,\\ 
\sigma^{\adv, k^{n+1}}_h(s^k_h, a^k_h) &= \sigma^{\adv, k^n+1}_h(s^k_h, a^k_h) = (1-\eta_n)\sigma^{\adv, k^n}_h(s^k_h, a^k_h) + \eta_n \big( V^{k^n}_{h+1}(s^{k^n}_{h+1}) - \overline{V}^{k^n}_{h+1}(s^{k^n}_{h+1}) \big)^2.
\end{align*} 
Applying this relation recursively and invoking the definitions of  $\eta_n^{N_h^k}$ in \eqref{equ:learning rate notation} give
\begin{align}\label{eq:recursion_mu_sigma_adv}
  \mu_h^{\adv, k^{N_h^k}+1}(s,a) &= \sum_{n=1}^{N_h^k} \eta_n^{N_h^k} P_h^{k^n} \big(V_{h+1}^{k^n} - \overline{V}_{h+1}^{k^n} \big), \;\;
  {\sigma}_h^{\adv, k^{N_h^k}+1}(s,a)  = \sum_{n = 1}^{{N_h^k}} \eta_n^{N_h^k} P_h^{k^n} \big(V_{h+1}^{k^n} - \overline{V}_{h+1}^{k^n} \big)^2.
\end{align}
Similarly, according to the update rules of ${\mu}^{\re, k^{n+1}}_h(s,a)$ and ${\sigma}^{\re, k^{n+1}}_h(s, a) $  in lines~\ref{line:refmu_h}-\ref{line:refsigma_h} of Algorithm~\ref{algo:subroutine}, we obtain
\begin{align*}
{\mu}^{\re, k^{n+1}}_h(s,a) &= {\mu}^{\re, k^n+1}_h(s,a) = \left(1-\frac{1}{n}\right) {\mu}^{\re, k^n}_h(s,a) + \frac{1}{n} \overline{V}^{\nnext, k^n}_{h+1}(s^{k^n}_{h+1}), \\
{\sigma}^{\re, k^{n+1}}_h(s,a) &= {\sigma}^{\re, k^n+1}_h(s,a) = \left(1-\frac{1}{n}\right){\sigma}^{\re, k^n}_h(s,a) + \frac{1}{n}\left(\overline{V}^{\nnext, k^n}_{h+1}(s^{k^n}_{h+1})\right)^2.
\end{align*}
Simple recursion leads to
\begin{align}\label{eq:recursion_mu_sigma_ref}
  {\mu}_h^{\re, k^{N_h^k}+1}(s,a) &=\frac{1}{N_h^k}  \sum_{n=1}^{N_h^k} P_h^{k^n} \overline{V}_{h+1}^{\nnext,k^n},\quad 
  {\sigma}_h^{\re, k^{N_h^k}+1}(s,a) = \frac{1}{N_h^k} \sum_{n = 1}^{{N_h^k}} P_h^{k^n}\big(\overline{V}_{h+1}^{\nnext, k^n}\big)^2.
\end{align}

\subsection{Proof of Lemma~\ref{lem:lcb-adv-lower}}\label{proof:lemma-mono-lcb-adv}

Akin to the proof of Lemma~\ref{lem:Vk-lower}, the second inequality  of \eqref{eq:lcb-adv-lower} holds trivially since
\begin{equation*}
  V_h^{\pi}(s) \leq V_h^\star(s) 
  \end{equation*}
 holds for any policy $\pi$. Thus, it suffices  to focus on justifying the first inequality of \eqref{eq:lcb-adv-lower}, namely, 
  \begin{equation}\label{eq:lcb-adv-lower-induction}
    V_{h}^{k}(s) \leq V_{h}^{\pi^{k}}(s) \qquad \forall (k,h,s) \in [K] \times [H] \times\cS,
  \end{equation}
which we shall prove by induction.

\paragraph{Step 1: introducing the induction hypothesis.}
 
For notational simplicity, let us define 
\begin{align}
  k_o(h,k,s) \coloneqq \max\left\{l: l< k \text{ and } V_h^{l}(s) = \max_a \max\left\{ Q_h^{\LCB, l}(s,a), \overline{Q}_h^{l}(s,a)\right\}  \right\} \label{eq:lcb-adv-k0-def}
\end{align}
for any $(h,k,s)\in[H]\times[K]\times\cS$. Here, $k_o(h,k,s)$ denotes the index of the latest episode --- right at the end  of the $(k-1)$-th episode 
--- in which $V_h(s)$ has been updated, which shall be abbreviated as $k_o(h)$ whenever it is clear from context.

In what follows, we shall first justify the advertised inequality for the base case where $h=H+1$ for all episodes $k\in [K]$, followed by an induction argument. Regarding the induction part, let us consider any $k\in[K]$ and any $h\in [H]$, and suppose that 
\begin{subequations}\label{eq:lcb-adv-lower-indunction-assumption-12}
\begin{align}
V_{h'}^{k'}(s) &\leq V_{h'}^{\pi^{k'}}(s) \qquad \text{for all } (k',h',s) \in [k-1]\times [H+1]\times \cS, \label{eq:lcb-adv-lower-indunction-assumption2}\\
	V_{h'}^k(s) &\leq V_{h'}^{\pi^k}(s) \qquad \text{for all } h'\geq h+1\text{ and }s\in\cS. \label{eq:lcb-adv-lower-indunction-assumption}
\end{align}
\end{subequations}
We intend to justify 
  \begin{equation}\label{eq:lcb-adv-lower-induction2}
    V_{h}^{k}(s) \leq V_{h}^{\pi^{k}}(s) \qquad \forall s \in \cS,
\end{equation}
assuming that the induction hypotheses \eqref{eq:lcb-adv-lower-indunction-assumption-12} hold.

\paragraph{Step 2: controlling the confident bound $\sum_{n = 1}^{N_h^k} \eta^{N_h^k}_n \overline{b}^{k^n+1}_h$.}
Before proceeding, we first introduce an auxiliary result on bounding  
$\sum_{n = 1}^{N_h^k} \eta^{N_h^k}_n \overline{b}^{k^n+1}_h$, which plays a crucial role.   
For any $(s,a)$, it is easily seen that
\begin{align}
 N_h^k(s,a) = 0 \qquad \Longrightarrow \qquad \sum_{n=1}^{N_h^k(s,a)} \eta_n^{N_h^k(s,a)} \overline{b}_h^{k^n(s,a)+1}  = 0. \label{eq:N-zero-case}
\end{align}
When $N_h^k(s,a) >0$, expanding the definitions of $\overline{b}^{k^n+1}_h$ (cf.~line~\ref{line:bonus_2} of Algorithm~\ref{algo:subroutine}) and $\overline{\delta}_h^{k+1}$ (cf.~line~\ref{eq:line-delta} of Algorithm~\ref{algo:subroutine}) leads to 
\begin{align}
& \sum_{n = 1}^{N_h^k} \eta^{N_h^k}_n \overline{b}^{k^n+1}_h \nonumber \\
&=\sum_{n = 1}^{N_h^k} \eta_n \prod_{i = n+1}^{N_h^k}(1-\eta_i) \cdot \left( \Big(1-\frac{1}{\eta_n}\Big) \overline{\sumb}^{k^n}_h(s,a) + \frac{1}{\eta_n} \overline{\sumb}^{k^n+1}_h(s,a) \right) + \cb \sum_{n=1}^{N_h^k}\frac{ \eta_n^{N_h^k} }{n^{3/4}}  H^{7/4}\iota + \cb \sum_{n=1}^{N_h^k}\frac{ \eta_n^{N_h^k} }{n}  H^2\iota\nonumber \\
&= \sum_{n = 1}^{N_h^k} \left(\prod_{i = n+1}^{N_h^k}(1-\eta_i) \overline{\sumb}^{k^n+1}_h(s,a) - \prod_{i = n}^{N_h^k}(1-\eta_i) \overline{\sumb}^{k^n}_h(s,a)\right) + \cb \sum_{n=1}^{N_h^k}\frac{ \eta_n^{N_h^k} }{n^{3/4}}  H^{7/4}\iota + \cb \sum_{n=1}^{N_h^k}\frac{ \eta_n^{N_h^k} }{n}  H^2\iota\nonumber \\
&\overset{\mathrm{(i)}}{=} \sum_{n=1}^{N_h^k}  \prod_{i = n+1}^{N_h^k}(1-\eta_i) \overline{\sumb}^{k^n+1}_h(s,a)  - \sum_{n=2}^{N_h^k} \prod_{i = n}^{N_h^k}(1-\eta_i) \overline{\sumb}^{k^n}_h(s,a) + \cb \sum_{n=1}^{N_h^k}\frac{ \eta_n^{N_h^k} }{n^{3/4}}  H^{7/4}\iota + \cb \sum_{n=1}^{N_h^k}\frac{ \eta_n^{N_h^k} }{n}  H^2\iota\nonumber \\
&\overset{\mathrm{(ii)}}{=}\sum_{n=1}^{N_h^k}  \prod_{i = n+1}^{N_h^k}(1-\eta_i) \overline{\sumb}^{k^n+1}_h(s,a)  - \sum_{n=1}^{N_h^k-1} \prod_{i = n+1}^{N_h^k}(1-\eta_i) \overline{\sumb}^{k^n+1}_h(s,a) + \cb \sum_{n=1}^{N_h^k}\frac{ \eta_n^{N_h^k} }{n^{3/4}}  H^{7/4}\iota + \cb \sum_{n=1}^{N_h^k}\frac{ \eta_n^{N_h^k} }{n}  H^2\iota\nonumber \\
&=\overline{\sumb}^{k^{N_h^k}+1}_h(s,a) + \cb \sum_{n=1}^{N_h^k}\frac{ \eta_n^{N_h^k} }{n^{3/4}}  H^{7/4}\iota + \cb \sum_{n=1}^{N_h^k}\frac{ \eta_n^{N_h^k} }{n}  H^2\iota, \label{eq:lcb-adv-b-boundby-B}
\end{align}
where we abuse the notation to let $\prod_{i=j+1}^j (1-\eta_i)=1$. 
Here, (i) holds since $\overline{\sumb}^{k^1}(s,a) = 0$, (ii) follows from the fact that  $\overline{\sumb}^{k^n+1}(s,a) =  \overline{\sumb}^{k^{n+1}}(s,a)$, since  $(s,a)$ has not been visited at step $h$ during the episodes between the indices $k^n +1$ and $k^{n+1} - 1$.
Combining the above result in \eqref{eq:lcb-adv-b-boundby-B} with the properties $\frac{1}{(N_h^k)^{3/4}} \le \sum_{n = 1}^{N_h^k} \frac{\eta^{N_h^k}_n}{n^{3/4}} \le \frac{2}{(N_h^k)^{3/4}}$ and $\frac{1}{N_h^k} \le \sum_{n = 1}^{N_h^k} \frac{\eta^{N_h^k}_n}{n} \le \frac{2}{N_h^k}$ (see Lemma~\ref{lemma:property of learning rate}), we arrive at
\begin{equation}  \label{lemma1:equ10}
 \overline{\sumb}^{k^{N_h^k}+1}_h(s,a) + \cb\frac{H^{7/4}\iota}{(N_h^k)^{3/4}} + \cb \frac{H^2\iota}{N_h^k} \le \sum_{n = 1}^{N_h^k} \eta^{N_h^k}_n \overline{b}^{k^n+1}_h \le \overline{\sumb}^{k^{N_h^k}+1}_h(s,a) + 2\cb\frac{H^{7/4}\iota}{(N_h^k)^{3/4}} + 2\cb \frac{H^2\iota}{N_h^k} 
\end{equation} 
as long as $N_h^k(s,a) > 0$.

\paragraph{Step 3: base case.}
Let us look at the base case with $h=H+1$ for any $k\in[K]$. Recalling the facts that $V_{H+1}^{\pi}=V_{H+1}^{k}=0$ for any $\pi$ and any $k\in[K]$, we reach
\begin{align}
  V_{H+1}^k(s) &\leq V_{H+1}^{\pi^k}(s) \qquad \text{for all } (k,s)\in [K] \times \cS.
\end{align}

\paragraph{Step 4: induction arguments.}
We now turn to the induction arguments. 
Suppose that \eqref{eq:lcb-adv-lower-indunction-assumption-12} holds for a pair $(k,h)\in [K]\times [H]$.
Everything comes down to justifying \eqref{eq:lcb-adv-lower-induction2} for time step $h$ in the episode $k$.

First, we recall the update rule of $V_h(s)$ in lines~\ref{eq:line-q-update-k}-\ref{eq:line-v-update-k} of Algorithm~\ref{alg:lcb-advantage-per-epoch-k}:
\begin{align*}
 &V^k_h(s) = \max_a Q^k_h(s, a) = Q^k_h\big(s, \pi_h^k(s) \big)
   = \max \left\{Q^{\LCB,k}_h
\left(s, \pi_h^k(s)\right), \overline{Q}^k_h\left(s, \pi_H^k(s)\right), Q^{k-1}_h\left(s, \pi_h^k(s)\right) \right\}.
\end{align*}
Then we shall verify \eqref{eq:lcb-adv-lower-induction2} in three different cases.

\begin{itemize}
\item When $V^k_h(s) = Q^{\LCB,k}_h \left(s, \pi_h^k(s)\right)$, the term of interest can be controlled by
\begin{align*}
  V_{h}^{\pi^k}(s) - V_{h}^k(s) \overset{\mathrm{(i)}}{=} Q_{h}^{\pi^k}\left(s,\pi_h^k(s)\right) - Q^{\LCB,k}_h\left(s, \pi_h^k(s)\right) \geq 0,
\end{align*}
where (i) holds since $\pi^k$ is set to be the greedy policy such that $V_h^{\pi^k}(s)  = Q_h^{\pi^k}(s, \pi_h^k(s)),$ and the last inequality follows directly from the analysis for \LCBQ (see \eqref{equ:lcb-lower-induction-result}).

\item When $V^k_h(s) = \overline{Q}^k_h \left(s, \pi_h^k(s)\right)$, we obtain 
\begin{align}
  V_{h}^{\pi^k}(s) - V_{h}^k(s) = Q_{h}^{\pi^k}\left(s,\pi_h^k(s)\right) - \overline{Q}^k_h\left(s, \pi_h^k(s)\right). \label{eq:V2Q-lcb-adv}
\end{align}
To prove the term on the right-hand side of \eqref{eq:V2Q-lcb-adv} is non-negative, we proceed by developing a more general lower bound on $Q_{h}^{\pi^k}(s, a) - \overline{Q}_{h}^k(s, a)$ for every $(s,a) \in \cS \times \cA$. 
Towards this, recalling the definition of $N_h^k$ and $k^n$, we can express
\begin{align*}
  \overline{Q}_{h}^k(s,a) = \overline{Q}_{h}^{k^{N_h^k} + 1}(s,a).
\end{align*}
Thus, according to the update rule (cf. line~\ref{line:ref-q-update} in Algorithm~\ref{algo:subroutine}), we arrive at
\begin{align*}
	\overline{Q}_{h}^k(s,a)  & = \overline{Q}_{h}^{k^{N_h^k}+1}(s,a) \\
	& =(1-\eta_{N_h^k}) \overline{Q}_{h}^{k^{N_h^k}}(s,a) +\eta_{N_h^k}\bigg\{ r_{h}(s,a)+V_{h+1}^{k^{N_h^k}}(s_{h+1}^{k^{N_h^k}})- \overline{V}_{h+1}^{k^{N_h^k}}(s_{h+1}^{k^{N_h^k}})+ \overline{\mu}_{h}^{k^{N_h^k}}(s,a) - \overline{b}_h^{k^{N_h^k} + 1}\bigg\}.
\end{align*}
Applying this relation recursively and invoking the definitions of $\eta_0^{N_h^k}$ and $\eta_n^{N_h^k}$ in \eqref{equ:learning rate notation} give
\begin{align}
&\overline{Q}^k_h(s,a)
  = \eta_0^{N_h^k} \overline{Q}^{1}_h(s,a) + \sum_{n = 1}^{N_h^k} \eta_n^{N_h^k} \bigg\{ r_h(s,a) + V^{k^n}_{h+1}(s^{k^n}_{h+1}) - \overline{V}^{ k^n}_{h+1}(s^{k^n}_{h+1}) + \overline{\mu}_{h}^{k^n}(s,a) - \overline{b}^{k^n+1}_h \bigg\}.
  \label{eq:Q-ref-decompose-12345}
\end{align}

Additionally, for any policy $\pi^k$, the basic relation $\eta_0^{N_h^k} + \sum_{n = 1}^{N_h^k} \eta_n^{N_h^k} = 1$ (see \eqref{eq:sum-eta-n-N} and \eqref{equ:learning rate notation}) gives
\begin{equation}\label{eq:decomp_Qhpi_eta}
  Q_{h}^{\pi^k} (s,a) = \eta^{N_h^k}_0 Q_{h}^{\pi^k} (s,a) + \sum_{n = 1}^{N_h^k} \eta^{N_h^k}_n Q_{h}^{\pi^k} (s,a).
\end{equation}
Combing \eqref{eq:Q-ref-decompose-12345} and \eqref{eq:decomp_Qhpi_eta} leads to
\begin{align}
 & Q_{h}^{\pi^k}(s,a) - \overline{Q}_{h}^k(s,a) =\eta_{0}^{N_h^k}\big(Q_{h}^{\pi^k}(s,a)- \overline{Q}_{h}^{1}(s,a)\big)\nonumber\\
 & \qquad+\sum_{n=1}^{N_h^k}\eta_{n}^{N_h^k}\bigg\{Q_{h}^{\pi^k}(s,a) - r_{h}(s,a) - V_{h+1}^{k^n}(s_{h+1}^{k^n}) + \overline{V}_{h+1}^{k^n}(s_{h+1}^{k^n}) - \overline{\mu}_{h}^{k^n}(s,a)  + \overline{b}_h^{k^n+1} \bigg\}.
  \label{equ:lemma4 vr 1}
\end{align}

Plugging in the construction of $\overline{\mu}_h$ in \eqref{equ:definition-ref-refmean} and invoking the Bellman equation 
\begin{equation}\label{equ:bellman-equ}
  Q_{h}^{\pi^k}(s, a)  = r_h(s, a) + P_{h,s, a }V_{h+1}^{\pi^k},
  \end{equation}
we arrive at
\begin{align}
&Q_{h}^{\pi^k}(s,a) - r_{h}(s,a) - V_{h+1}^{k^n}(s_{h+1}^{k^n}) + \overline{V}_{h+1}^{k^n}(s_{h+1}^{k^n}) - \overline{\mu}_{h}^{k^n}(s,a)  + \overline{b}_h^{k^n+1}  \nonumber \\
&\quad = P_{h,s,a} V_{h+1}^{\pi^k} +  \overline{V}^{ k^n}_{h+1}(s^{k^n}_{h+1}) - V^{k^n}_{h+1}(s^{k^n}_{h+1})  - \frac{\sum_{i= N_h^{(m^n-1,1)} + 1 }^{ N_h^{(m^n,1)}} \overline{V}^{k^n}_{h+1}(s^{k^i}_{h+1})}{\widehat{N}_h^{\epo, m^n-1}(s,a) \vee 1} + \overline{b}^{k^n+1}_h \nonumber \\
&\quad = P_{h,s,a} V_{h+1}^{\pi^k} - V^{k^n}_{h+1}(s^{k^n}_{h+1}) + \left(P_h^{k^n} - P_{h,s,a}\right)  \overline{V}^{ k^n}_{h+1}\label{eq:dejavu}+ \left(P_{h,s,a} - \frac{\sum_{i= N_h^{(m^n-1,1)} + 1 }^{ N_h^{(m^n,1)}} P_h^{k^i}}{\widehat{N}_h^{\epo, m^n-1}(s,a) \vee 1} \right)\overline{V}^{ k^n}_{h+1} + \overline{b}^{k^n+1}_h \nonumber\\
  &\quad = P_{h, s,a}\left( V_{h+1}^{\pi^k} - V^{k^n}_{h+1}\right) + \overline{b}^{k^n+1}_h + \xi^{k^n}_h, \nonumber
\end{align}
where 
\begin{equation}
\xi^{k^n}_h \coloneqq  \big(P^{k^n}_{h} - P_{h, s,a} \big)\big(\overline{V}^{ k^n}_{h+1} - V^{k^n}_{h+1}  \big) +  \left(P_{h,s,a} - \frac{\sum_{i= N_h^{(m^n-1,1)} + 1 }^{ N_h^{(m^n,1)}} P_h^{k^i}}{\widehat{N}_h^{\epo, m^n-1}(s,a) \vee 1} \right)\overline{V}^{ k^n}_{h+1}.
  \label{eq:defn-xi-kh-123}
\end{equation}
Inserting the above result into \eqref{equ:lemma4 vr 1} leads to the following decomposition
\begin{align}
Q_{h}^{\pi^k}(s,a) - \overline{Q}_{h}^k(s,a)  & = \eta_{0}^{N_h^k}\big(Q_{h}^{\pi^k}(s,a)- \overline{Q}_{h}^{1}(s,a)\big)+ \sum_{n = 1}^{N_h^k} \eta^{N_h^k}_n \bigg\{ P_{h, s,a} \left( V_{h+1}^{\pi^k} - V^{k^n}_{h+1} \right) + \overline{b}^{k^n+1}_h + \xi^{k^n}_h \bigg\} \label{equ:concise-update-no-bound}\\
  &\geq \sum_{n = 1}^{N_h^k} \eta^{N_h^k}_n (\overline{b}^{k^n+1}_h + \xi^{k^n}_h ),
  \label{equ:concise update}
\end{align}
which holds by virtue of the following facts:
\begin{itemize}
\item [(i)] The initialization $\overline{Q}_{h}^{1}(s,a) = 0$ and the non-negativity of $Q_h^{\pi}(s,a)$ for any policy $\pi$ and $(s,a)\in \cS\times \cA$ lead to  $Q_{h}^{\pi^k}(s,a)- \overline{Q}_{h}^{1}(s,a) = Q_{h}^{\pi^k}(s,a) \geq 0$.
\item [(ii)] For any episode $k^n$ appearing before $k$, making use of the 
 induction hypothesis $V_{h+1}^{\pi^k}(s) \geq V_{h+1}^k(s) $ in \eqref{eq:lcb-adv-lower-indunction-assumption} and the monotonicity of $V_h(s)$ in \eqref{equ:monotone-lcb-adv}, we obtain
 \begin{equation}
  V_{h+1}^{\pi^k}(s) -  V^{k^n}_{h+1}(s) \geq V_{h+1}^k(s) - V^{k^n}_{h+1}(s) \geq 0.
  \end{equation}

\end{itemize}

The following lemma ensures that the right-hand side of \eqref{equ:concise update} is non-negative. 
We postpone the proof of Lemma~\ref{lem:lcb-adv-bonus-upper} to Appendix~\ref{proof:lem:lcb-adv-bonus-upper} to streamline our discussion. 

\begin{lemma} \label{lem:lcb-adv-bonus-upper}
For any $\delta \in (0, 1)$, there exists some sufficiently large constant $\cb >0$, such that with probability at least $1-\delta$, 
\begin{equation}
  \bigg| \sum_{n = 1}^{N_h^k} \eta^{N_h^k}_n \xi_h^{k^n} \bigg| \le 
  \sum_{n = 1}^{N_h^k} \eta^{N_h^k}_n \overline{b}^{k^n+1}_h , \qquad \forall k\in[K].
  \label{eq:claim-sum-eta-sum-b}
\end{equation}
\end{lemma}
Taking this lemma together with the inequalities \eqref{eq:V2Q-lcb-adv} and \eqref{equ:concise update} yields 
\begin{align*}
  V_{h}^{\pi^k}(s) - V_{h}^k(s) = Q_{h}^{\pi^k}(s,a) - \overline{Q}_{h}^k(s,a) \geq  \sum_{n = 1}^{N_h^k} \eta^{N_h^k}_n  \overline{b}^{k^n+1}_h  - \bigg| \sum_{n = 1}^{N_h^k} \eta^{N_h^k}_n  \xi^{k^n}_h \bigg|  \geq  0.
\end{align*}

\item Next, consider the case where $V^k_h(s) = Q^{k-1}_h\left(s, \pi_h^k(s)\right)$. 
	In view of the definition of $k_o(h)$ in \eqref{eq:lcb-adv-k0-def}, one has 
  \begin{align*}
    V^k_h(s) &= Q^{k-1}_h\left(s, \pi_h^k(s)\right) = Q^{k_o(h)}_h\left(s, \pi_h^k(s)\right)= \max \left\{Q^{\LCB,k_o(h)}_h
\left(s, \pi_h^k(s)\right), \overline{Q}^{k_o(h)}_h\left(s, \pi_h^k(s)\right)\right\},
  \end{align*}
  since $Q_h\left(s, \pi_h^k(s)\right)$ has not been updated during the episode $k_o(h)$ and remains unchanged in the episodes $k_o(h)+1,k_o(h)+2,\cdots, k-1$.
With this equality in hand, the term of interest in \eqref{eq:lcb-adv-lower-induction2} can be controlled by
\begin{align*}
&V_{h}^{\pi^k}(s) - V_{h}^k(s)= Q_h^{\pi^k}(s, \pi_h^k(s)) - \max \left\{Q^{\LCB,k_o(h)}_h
\left(s, \pi_h^k(s)\right), \overline{Q}^{k_o(h)}_h\left(s, \pi_h^k(s)\right)\right\} \geq 0,
\end{align*}
where the last inequality follows from the facts 
\begin{align*}
Q_{h}^{\pi^k}(s, \pi_h^k(s)) - Q_{h}^{\LCB, k_o(h)}(s,  \pi_h^k(s))& \overset{\mathrm{(i)}}{\geq} 0,\\
Q_{h}^{\pi^k}(s,  \pi_h^k(s)) - \overline{Q}_{h}^{k_o(h)}(s,  \pi_h^k(s)) &\overset{\mathrm{(ii)}}{\geq} 0.
\end{align*} 
Here, (i) follows from the same analysis framework for showing \eqref{equ:lcb-lower-1} and \eqref{equ:lcb-lower-new-1}; (ii) holds due to the following fact
\begin{align*}
  Q_{h}^{\pi^k}(s,a) - \overline{Q}_{h}^{k_o(h)}(s,a)  \geq \sum_{n = 1}^{N_h^{k_o(h)}} \eta^{N_h^{k_o(h)}}_n (\overline{b}^{k^n+1}_h + \xi^{k^n}_h ) \geq 0,
\end{align*}
which is obtained directly by adapting \eqref{equ:concise update} and then invoking \eqref{eq:claim-sum-eta-sum-b} for $k = k_o(h)$; 
since the analysis follows verbatim, we omit their proofs here. 

\end{itemize}

Combining the above three cases verifies the induction hypothesis in \eqref{eq:lcb-adv-lower-induction2}, provided that \eqref{eq:lcb-adv-lower-indunction-assumption-12} is satisfied.

\paragraph{Step 5: putting everything together.} Combining the base case in Step 3 and induction arguments in Step 4, we can readily verify the induction hypothesis in Step 1, which in turn establishes Lemma~\ref{lem:lcb-adv-lower}.

\subsection{Proof of Lemma~\ref{lem:lcb-adv-each-h}}\label{proof:lem:lcb-adv-decompose}

For every $h\in [H]$, we can decompose 
\begin{align}
    \sum_{k=1}^K \sum_{s\in \cS} d_h^{\pi^\star}(s) \left(V_h^\star(s) - V_h^{k}(s)\right) \nonumber
    &\overset{\mathrm{(i)}}{\leq} \sum_{k=1}^K \sum_{s\in \cS} d_h^{\pi^\star}\big(s, \pi_h^\star(s)\big) \left(Q_h^\star\big(s, \pi_h^\star(s)\big) 
	- \overline{Q}_h^{k}\big(s, \pi_h^\star(s)\big)\right) \nonumber\\
    & = \sum_{k=1}^K \sum_{s, a\in \cS \times \cA} d_h^{\pi^\star}(s, a) \left(Q_h^\star(s, a) - \overline{Q}_h^{k}(s, a)\right),\label{equ:lcb-adv-rewrite}
\end{align}
where (i) follows from the fact $V_h^{k}(s) =\max_a Q_h^k(s,a) \geq \max_a \overline{Q}_h^k(s,a) \geq \overline{Q}_h^{k}(s, \pi_h^\star(s))$ (see lines~\ref{eq:line-q-update-k}-\ref{eq:line-v-update-k} in Algorithm~\ref{alg:lcb-advantage-per-epoch-k}).
Here, the last equality is due to \eqref{eq:d_h_pi_star}.

\paragraph{Step 1: bounding $Q_{h}^{\star}(s,a) - \overline{Q}_{h}^{k}(s,a)$.}
The basic relation $\eta_0^{N_h^k} + \sum_{n = 1}^{N_h^k} \eta_n^{N_h^k} = 1$ (see \eqref{eq:sum-eta-n-N} and \eqref{equ:learning rate notation}) gives
\begin{equation}\label{eq:decomp_Qhstar_eta}
  Q_{h}^{\star} (s,a) = \eta^{N_h^k}_0 Q_{h}^{\star} (s,a) + \sum_{n = 1}^{N_h^k} \eta^{N_h^k}_n Q_{h}^{\star} (s,a),
\end{equation}
which combined with \eqref{eq:Q-ref-decompose-12345} leads to

\begin{align}
 & Q_{h}^{\star}(s,a) - \overline{Q}_{h}^{k}(s,a) =\eta_{0}^{N_h^k}\big(Q_{h}^{\star}(s,a)- \overline{Q}_{h}^{1}(s,a)\big)\nonumber\\
 & \qquad+\sum_{n=1}^{N_h^k}\eta_{n}^{N_h^k}\bigg\{Q_{h}^{\star}(s,a) - r_{h}(s,a) - V_{h+1}^{k^n}(s_{h+1}^{k^n}) +  \overline{V}_{h+1}^{k^n}(s_{h+1}^{k^n}) - \overline{\mu}_{h}^{k^n}(s,a)  +  \overline{b}_{h}^{k^n+1} \bigg\}.
  \label{equ:lemma4 vr 1 adv}
\end{align}
Invoking the Bellman optimality equation
\begin{equation}
    Q_{h}^{\star}(s, a)  = r_h(s, a) + P_{h,s, a }V_{h+1}^{\star}, 
\end{equation}
we can decompose $Q_{h}^{\star}(s,a) - \overline{Q}_{h}^{k}(s,a)$ similar to \eqref{equ:concise-update-no-bound} by inserting \eqref{eq:defn-xi-kh-123} as follows:
\begin{align} 
  &Q_{h}^{\star}(s,a) - \overline{Q}_{h}^{k}(s,a)  = \eta_{0}^{N_h^k}\big(Q_{h}^{\star}(s,a)-\overline{Q}_{h}^{1}(s,a)\big) + \sum_{n = 1}^{N_h^k} \eta^{N_h^k}_n \bigg\{ P_{h, s,a}\left( V_{h+1}^{\star} - V^{k^n}_{h+1}\right) + \overline{b}^{k^n+1}_h + \xi^{k^n}_h \bigg\} \nonumber\\
  &\overset{\mathrm{(i)}}{\leq} \eta_{0}^{N_h^k} H + \sum_{n = 1}^{N_h^k} \eta^{N_h^k}_n \left(\overline{b}^{k^n+1}_h + \xi^{k^n}_h\right) + \sum_{n = 1}^{N_h^k} \eta^{N_h^k}_n  P_{h, s,a}\left( V_{h+1}^{\star} - V^{k^n}_{h+1}\right) \nonumber\\
  &\overset{\mathrm{(ii)}}{\leq}  \eta_{0}^{N_h^k} H + \sum_{n = 1}^{N_h^k} \eta^{N_h^k}_n  P_{h, s,a}\left( V_{h+1}^{\star} - V^{k^n}_{h+1}\right) + 2\sum_{n = 1}^{N_h^k} \eta^{N_h^k}_n \overline{b}^{k^n+1}_h \nonumber \\
  & \overset{\mathrm{(iii )}}{\leq} \eta_{0}^{N_h^k} H + \sum_{n = 1}^{N_h^k} \eta^{N_h^k}_n  P_{h, s,a}\left( V_{h+1}^{\star} - V^{k^n}_{h+1}\right) + 2\left(\overline{\sumb}^{k}_h(s,a) + 2\cb\frac{H^{7/4}\iota}{\left(N_h^k \vee 1\right)^{3/4}} + 2\cb \frac{H^2\iota}{N_h^k \vee 1}  \right),   \label{eq:lcb-adv-Q-upper}
  \end{align}
  where (i) follows from the initialization $\overline{Q}_{h}^{1}(s,a) = 0$ and the trivial upper bound $Q_{h}^{\pi}(s,a)\leq H$ for any policy $\pi$, (ii) holds owing to the fact (see \eqref{eq:claim-sum-eta-sum-b})
  \begin{align}
  \sum_{n = 1}^{N_h^k} \eta^{N_h^k}_n \left(\overline{b}^{k^n+1}_h + \xi^{k^n}_h\right) \leq \sum_{n = 1}^{N_h^k} \eta^{N_h^k}_n\overline{b}^{k^n+1}_h + \bigg| \sum_{n = 1}^{N_h^k} \eta^{N_h^k}_n  \xi^{k^n}_h \bigg| \leq 2\sum_{n = 1}^{N_h^k} \eta^{N_h^k}_n \overline{b}^{k^n+1}_h,
  \end{align}
  and (iii) comes from \eqref{lemma1:equ10} with the fact $\overline{\sumb}^{k^{N_h^k}+1}_h(s,a) = \overline{\sumb}^{k}_h(s,a)$.

\paragraph{Step 2: decomposing the error in \eqref{equ:lcb-adv-rewrite}.}
Plugging \eqref{eq:lcb-adv-Q-upper} into \eqref{equ:lcb-adv-rewrite} and rearranging terms yield
\begin{align}
    &\sum_{k=1}^K \sum_{s\in \cS} d_h^{\pi^\star}(s) \left(V_h^\star(s) - V_h^{k}(s)\right)\\
    & \leq \sum_{k=1}^K \sum_{(s,a)\in \cS \times \cA} d_h^{\pi^\star}(s, a)\left[\eta_0^{N_h^k(s,a)}H + 2\overline{\sumb}_h^{k}(s,a) + \frac{4 \cb H^{7/4}\iota}{\left(N_h^k(s,a) \vee 1\right)^{3/4}} +  \frac{4\cb H^2\iota}{N_h^k(s,a) \vee 1}  \right] \nonumber\\
    & \quad + \sum_{k=1}^K  \sum_{(s,a)\in \cS \times \cA} d_h^{\pi^\star}(s, a) P_{h, s, a}\sum_{n=1}^{N_h^k(s,a)} \eta_n^{N_h^k(s,a)}\left(V_{h+1}^\star - V_{h+1}^{k^n(s,a)}\right) \nonumber \\
    &\leq \underbrace{\sum_{k=1}^K \sum_{(s,a)\in \cS \times \cA} d_h^{\pi^\star}(s, a)\left[\eta_0^{N_h^k(s,a)}H + \frac{4 \cb H^{7/4}\iota}{\left(N_h^k(s,a) \vee 1\right)^{3/4}} +  \frac{4\cb H^2\iota}{N_h^k(s,a) \vee 1}\right]}_{=: J_h^1} + \underbrace{2\sum_{k=1}^K \sum_{(s,a)\in \cS \times \cA} d_h^{\pi^\star}(s, a)\overline{\sumb}_h^{k}(s,a)}_{=: J_h^2} \nonumber \\
  &\quad + \sum_{k=1}^K  \sum_{(s,a)\in \cS \times \cA} d_h^{\pi^\star}(s, a) P_{h, s, a}\sum_{n=1}^{N_h^k(s,a)} \eta_n^{N_h^k(s,a)}\left(V_{h+1}^\star - V_{h+1}^{k^n(s,a)}\right). \label{equ:lcb-adv-recursion}
\end{align}

\paragraph{Step 3: controlling the last term in \eqref{equ:lcb-adv-recursion}.}

If we could verify the following result
\begin{align}
  &\sum_{k=1}^K  \sum_{(s,a)\in \cS \times \cA} d_h^{\pi^\star}(s, a) P_{h, s, a}\sum_{n=1}^{N_h^k(s,a)} \eta_n^{N_h^k(s,a)}\left(V_{h+1}^\star - V_{h+1}^{k^n(s,a)}\right) \nonumber \\
  &\leq \underbrace{\left(1+\frac{1}{H}\right) \sum_{s\in \cS} d_{h+1}^{\pi^\star}(s) \left(V_{h+1}^\star(s) - V_{h+1}^{k}(s)\right) + 48\sqrt{HC^\star K \log\frac{2H}{\delta} } + 28 c_\mathrm{a} H^{3}  C^\star \sqrt{S} \iota^2}_{\eqqcolon J_h^3}, \label{eq:lcb-adv-recursion-key}
\end{align}
then combining this result with inequality~\eqref{equ:lcb-adv-recursion} would immediately establish Lemma~\ref{lem:lcb-adv-each-h}. 
As a result, it suffices to verify the inequality~\eqref{eq:lcb-adv-recursion-key}, which shall be accomplished as follows.

\paragraph{Proof of inequality~\eqref{eq:lcb-adv-recursion-key}.} 
We first make the observation that the left-hand side of inequality~\eqref{eq:lcb-adv-recursion-key} is the same as what Lemma~\ref{lemma:recursion} shows. Therefore, we shall establish this inequality following the same framework as in Appendix~\ref{proof:lemma-lcb-revursion}. 
To begin with, let us recall several definitions in Appendix~\ref{proof:lemma-lcb-revursion}:  
\begin{align}
  A_h & \defn \sum_{k=1}^K  \underbrace{\sum_{(s, a) \in \cS \times \cA} d_h^{\pi^\star}(s, a)  P_{h, s, a}\sum_{n=1}^{N_h^k(s,a)} \eta_n^{N_h^k(s,a)} \left(V_{h+1}^\star - V_{h+1}^{k^n(s,a)}\right)}_{\eqqcolon A_{h,k}},\nonumber \\
  B_{h,k} &\defn \left(1+\frac{1}{H}\right) \sum_{s\in \cS} d_{h+1}^{\pi^\star}(s) \left(V_{h+1}^\star(s) - V_{h+1}^{k}(s)\right),\nonumber \\
  Y_{h,k} &= \frac{d_{h}^{\pi_{\star}}(s_{h}^{k},a_{h}^{k})}{d_{h}^{\mu}(s_{h}^{k},a_{h}^{k})} P_{h,s_{h}^{k},a_{h}^{k}} \sum_{n=1}^{N_{h}^{k}(s_{h}^{k},a_{h}^{k})}\eta_{n}^{N_{h}^{k}(s_{h}^{k},a_{h}^{k})}\left(V_{h+1}^{\star}-V_{h+1}^{k^{n}(s_{h}^{k},a_{h}^{k})}\right),\nonumber \\
  Z_{h,k} &= \left(1+\frac{1}{H}\right) \frac{d_{h}^{\pi_{\star}}(s_{h}^{k},a_{h}^{k})}{d_{h}^{\mu}(s_{h}^{k},a_{h}^{k})} P_{h,s_{h}^{k},a_{h}^{k}} \left(V_{h+1}^{\star}-V_{h+1}^{k}\right), \label{eq:lcb-adv-recall-A-B-X-Y}
\end{align}
and we also remind the reader of the relation in \eqref{eq:recursion-extra-error} as follows
\begin{align}
  A_h \leq \sum_{k=1}^K B_{h,k} + \sum_{k=1}^K \left(Z_{h,k} - B_{h,k}\right) + \sum_{k=1}^K \left(A_{h,k} - Y_{h,k}\right). \label{eq:lcb-adv-decompose-A}
\end{align}

Equipped with these relations, we aim to control $\sum_{k=1}^K \left(Z_{h,k} - B_{h,k}\right)$ and $\sum_{k=1}^K \left(A_{h,k} - Y_{h,k}\right)$ respectively as follows.
\begin{itemize}
	\item We first bound $\sum_{k=1}^K \left(A_{h,k} - Y_{h,k}\right)$, which is similar to \eqref{eq:bound-of-Y} (as controlled by Lemma~\ref{lemma:martingale-union-recursion}). Repeating the argument and tightening the bound from the second line of \eqref{eq:bound-of-Y}, we have for all $(h,s,a) \in [H] \times \cS\times \cA$, with probability at least $1-\delta$,
\begin{align}
   &\left|\sum_{k=1}^K \left(A_{h,k} - Y_{h,k}\right)\right|  
    \leq  \sqrt{\sum_{k=1}^K 8C_{\mathrm{d}}^2 C^\star \sum_{(s,a)\in\cS\times\cA} d_{h}^{\pi_{\star}}(s,a) \left[P_{h,s,a}W_{h+1}^{k}(s,a)\right]^2 \log\frac{2H}{\delta}} + 2C_{\mathrm{d}} C^\star C_\mathrm{w}\log\frac{2H}{\delta} \nonumber \\
   & \leq  \sqrt{8 C^\star \log\frac{2H}{\delta}\sum_{k=1}^K  \sum_{(s,a)\in\cS\times\cA} d_{h}^{\pi_{\star}}(s,a) \left[\sum_{n=1}^{N_h^k(s,a)} \eta_n^{N_h^k(s,a)} P_{h,s,a}\left(V_{h+1}^\star - V_{h+1}^{k^n(s,a)}\right) \right]^2 } + 4H C^\star \log\frac{2H}{\delta} \nonumber \\
   & \overset{\mathrm{(i)}}{\leq} \sqrt{8 C^\star \log\frac{2H}{\delta} \left(36HK +3c_\mathrm{a}^2 H^6SC^\star \iota\right)} + 4H C^\star \log\frac{2H}{\delta} \nonumber \\
   &\leq  32\sqrt{HC^\star K \log\frac{2H}{\delta}} + 12c_\mathrm{a} H^{3} C^\star \sqrt{S} \iota^2. \label{lcb-adv-bound-A-Z}
\end{align}
Here, (i) holds by virtue of the following fact
\begin{align}
   \sum_{k=1}^K  \sum_{(s,a)\in\cS\times\cA} d_{h}^{\pi_{\star}}(s,a) \left[\sum_{n=1}^{N_h^k(s,a)} \eta_n^{N_h^k(s,a)} P_{h,s,a}\left(V_{h+1}^\star - V_{h+1}^{k^n(s,a)}\right) \right]^2  \leq 36HK +3c_\mathrm{a}^2 H^6SC^\star \iota  \label{eq:A-Y-upper-bound},
\end{align}
whose proof is postponed to Appendix~\ref{proof:eq:A-Y-upper-bound}.
    \item Next, we turn to $\sum_{k=1}^K \left(Z_{h,k} - B_{h,k}\right)$, which can be bounded similar to \eqref{eq:bound-of-Z} (as controlled via Lemma~\ref{lemma:martingale-union-recursion}). Repeating the argument and tightening the bound from the second line of \eqref{eq:bound-of-Z} yield
    \begin{align}
   &   \left|\sum_{k=1}^K \left(B_{h,k} - Z_{h,k}\right)\right| \leq  \sqrt{\sum_{k=1}^K 8C_{\mathrm{d}}^2 C^\star \sum_{(s,a)\in\cS\times\cA} d_{h}^{\pi_{\star}}(s,a) \left[P_{h,s,a}W_{h+1}^{k}(s,a)\right]^2 \log\frac{2H}{\delta}} + 2C_{\mathrm{d}} C^\star C_\mathrm{w}\log\frac{2H}{\delta} \nonumber \\
  &    \leq  8\sqrt{C^\star \log\frac{2H}{\delta} \sum_{k=1}^K \sum_{(s,a)\in\cS\times\cA} d_{h}^{\pi_{\star}}(s,a) \left[P_{h,s,a} \left(V_{h+1}^\star - V_{h+1}^{k}\right)\right]^2 } + 8HC^\star\log\frac{2H}{\delta}. \label{eq:lcb-adv-bound-of-Z}
    \end{align}
    To further control \eqref{eq:lcb-adv-bound-of-Z}, we have
    \begin{align}
      \sum_{k=1}^K \sum_{(s,a)\in\cS\times\cA} d_{h}^{\pi_{\star}}(s,a) \left[P_{h,s,a} \left(V_{h+1}^\star - V_{h+1}^{k}\right)\right]^2 &\overset{\mathrm{(i)}}{\leq} \sum_{k=1}^K \sum_{(s,a)\in\cS\times\cA} d_{h}^{\pi_{\star}}(s,a)P_{h,s,a} \left(V_{h+1}^\star - V_{h+1}^{k}\right)^2 \nonumber\\
      &\overset{\mathrm{(ii)}}{\leq} H \sum_{k=1}^K \sum_{(s,a)\in\cS\times\cA} d_{h}^{\pi_{\star}}(s,a)P_{h,s,a} \left(V_{h+1}^\star - V_{h+1}^{k}\right) \nonumber \\
      &\overset{\mathrm{(iii)}}{\leq} 2HK + c_\mathrm{a}^2 H^6SC^\star \iota. \label{eq:lcb-adv-bound-of-Z-key}
    \end{align}
    Here, (i) holds due to the non-negativity of the variance
    \begin{align}
      \Var_{h, s,a}(V_{h+1}^{\star} - \overline{V}_{h+1}^{k}) = P_{h, s,a}(V_{h+1}^{\star} - V_{h+1}^{k})^{2} - \left(P_{h, s,a}(V_{h+1}^{\star} - V_{h+1}^{k}) \right)^2 \geq 0;
    \end{align}
    (ii) follows from the basic property $\left\|V_{h+1}^{\star} - V_{h+1}^{k}\right\|_\infty \leq H$; to see why (iii) holds, we refer the reader to \eqref{equ:146-3}, which will be proven in Appendix~\ref{proof:eq:A-Y-upper-bound} as well. 
  Inserting \eqref{eq:lcb-adv-bound-of-Z-key} back into \eqref{eq:lcb-adv-bound-of-Z} yields
  \begin{align}
    \left|\sum_{k=1}^K \left(B_{h,k} - Z_{h,k}\right)\right| &\leq 8\sqrt{C^\star \log\frac{2H}{\delta} \left(2KH + c_\mathrm{a}^2 H^6SC^\star \iota\right) } + 8HC^\star\log\frac{2H}{\delta}\nonumber\\
    &\leq 16\sqrt{HC^\star K \log\frac{2H}{\delta} } + 16 c_\mathrm{a} H^{3}  C^\star \sqrt{S} \iota. \label{eq:lcb-adv-bound-of-Z-final-result}
  \end{align}
\end{itemize}

Substituting the inequalities~\eqref{lcb-adv-bound-A-Z} and \eqref{eq:lcb-adv-bound-of-Z-final-result} into \eqref{eq:lcb-adv-decompose-A}, and using the definitions in \eqref{eq:lcb-adv-recall-A-B-X-Y}, we arrive at 
\begin{align}
 &  A_h = \sum_{k=1}^K  \sum_{(s, a) \in \cS \times \cA} d_h^{\pi^\star}(s, a)  P_{h, s, a}\sum_{n=1}^{N_h^k(s,a)} \eta_n^{N_h^k(s,a)} \left(V_{h+1}^\star - V_{h+1}^{k^n(s,a)}\right) \nonumber\\
   &\leq \left(1+\frac{1}{H}\right) \sum_{s\in \cS} d_{h+1}^{\pi^\star}(s) \left(V_{h+1}^\star(s) - V_{h+1}^{k}(s)\right) + \sum_{k=1}^K \left(Z_{h,k} - B_{h,k}\right) + \sum_{k=1}^K \left(A_{h,k} - Y_{h,k}\right) \nonumber\\
   &\leq \left(1+\frac{1}{H}\right) \sum_{s\in \cS} d_{h+1}^{\pi^\star}(s) \left(V_{h+1}^\star(s) - V_{h+1}^{k}(s)\right) + 32\sqrt{HC^\star K \log\frac{2H}{\delta}} + 12c_\mathrm{a} H^{3} C^\star \sqrt{S} \iota^2 \nonumber \\
   &\qquad + 16\sqrt{HC^\star K \log\frac{2H}{\delta} } + 16 c_\mathrm{a} H^{3}  C^\star \sqrt{S} \iota \nonumber\\
   &\leq \left(1+\frac{1}{H}\right) \sum_{s\in \cS} d_{h+1}^{\pi^\star}(s) \left(V_{h+1}^\star(s) - V_{h+1}^{k}(s)\right) + 48\sqrt{HC^\star K \log\frac{2H}{\delta} } + 28 c_\mathrm{a} H^{3}  C^\star \sqrt{S} \iota^2,
\end{align}
which directly verifies \eqref{eq:lcb-adv-recursion-key} and completes the proof.

\subsubsection{Proof of inequality~\eqref{eq:A-Y-upper-bound}}\label{proof:eq:A-Y-upper-bound}
\paragraph{Step 1: rewriting the term of interest.}
We first invoke Jensen's inequality to obtain
\begin{align*}
  \Big(\sum_{n = 1}^{{N_h^k}} \eta_n^{N_h^k} P_{h, s,a}\left(V^{\star}_{h+1} - V^{ k^n}_{h+1}\right)\Big)^2  &\leq \sum_{n = 1}^{{N_h^k}} \eta_n^{N_h^k} \Big(P_{h, s,a}\left(V^{\star}_{h+1} - V^{ k^n}_{h+1}\right)\Big)^2 \leq \sum_{n = 1}^{{N_h^k}} \eta_n^{N_h^k} P_{h, s,a}\left(V^{\star}_{h+1} - V^{ k^n}_{h+1}\right)^2,
\end{align*}
where the first inequality follows from $\sum_{n = 1}^{N_h^k} \eta_n^{N_h^k} = 1$ (see \eqref{eq:sum-eta-n-N} and \eqref{equ:learning rate notation}), and the last inequality holds by the non-negativity of the variance $\mathsf{Var}_{h,s,a}[V^{\star}_{h+1} - V^{ k^n}_{h+1}]$. This allows one to derive
\begin{align}
  &\sum_{k=1}^K  \sum_{(s,a)\in\cS\times\cA} d_{h}^{\pi_{\star}}(s,a) \left[\sum_{n=1}^{N_h^k(s,a)} \eta_n^{N_h^k(s,a)} P_{h,s,a}\left(V^{\star}_{h+1} - V^{ k^n}_{h+1}\right)\right]^2 \nonumber\\
  & \leq \sum_{k=1}^K  \sum_{(s,a)\in\cS\times\cA} d_{h}^{\pi_{\star}}(s,a) P_{h, s,a} \sum_{n = 1}^{{N_h^k}} \eta_n^{N_h^k} \left(V^{\star}_{h+1} - V^{ k^n}_{h+1}\right)^2\nonumber \\
  &  \overset{\mathrm{(i)}}{\leq}
 \left(1+\frac{1}{H}\right) \sum_{k=1}^K  \sum_{s\in \cS} d_{h+1}^{\pi^\star}(s) \left(V^{\star}_{h+1}(s) - V^{ k}_{h+1}(s)\right)^2 + 32\sqrt{H^4C^\star K \log\frac{2H}{\delta}} + 32H^2C^\star\log\frac{2H}{\delta}, \label{eq:recursion-bound-2}
\end{align}
where (i) can be verified in a way similar to the proof of Lemma~\ref{lemma:recursion} in Appendix~\ref{proof:lemma-lcb-revursion}. We omit the details for conciseness.

\paragraph{Step 2: controlling the first term in \eqref{eq:recursion-bound-2}.}\label{sec:proof:equ:bound-for-reference-sum}
Let us introduce the following short-hand notation 
\begin{align*}
k_{\mathsf{stop}} \coloneqq c_\mathrm{a}^2 H^5SC^\star \iota,
\end{align*}
and decompose the term  in \eqref{eq:recursion-bound-2} as follows
    \begin{align}
   & \sum_{s\in\cS} d_{h+1}^{\pi^\star}(s) \sum_{k=1}^{K}  \left(V_{h+1}^{\star}(s) - V_{h+1}^{k}(s)\right)^2 \overset{\mathrm{(i)}}{\leq} H \sum_{k=1}^{K} \sum_{s\in\cS} d_{h+1}^{\pi^\star}(s) \left(V_{h+1}^{\star}(s) - V_{h+1}^{k}(s)\right) \nonumber\\
   &= H\sum_{k = 1}^{k_\mathsf{stop}} \sum_{s \in \cS} d_{h+1}^{\pi^\star}(s) \big(V^{\star}_{h+1}(s) - V_{h+1}^{k}(s)\big) + H\sum_{k = k_\mathsf{stop}+1}^{K} \sum_{s \in \cS}  d_{h+1}^{\pi^\star}(s) \big(V^{\star}_{h+1}(s) -V_{h+1}^{k}(s)\big).  \label{equ:adv-L-1} 
    \end{align}
Here, (i) holds since $0\leq V_{h+1}^{\star}(s) - V_{h+1}^{k}(s) \leq H$.
The first term in \eqref{equ:adv-L-1} satisfies
\begin{align}
  H\sum_{k = 1}^{k_\mathsf{stop}} \sum_{s \in \cS} d_{h+1}^{\pi^\star}(s) \big( V^{\star}_{h+1}(s) - V_{h+1}^{k}(s) \big) \leq H \left( c_\mathrm{a}\sqrt{H^5SC^\star \iota k_\mathsf{stop}} +  c_\mathrm{a} H^2 SC^\star \iota \right) \leq c_\mathrm{a}^2 H^6SC^\star \iota, \label{equ:adv-L-1-term1}
\end{align}
where the first inequality holds by applying the results of \LCBQ in \eqref{eq:lcb-final-result-H-layers} with $K = k_\mathsf{stop}$.
The second term in \eqref{equ:adv-L-1} can be controlled as follows:
    \begin{align}
  H\sum_{k = k_\mathsf{stop}+1}^{K} \sum_{s \in \cS}  d_{h+1}^{\pi^\star}(s) \big(V^{\star}_{h+1}(s) -V_{h+1}^{k}(s) \big) 
  &\leq HK \sum_{s\in\cS} d_{h+1}^{\pi^\star}(s) \left(V^{\star}_{h+1}(s) - V^{k_{\mathsf{stop}}}_{h+1}(s)\right) \nonumber \\
  &\leq HK\frac{1}{ k_{\mathsf{stop}}}\sum_{k=1}^{ k_{\mathsf{stop}}} \sum_{s\in\cS} d_{h+1}^{\pi^\star}(s) \left(V^{\star}_{h+1}(s) - V^{k}_{h+1}(s)\right) \nonumber\\
  &\leq HK \left(c_\mathrm{a}  \sqrt{\frac{H^5SC^\star \iota}{k_{\mathsf{stop}}}} + \frac{c_\mathrm{a} H^2 SC^\star \iota }{k_{\mathsf{stop}}}\right) \leq 2HK, \label{equ:adv-L-1-term2}
  \end{align}
where the first and the second inequalities hold by the monotonicity property $V_{h+1}^{k+1} \geq V_{h+1}^k$ introduced in \eqref{equ:monotone-lcb-adv}, and the final inequality follows from applying \eqref{eq:lcb-final-result-H-layers}.

Inserting the results in \eqref{equ:adv-L-1-term1} and \eqref{equ:adv-L-1-term2} into \eqref{equ:adv-L-1} yields
\begin{align}
&\sum_{s\in\cS} d_{h+1}^{\pi^\star}(s) \sum_{k=1}^{K}  \left(V_{h+1}^{\star}(s) - V_{h+1}^{k}(s)\right)^2  
 \leq H \sum_{k=1}^{K} \sum_{s\in\cS} d_{h+1}^{\pi^\star}(s) \left(V_{h+1}^{\star}(s) - V_{h+1}^{k}(s)\right)\leq  2HK + c_\mathrm{a}^2 H^6SC^\star \iota. \label{equ:146-3}
\end{align}

\paragraph{Step 3: combining the above results.}
Inserting the above result \eqref{equ:146-3} back into \eqref{eq:recursion-bound-2}, we reach:
\begin{align}
 &\sum_{k=1}^K  \sum_{(s,a)\in\cS\times\cA} d_{h}^{\pi_{\star}}(s,a) \left[\sum_{n=1}^{N_h^k(s,a)} \eta_n^{N_h^k(s,a)} P_{h,s,a}\left(V^{\star}_{h+1} - V^{ k^n}_{h+1}\right)\right]^2 \nonumber \\
 &\leq 
 \left(1+\frac{1}{H}\right) \sum_{k=1}^K  \sum_{s\in \cS} d_{h+1}^{\pi^\star}(s) \left(V^{\star}_{h+1} - V^{ k}_{h+1}\right)^2 + 32\sqrt{H^4C^\star K \log\frac{2H}{\delta}} + 32H^2C^\star\log\frac{2H}{\delta} \nonumber\\
 &\overset{\mathrm{(i)}}{\leq}  4HK + 2c_\mathrm{a}^2 H^6SC^\star \iota + 32\sqrt{H^4C^\star K \log\frac{2H}{\delta}} + 32H^2C^\star\log\frac{2H}{\delta} \nonumber\\
 &\overset{\mathrm{(ii)}}{\leq}  36HK +3c_\mathrm{a}^2 H^6SC^\star \iota,  \label{eq:lcb-adv-important1}
\end{align}
where (i) holds due to \eqref{equ:146-3} and $1+\frac{1}{H} \leq 2$, and (ii) results from  the Cauchy-Schwarz inequality.
\subsection{Proof of Lemma~\ref{lemma:lcb-adv-bound-terms}}\label{proof:lemma:lcb-adv-bound-terms}

We shall verify the three inequalities in \eqref{eq:lemma6} separately.
\subsubsection{Proof of inequality~\eqref{eq:lemma6-a}}
We start by rewriting the term of interest using the expression of $J_h^1$ in \eqref{eq:Jh123} as
\begin{align}
  &\sum_{h=1}^H \left(1+\frac{1}{H}\right)^{h-1} J_h^1 \nonumber \\
  &= \sum_{h=1}^H \left(1+\frac{1}{H}\right)^{h-1} \sum_{k=1}^K \sum_{(s, a) \in \cS \times \cA} d_h^{\pi^\star}(s, a)\left[\eta_0^{N_h^k(s,a)}H + \frac{4\cb H^{7/4} \iota}{\left(N_h^k(s,a) \vee 1 \right)^{3/4} } + \frac{4\cb H^{2} \iota}{N_h^k(s,a) \vee 1  }\right] \nonumber\\
  & = \underbrace{\sum_{h=1}^H \left(1+\frac{1}{H}\right)^{h-1} \sum_{k=1}^K \sum_{(s, a) \in \cS \times \cA} d_h^{\pi^\star}(s, a) \eta_0^{N_h^k(s,a)}H}_{=:\mathcal{J}_1^1} + \underbrace{\sum_{h=1}^H \left(1+\frac{1}{H}\right)^{h-1} \sum_{k=1}^K \sum_{(s, a) \in \cS \times \cA} d_h^{\pi^\star}(s, a)\frac{4\cb H^{7/4} \iota}{(N_h^k(s,a) \vee 1)^{3/4}}}_{=:\mathcal{J}_1^2} \nonumber\\
  & \quad + \underbrace{\sum_{h=1}^H \left(1+\frac{1}{H}\right)^{h-1} \sum_{k=1}^K \sum_{(s, a) \in \cS \times \cA} d_h^{\pi^\star}(s, a)\frac{4\cb H^2 \iota}{N_h^k(s,a) \vee 1}}_{=:\mathcal{J}_1^3}.
  \end{align}
Invoking \eqref{equ:visit-cover-bound} and \eqref{equ:algebra property} yields
\begin{align}
    \mathcal{J}_1^1 \lesssim H^2 SC^\star \iota. \label{equ:adv-J-11-result}
\end{align}
In terms of $\mathcal{J}_1^2$, one has 
\begin{align}
  \mathcal{J}_1^2 &= \sum_{h=1}^H \left(1+\frac{1}{H}\right)^{h-1} \sum_{k=1}^K \sum_{(s, a) \in \cS \times \cA} d_h^{\pi^\star}(s, a)\frac{4\cb H^{7/4} \iota}{(N_h^k(s,a) \vee 1)^{\frac{3}{4}}} \nonumber\\
  &\overset{\mathrm{(i)}}{\lesssim} H^{7/4}\iota^2 \sum_{h=1}^H \sum_{k=1}^K \sum_{(s, a) \in \cS \times \cA} d_h^{\pi^\star}(s, a)\frac{1}{(k d_h^{\mu}(s,a))^{\frac{3}{4}}} \nonumber\\
  &\overset{\mathrm{(ii)}}{\lesssim} H^{7/4}\iota^2 (C^\star)^{\frac{3}{4}}\sum_{h=1}^H   \sum_{k=1}^K \frac{1}{k^{\frac{3}{4}}} \sum_{(s,a)\in \cS\times \cA} \left(d_h^{\pi^\star}(s, a)\right)^{\frac{1}{4}} \nonumber \\
  & = H^{7/4}\iota^2 (C^\star)^{\frac{3}{4}}\sum_{h=1}^H \sum_{k=1}^K \frac{1}{k^{\frac{3}{4}}} \sum_{(s,a)\in \cS\times \cA} \ind\big(a = \pi^\star_h(s)\big) \left( d_h^{\pi^\star}(s, a)\right)^{\frac{1}{4}}, \nonumber 
\end{align}
where (i) holds due to \eqref{equ:algebra property} and $\frac{1}{N_h^k(s,a) \vee 1} \leq \frac{8\iota}{k d_h^{\mu}(s,a)}$ from Lemma~\ref{lem:binomial},
and (ii) follows from the definition of $C^{\star}$ in Assumption~\ref{assumption}.
A direct application of H\"older's inequality leads to 
\begin{align}
 \mathcal{J}_1^2 &\leq H^{7/4}\iota^2 (C^\star)^{\frac{3}{4}}\sum_{h=1}^H\sum_{k=1}^K \frac{1}{k^{\frac{3}{4}}}  \left(\sum_{(s,a)\in \cS\times \cA} \ind(a = \pi^\star_h(s)) \right)^{3/4} \left(\sum_{(s,a)\in \cS\times \cA}  d_h^{\pi^\star}(s, a)\right)^{1/4} \nonumber \\
  & \overset{\mathrm{(iii)}}{\leq}  H^{7/4}\iota^2 (SC^\star)^{\frac{3}{4}}\sum_{h=1}^H \sum_{k=1}^K \frac{1}{k^{\frac{3}{4}}}  \lesssim H^{2.75}(SC^\star)^{\frac{3}{4}}  K^{\frac{1}{4}}\iota^2, \label{equ:adv-J-12-result}
\end{align}
where (iii) follows since $\pi^\star$ is assumed to be a deterministic policy. 

Similarly, we can derive an upper bound on $\cJ_1^3$ as follows:
\begin{align}
  \cJ_1^3 &= \sum_{h=1}^H \left(1+\frac{1}{H}\right)^{h-1} \sum_{k=1}^K \sum_{(s, a) \in \cS \times \cA} d_h^{\pi^\star}(s, a)\frac{4\cb H^2 \iota}{N_h^k(s,a) \vee 1} \nonumber \\
  &\overset{\mathrm{(i)}}{\lesssim}  H^2\iota^2 \sum_{h=1}^H \sum_{k=1}^K \sum_{(s, a) \in \cS \times \cA} \frac{d_h^{\pi^\star}(s, a)}{kd_h^{\mu}(s,a)} \lesssim H^3 SC^\star \iota^3, \label{equ:adv-J-13-result}
\end{align}
where (i) follows from the result in \eqref{equ:algebra property} and the fact $\frac{1}{N_h^k(s,a) \vee 1} \leq \frac{8\iota}{k d_h^{\mu}(s,a)}$ (cf.~Lemma~\ref{lem:binomial}), and the last relation results from the definition of $C^{\star}$ (cf.~Assumption~\ref{assumption}) and the assumption that $\pi^\star$ is a deterministic policy. 

Putting the preceding results \eqref{equ:adv-J-11-result}, \eqref{equ:adv-J-12-result} and \eqref{equ:adv-J-13-result} together, we conclude that
\begin{align}
  &\sum_{h=1}^H \left(1+\frac{1}{H}\right)^{h-1} J_h^1 \lesssim H^{2.75}(SC^\star)^{\frac{3}{4}}  K^{\frac{1}{4}}\iota^2 +  H^3 SC^\star \iota^3.
\end{align}

\subsubsection{Proof of inequality~\eqref{eq:lemma6-b}}
Making use of the definition of $\overline{B}_h^k(s,a)$ (cf.~\eqref{eq:line-number-19}) in the expression of $J_h^2$ (cf.~\eqref{eq:Jh123}), we obtain
\begin{align}
\sum_{h=1}^H \left(1+\frac{1}{H}\right)^{h-1} J_h^2 & =2\sum_{h=1}^H \left(1+\frac{1}{H}\right)^{h-1} \sum_{k=1}^K \sum_{(s,a)\in \cS \times \cA} d_h^{\pi^\star}(s, a)\overline{\sumb}_h^{k}(s,a) \nonumber\\
& = 2 \sum_{h=1}^H\left(1+\frac{1}{H}\right)^{h-1}\cb\sqrt{H\iota } \sum_{(s,a)\in \cS \times \cA} d_h^{\pi^\star}(s, a)\sum_{k = 1}^K\sqrt{\frac{\sigma^{\adv, k}_h(s, a) - \big(\mu^{\adv, k}_h(s, a)\big)^2}{N_h^k(s, a) \vee 1}} \nonumber \\
   &\qquad + 2\sum_{h = 1}^H \left(1+\frac{1}{H}\right)^{h-1} \cb\sqrt{\iota}  \sum_{(s,a)\in \cS \times \cA} d_h^{\pi^\star}(s, a)\sum_{k = 1}^K \sqrt{\frac{\sigma^{\re, k}_h(s, a) - \big(\mu^{\re, k}_h(s, a)\big)^2}{N_h^k(s, a) \vee 1}} \nonumber \\
&\lesssim \underbrace{\sqrt{H\iota } \sum_{h=1}^H \sum_{(s,a)\in \cS \times \cA} d_h^{\pi^\star}(s, a) \sum_{k = 1}^K\sqrt{\frac{\sigma^{\adv, k}_h(s, a) - \big(\mu^{\adv, k}_h(s, a)\big)^2}{N_h^k(s,a) \vee 1}}}_{=:\mathcal{J}_2^1} \nonumber \\
   &\qquad + \underbrace{\sqrt{\iota} \sum_{h = 1}^H \sum_{(s,a)\in \cS \times \cA} d_h^{\pi^\star}(s, a)\sum_{k = 1}^K \sqrt{\frac{\sigma^{\re, k}_h(s, a) - \big(\mu^{\re, k}_h(s, a)\big)^2}{N_h^k(s,a) \vee 1}}}_{=:\mathcal{J}_2^2},\label{equ:bound of b_hat}
\end{align}
where the last inequality follows from \eqref{equ:algebra property}.
In the following, we shall look at the two terms in \eqref{equ:bound of b_hat} separately. 

\paragraph{Step 1: controlling $\mathcal{J}_2^1$.}
Recalling the expressions of ${\sigma}_h^{\adv, k}(s,a)= {\sigma}_h^{\adv, k^{N_h^k}+1}(s,a) $ in \eqref{eq:recursion_mu_sigma_adv}, we observe that the main part of $\mathcal{J}_2^1$ in \eqref{equ:bound of b_hat} satisfies
\begin{align}
&\sum_{h=1}^H \sum_{(s,a)\in \cS \times \cA} d_h^{\pi^\star}(s, a) \sum_{k = 1}^K\sqrt{\frac{\sigma^{\adv, k}_h(s, a)- \left(\mu^{\adv, k}_h(s, a)\right)^2}{N_h^k(s,a) \vee 1}} \le \sqrt{\iota} \sum_{h=1}^H \sum_{(s,a)\in \cS \times \cA}  \sum_{k = 1}^K\sqrt{d_h^{\pi^\star}(s, a)\frac{d_h^{\pi^\star}(s, a) \cdot\sigma^{\adv, k}_h(s, a)}{k d_h^{\mu}(s,a)}} \nonumber \\
&= \sqrt{\iota} \sum_{h=1}^H \sum_{(s,a)\in \cS \times \cA}  \sum_{k = 1}^K\sqrt{d_h^{\pi^\star}(s, a)\frac{d_h^{\pi^\star}(s, a)  \sum_{n = 1}^{N^{k}_h(s, a)} \eta^{N^{k}_h(s, a)}_n P_h^{k^n}\left(V^{k^n}_{h+1} - \overline{V}^{k^n}_{h+1}\right)^2}{k d_h^{\mu}(s,a)}} \nonumber \\
& \overset{\mathrm{(i)}}{\leq}  \sqrt{ C^\star \iota}\sum_{h=1}^H \sum_{(s,a)\in \cS \times \cA} \sum_{k = 1}^K\sqrt{\frac{1}{k} \ind\big(a = \pi^\star_h(s) \big) \cdot  d_h^{\pi^\star}(s, a) \sum_{n = 1}^{N^{k}_h(s, a)} \eta^{N^{k}_h(s, a)}_n P_h^{k^n}\left(V^{k^n}_{h+1} - \overline{V}^{k^n}_{h+1}\right)^2 } \nonumber \\
& \overset{\mathrm{(ii)}}{\leq} \sqrt{ C^\star \iota} \sqrt{\sum_{k = 1}^K \sum_{h=1}^H \sum_{(s,a)\in \cS \times \cA} d_h^{\pi^\star}(s, a) \sum_{n = 1}^{N^{k}_h(s, a)} \eta^{N^{k}_h(s, a)}_n P_h^{k^n}\left(V^{k^n}_{h+1} - \overline{V}^{k^n}_{h+1}\right)^2 }  \sqrt{\sum_{k = 1}^K \sum_{h=1}^H \sum_{(s,a)\in \cS \times \cA}  \ind\big(a = \pi^\star_h(s)\big) \frac{1}{k}} \nonumber \\
&\lesssim \sqrt{HSC^\star \iota^2 }  \sqrt{\sum_{k = 1}^K \sum_{h=1}^H \sum_{(s,a)\in \cS \times \cA} d_h^{\pi^\star}(s, a) \sum_{n = 1}^{N^{k}_h(s, a)} \eta^{N^{k}_h(s, a)}_n P_h^{k^n}\left(V^{k^n}_{h+1} - \overline{V}^{k^n}_{h+1}\right)^2 }   \label{eq:proof of var of var prelim},
\end{align}
where the first inequality is due to the fact $\frac{1}{N_h^k(s,a) \vee 1} \leq \frac{8\iota}{k d_h^{\mu}(s,a)}$ from Lemma~\ref{lem:binomial}, (i) follows from the definition of $C^{\star}$ in Assumption~\ref{assumption} and \eqref{eq:d_h_pi_star}, and (ii) follows from the Cauchy-Schwarz inequality.
To continue, we claim the following bound holds, which will be proven in Appendix~\ref{sec:proof:equ:adv-J-21-E}: 
\begin{align}
  & \sum_{k = 1}^K \sum_{h=1}^H \sum_{(s,a)\in \cS \times \cA} d_h^{\pi^\star}(s, a) \sum_{n = 1}^{N^{k}_h(s, a)} \eta^{N^{k}_h(s, a)}_n P_h^{k^n}\left(V^{k^n}_{h+1} - \overline{V}^{k^n}_{h+1}\right)^2  \nonumber \\
   & \lesssim   H^2 \max_{h\in [H] } \sum_{k=1}^K \sum_{s\in \cS} d_{h}^{\pi^\star}(s) \left(V^{\star}_{h}(s) - V^{k}_{h}(s)\right) + K + H^5\sqrt{S}C^\star\iota^2. \label{equ:adv-J-21-1}
\end{align}
Combining the above inequality with \eqref{eq:proof of var of var prelim}, we arrive at 
\begin{align}
  \mathcal{J}_2^1  &   \lesssim \sqrt{H^2SC^\star \iota^3 } \sqrt{H^2 \max_{h\in [H] } \sum_{k=1}^K \sum_{s\in \cS} d_{h}^{\pi^\star}(s) \left(V^{\star}_{h}(s) - V^{k}_{h}(s)\right) + K + H^5\sqrt{S}C^\star\iota^2} \nonumber \\
  &\lesssim \sqrt{H^4SC^\star \iota^3  \max_{h\in [H] } \sum_{k=1}^K \sum_{s\in \cS} d_{h}^{\pi^\star}(s) \left(V^{\star}_{h}(s) - V^{k}_{h}(s)\right) } + \sqrt{H^2SC^\star K \iota^3}  + H^{3.5}SC^\star \iota^{2.5}.\label{equ:adv-J-21-result}
\end{align}

\paragraph{Step 2: controlling $\mathcal{J}_2^2$.}
Recalling the expressions of $ {\mu}_h^{\re, k+1}(s,a) = {\mu}_h^{\re, k^{N_h^k}+1}(s,a)$ and $ {\sigma}_h^{\re, k +1}(s,a) = {\sigma}_h^{\re, k^{N_h^k}+1}(s,a) $ in \eqref{eq:recursion_mu_sigma_ref} to $\mathcal{J}_2^2$ in \eqref{equ:bound of b_hat}, we can deduce that
\begin{align}
& \mathcal{J}_2^2 =\sqrt{\iota}  \sum_{h = 1}^H \sum_{(s,a)\in \cS \times \cA} d_h^{\pi^\star}(s, a)\sum_{k = 1}^K \sqrt{\frac{\sigma^{\re, k}_h(s, a) - \left(\mu^{\re, k}_h(s, a)\right)^2}{N_h^k(s,a) \vee 1}} \nonumber \\
  & \leq \sqrt{\iota}    \sum_{h = 1}^H \sum_{(s,a)\in \cS \times \cA} d_h^{\pi^\star}(s, a)\sum_{k = 1}^K\sqrt{\frac{1}{N_h^k(s,a) \vee 1}} \underbrace{ \sqrt{ \frac{\sum_{n = 1}^{N^{k}_h(s, a)}\left(\overline{V}^{\nnext, k^n}_{h+1}(s^{k^n}_{h+1})\right)^2}{N^{k}_h(s, a) \vee 1} - \Big(\frac{\sum_{n = 1}^{N^{k}_h(s, a)} \overline{V}^{\nnext, k^n}_{h+1}(s^{k^n}_{h+1})}{N^{k}_h(s, a) \vee 1}\Big)^2 }  }_{=: F_{h,k}} \label{equ:adv-var1}.
  \end{align}
We further decompose and bound $F_{h,k}$ as follows:  
  \begin{align}
    & F_{h,k} 
     \overset{\mathrm{(i)}}{\leq}   \sqrt{\frac{\sum_{n = 1}^{N^{k}_h(s, a)}\left(V^{\star}_{h+1}(s^{k^n}_{h+1})\right)^2}{N^{k}_h(s, a) \vee 1} - \Big(\frac{\sum_{n = 1}^{N^{k}_h(s, a)} \overline{V}^{\nnext, k^n}_{h+1}(s^{k^n}_{h+1})}{N^{k}_h(s, a) \vee 1}\Big)^2} \nonumber\\
  &   = \sqrt{\frac{\sum_{n = 1}^{N^{k}_h(s,a)}\left(V^{\star}_{h+1}(s^{k^n}_{h+1})\right)^2}{N^{k}_h(s,a) \vee 1} - \Big(\frac{\sum_{n = 1}^{N^{k}_h(s,a)}V^{\star}_{h+1}(s^{k^n}_{h+1})}{N^{k}_h(s,a) \vee 1}\Big)^2 + \Big(\frac{\sum_{n = 1}^{N^{k}_h(s, a)}V^{\star}_{h+1}(s^{k^n}_{h+1})}{N^{k}_h(s, a) \vee 1}\Big)^2 - \Big(\frac{\sum_{n = 1}^{N^{k}_h(s, a)} \overline{V}^{\nnext,k^n}_{h+1}(s^{k^n}_{h+1})}{N^{k}_h(s, a) \vee 1}\Big)^2} \nonumber\\
&\overset{\mathrm{(ii)}}{\leq}  \underbrace{ \sqrt{\frac{\sum_{n = 1}^{N^{k}_h(s,a) }\left(V^{\star}_{h+1}(s^{k^n}_{h+1})\right)^2}{N^{k}_h(s,a) \vee 1} - \Big(\frac{\sum_{n = 1}^{N^{k}_h(s,a)}V^{\star}_{h+1}(s^{k^n}_{h+1})}{N^{k}_h(s,a) \vee 1}\Big)^2}}_{G_{h,k}} + \underbrace{ \sqrt{\frac{\sum_{n = 1}^{N^{k}_h(s,a)}2H \left(V^{\star}_{h+1}(s^{k^n}_{h+1}) - \overline{V}^{\nnext,k^n}_{h+1}(s^{k^n}_{h+1})\right)}{N^{k}_h(s,a) \vee 1} }}_{=: L_{h,k}} ,\label{equ:adv-J22-terms}
    \end{align}
where (i) follows from the fact that for some $k'\in[K]$, $\overline{V}^{\nnext, k^n}_{h+1}  = V^{k'}_{h+1} \leq V^{\star}_{h+1}$ (see the update rule of $\overline{V}^{\nnext}$ in line~\ref{eq:update-mu-reference-v-next-k} and the fact in \eqref{eq:lcb-adv-lower}), and (ii) holds due to the fact that 
$$\Big(\frac{\sum_{n = 1}^{N^{k}_h(s, a)}V^{\star}_{h+1}(s^{k^n}_{h+1})}{N^{k}_h(s, a) \vee 1}\Big)^2 - \Big(\frac{\sum_{n = 1}^{N^{k}_h(s, a)} \overline{V}^{\nnext, k^n}_{h+1}(s^{k^n}_{h+1})}{N^{k}_h(s, a) \vee 1}\Big)^2\leq 2H \frac{\sum_{n = 1}^{N^{k}_h(s, a)}\left(V^{\star}_{h+1}(s^{k^n}_{h+1}) - \overline{V}^{\nnext, k^n}_{h+1}(s^{ k^n}_{h+1})\right)}{N^{k}_h(s, a) \vee 1} .$$
Inserting \eqref{equ:adv-J22-terms} back into \eqref{equ:adv-var1}, we arrive at
\begin{align}
  \notag \mathcal{J}_2^2 & \leq \sqrt{\iota}    \sum_{h = 1}^H \sum_{(s,a)\in \cS \times \cA} d_h^{\pi^\star}(s, a)\sum_{k = 1}^K\sqrt{\frac{1}{N_h^k(s,a) \vee 1}}  \left(G_{h,k} + L_{h,k}\right)\\
    &\overset{\mathrm{(i)}}{\lesssim} \sqrt{\iota} \left(\sqrt{H^3S C^\star K \iota^4} + H^{4} S C^\star \iota^3 + \sqrt{H^3 S C^\star K\iota^2} + H^{2.5}SC^\star\iota^3 \right)\lesssim \sqrt{H^3S C^\star K \iota^5} + H^{4} S C^\star \iota^4, \label{equ:adv-J-22-result}
\end{align}
where (i) follows from the following facts 
\begin{align}
  &\sum_{h = 1}^H \sum_{(s,a)\in \cS \times \cA} d_h^{\pi^\star}(s, a)\sum_{k = 1}^K\sqrt{\frac{1}{N^{k}_h(s, a) \vee 1}} L_{h,k}
 \lesssim \sqrt{H^3S C^\star K \iota^4} + H^{4} S C^\star \iota^3,\label{adv-J-22-L}\\
&\sum_{h = 1}^H \sum_{(s,a)\in \cS \times \cA} d_h^{\pi^\star}(s, a)\sum_{k = 1}^K\sqrt{\frac{1}{N^{k}_h(s, a) \vee 1}} G_{h,k}\lesssim  \sqrt{H^3 S C^\star K\iota^2} + H^{2.5}SC^\star\iota^3. \label{adv-J-22-G}
\end{align}
We postpone the proofs of \eqref{adv-J-22-L} and \eqref{adv-J-22-G} to Appendix~\ref{proof:adv-J-22-L} and Appendix~\ref{proof:adv-J-22-G}, respectively.

\paragraph{Putting the bounds together.}
Substitute \eqref{equ:adv-J-21-result} and  \eqref{equ:adv-J-22-result} back into \eqref{equ:bound of b_hat} to yield
\begin{align*}
 \sum_{h=1}^H \left(1+\frac{1}{H}\right)^{h-1} J_h^2 & \lesssim \sqrt{H^4SC^\star \iota^3  \max_{h\in [H] } \sum_{k=1}^K \sum_{s\in \cS} d_{h}^{\pi^\star}(s) \left(V^{\star}_{h}(s) - V^{k}_{h}(s)\right) }  + \sqrt{H^2SC^\star K \iota^3}  + H^{3.5}SC^\star \iota^{2.5} \\
 &\qquad + \sqrt{H^3S C^\star K \iota^5} + H^{4} S C^\star \iota^4\\
 &\lesssim \sqrt{H^4SC^\star \iota^3  \max_{h\in [H] } \sum_{k=1}^K \sum_{s\in \cS} d_{h}^{\pi^\star}(s) \left(V^{\star}_{h}(s) - V^{k}_{h}(s)\right) } + \sqrt{H^3S C^\star K \iota^5} + H^{4} S C^\star \iota^4.
\end{align*}

\subsubsection{Proof of inequality~\eqref{eq:lemma6-c}}

Invoking inequality~\eqref{equ:algebra property} directly leads to 
\begin{align*}
  \sum_{h=1}^H \left(1+\frac{1}{H}\right)^{h-1} \left(48\sqrt{HC^\star K \log\frac{2H}{\delta} } + 28 c_\mathrm{a} H^{3}  C^\star \sqrt{S} \iota^2 \right) &\lesssim \sqrt{H^3C^\star K \log\frac{2H}{\delta}} + H^{4}C^\star \sqrt{S} \iota^2 
\end{align*}
as claimed.

\subsubsection{Proof of inequality \eqref{equ:adv-J-21-1}}
\label{sec:proof:equ:adv-J-21-E}
We shall control the term in \eqref{equ:adv-J-21-1} in a way similar to the proof of Lemma~\ref{lemma:recursion} in Appendix~\ref{proof:lemma-lcb-revursion}.

\paragraph{Step 1: decomposing the terms of interest.}
Akin to Appendix~\ref{proof:lemma-lcb-revursion}, let us introduce the terms of interest and definitions as follows:
\begin{align}
  A_h & \defn \sum_{k=1}^K  \underbrace{\sum_{(s, a) \in \cS \times \cA} d_h^{\pi^\star}(s, a)  \sum_{n=1}^{N_h^k(s,a)} \eta_n^{N_h^k(s,a)} P_h^{k^n}\left(V^{k^n}_{h+1} - \overline{V}^{k^n}_{h+1}\right)^2}_{\eqqcolon A_{h,k}},\nonumber \\
  B_{h,k} &\defn \left(1+\frac{1}{H}\right) \sum_{s\in \cS} d_{h+1}^{\pi^\star}(s) \left(V^{k}_{h+1}(s) - \overline{V}^{k}_{h+1}(s)\right)^2,\nonumber \\
  Y_{h,k} &= \frac{d_{h}^{\pi_{\star}}(s_{h}^{k},a_{h}^{k})}{d_{h}^{\mu}(s_{h}^{k},a_{h}^{k})}  \sum_{n=1}^{N_{h}^{k}(s_{h}^{k},a_{h}^{k})}\eta_{n}^{N_{h}^{k}(s_{h}^{k},a_{h}^{k})}P_h^{k^n}\left(V^{k^n}_{h+1} - \overline{V}^{k^n}_{h+1}\right)^2,\nonumber \\
  Z_{h,k} &= \left(1+\frac{1}{H}\right) \frac{d_{h}^{\pi_{\star}}(s_{h}^{k},a_{h}^{k})}{d_{h}^{\mu}(s_{h}^{k},a_{h}^{k})} P_h^{k} \left(V^{k}_{h+1} - \overline{V}^{k}_{h+1}\right)^2. \label{eq:lcb-adv-recall-A-B-X-Y-2}
\end{align}
With these definitions in place, we directly adapt the argument in \eqref{eq:recursion-extra-error} to arrive at
\begin{align}
  A_h \leq \sum_{k=1}^K B_{h,k} + \sum_{k=1}^K \left(Z_{h,k} - B_{h,k}\right) + \sum_{k=1}^K \left(A_{h,k} - Y_{h,k}\right). \label{eq:lcb-adv-decompose-A-2}
\end{align}
As a consequence, it remains to control $\sum_{k=1}^K \left(Z_{h,k} - B_{h,k}\right)$ and $\sum_{k=1}^K \left(A_{h,k} - Y_{h,k}\right)$ separately.

\paragraph{Step 2: controlling $\sum_{k=1}^K \left(A_{h,k} - Y_{h,k}\right)$.}
To control $\sum_{k=1}^K \left(A_{h,k} - Y_{h,k}\right)$, we resort to Lemma~\ref{lemma:martingale-union-recursion} by setting
\begin{align}
  W_{h+1}^{k}(s,a) \coloneqq \sum_{n=1}^{N_h^k(s,a)} \eta_n^{N_h^k(s,a)} \left(V^{k^n}_{h+1} - \overline{V}^{k^n}_{h+1}\right)^2, \qquad C_{\mathrm{d}} \coloneqq 1, \label{eq:lcb-adv-A-Y-w}
\end{align}
which satisfies
\begin{align*}
    \left\|W_{h+1}^{k}(s,a)\right\|_\infty \leq 4H^2 \eqqcolon C_{\mathrm{w}}.
\end{align*}
Applying Lemma~\ref{lemma:martingale-union-recursion} with \eqref{eq:lcb-adv-A-Y-w} yields that: with probability at least $1-\delta$,
\begin{align}
   &\left|\sum_{k=1}^K \left(A_{h,k} - Y_{h,k}\right)\right| = \left|\sum_{k=1}^K \overline{X}_{h,k}\right|\nonumber \\
   &\leq  \sqrt{\sum_{k=1}^K 8C_{\mathrm{d}}^2 C^\star \sum_{(s,a)\in\cS\times\cA} d_{h}^{\pi_{\star}}(s,a) P_{h,s,a}\left[W_{h+1}^{k}(s,a)\right]^2 \log\frac{2H}{\delta}} + 2C_{\mathrm{d}} C^\star C_\mathrm{w}\log\frac{2H}{\delta} \nonumber \\
   & \lesssim  \sqrt{C^\star \log\frac{2H}{\delta}\sum_{k=1}^K  \sum_{(s,a)\in\cS\times\cA} d_{h}^{\pi_{\star}}(s,a) P_{h,s,a}\left[\sum_{n=1}^{N_h^k(s,a)} \eta_n^{N_h^k(s,a)} \left(V^{k^n}_{h+1} - \overline{V}^{k^n}_{h+1}\right)^2 \right]^2 } + C^\star H^2\log\frac{2H}{\delta} \label{lcb-adv-bound-A-Y-2}.
\end{align}
To further 
 control the first term in \eqref{lcb-adv-bound-A-Y-2}, it follows from Jensen's inequality that 
\begin{align}
  P_{h, s,a}\left[\sum_{n = 1}^{{N_h^k}} \eta_n^{N_h^k} \left(V^{k^n}_{h+1} - \overline{V}^{ k^n}_{h+1}\right)^2\right]^2  &\leq P_{h, s,a}\sum_{n = 1}^{{N_h^k}} \eta_n^{N_h^k} \Big(V^{k^n}_{h+1} - \overline{V}^{ k^n}_{h+1}\Big)^4,
\end{align}
which yields
\begin{align}
 & \sum_{k=1}^K  \sum_{(s,a)\in\cS\times\cA} d_{h}^{\pi_{\star}}(s,a) P_{h,s,a}\left[\sum_{n=1}^{N_h^k(s,a)} \eta_n^{N_h^k(s,a)} \left(V^{k^n}_{h+1} - \overline{V}^{k^n}_{h+1}\right)^2 \right]^2 \nonumber \\
 &\leq \sum_{k=1}^K  \sum_{(s,a)\in\cS\times\cA} d_{h}^{\pi_{\star}}(s,a) P_{h, s,a} \sum_{n = 1}^{{N_h^k}} \eta_n^{N_h^k} \Big(V^{k^n}_{h+1} - \overline{V}^{ k^n}_{h+1}\Big)^4 \nonumber \\
  & \leq 
 \left(1+\frac{1}{H}\right) \sum_{k=1}^K  \sum_{s\in \cS} d_{h+1}^{\pi^\star}(s) \left(V^{k}_{h+1}(s) - \overline{V}_{h+1}^{k}(s)\right)^4 + 32\sqrt{H^8C^\star K \log\frac{2H}{\delta}} + 32H^4C^\star\log\frac{2H}{\delta}. \label{eq:recursion-bound-3}
\end{align}
This can be verified similar to the proof for Lemma~\ref{lemma:recursion} in Appendix~\ref{proof:lemma-lcb-revursion}. We omit the details for conciseness.
To continue, it follows that 
\begin{align}
&  \sum_{k=1}^K \sum_{s\in\cS} d_{h+1}^{\pi^\star}(s)\left(V_{h+1}^k(s) - \overline{V}_{h+1}^k(s)\right)^4 \nonumber \\
  &\overset{\mathrm{(i)}}{\leq}   \sum_{m=1}^M \sum_{t=1}^{L_m} \sum_{s\in\cS} d_{h+1}^{\pi^\star}(s)\left(V_{h+1}^{\star}(s) - \overline{V}_{h+1}^{(m,t)}(s)\right)^4 \nonumber \\
  & \overset{\mathrm{(ii)}}{=}  \sum_{m=1}^M \sum_{t=1}^{L_m} \sum_{s\in\cS} d_{h+1}^{\pi^\star}(s)\left(V_{h+1}^{\star}(s) - V^{\left((m-1) \vee 1 ,1\right)}_{h+1}(s)\right)^4 \nonumber\\
  &\overset{\mathrm{(iii)}}{=}  \sum_{s\in\cS} d_{h+1}^{\pi^\star}(s)  \sum_{m=1}^M 2^{m} \left(V_{h+1}^{\star}(s) - V^{\left((m-1) \vee 1 ,1\right)}_{h+1}(s)\right)^4 \nonumber \\
  &= 4 \sum_{s\in\cS} d_{h+1}^{\pi^\star}(s)  \sum_{m-2=-1}^{M-2} 2^{m-2} \left(V_{h+1}^{\star}(s) - V^{\left((m-1) \vee 1 ,1\right)}_{h+1}(s)\right)^4 \nonumber \\
  & = 4 \sum_{m-2 = -1}^{0} 2^{m-2} \left(V_{h+1}^{\star}(s) - V^{\left(1 ,1\right)}_{h+1}(s)\right)^4  + 4  \sum_{s\in\cS} d_{h+1}^{\pi^\star}(s)  \sum_{m-2=1}^{M-2} 2^{m-2} \left(V_{h+1}^{\star}(s) - V^{\left(m-1 ,1\right)}_{h+1}(s)\right)^4. \nonumber 
\end{align}
Here, (i) holds by using the pessimistic property $V^\star \geq V^k \geq \overline{V}^k$ for all $k\in[K]$ (see \eqref{eq:lcb-adv-lower}) and by regrouping the summands;  
 (ii) follows from the fact (see updating rules in line~\ref{eq:update-mu-reference-v-k} and line~\ref{eq:update-mu-reference-v-next-k}) that for any $(m,s,h)\in [M] \times \cS\times [H + 1]$, 
\begin{align}
  \overline{V}^{(m,t)}_{h}(s) = V^{\left((m-1) \vee 1 ,1\right)}_{h}(s), \qquad t =1,2,\cdots, L_m;
\end{align}
and (iii) results from the choice of the parameter $L_m = 2^m$. 
In addition, we can further control 
\begin{align}
  \sum_{k=1}^K \sum_{s\in\cS} d_{h+1}^{\pi^\star}(s)\left(V_{h+1}^k(s) - \overline{V}_{h+1}^k(s)\right)^4  & \overset{\mathrm{(iv)}}{\leq} 8H^4 + 4\sum_{s\in\cS} d_{h+1}^{\pi^\star}(s) \sum_{m=1}^{M-2}  \sum_{t=1}^{L_m}   \left(V_{h+1}^{\star}(s) - V_{h+1}^{(m+1,1)}(s)\right)^4 \nonumber \\
  & \overset{\mathrm{(v)}}{\leq} 8H^4 + 4\sum_{s\in\cS} d_{h+1}^{\pi^\star}(s) \sum_{m=1}^{M-2}  \sum_{t=1}^{L_m}   \left(V_{h+1}^{\star}(s) - V_{h+1}^{(m,t)}(s)\right)^4 \nonumber \\
  & \leq 8H^4 +  4 \sum_{s\in\cS} d_{h+1}^{\pi^\star}(s) \sum_{k=1}^{K}  \left(V_{h+1}^{\star}(s) - V_{h+1}^{k}(s)\right)^4 \label{eq:lcb-adv-important2-inter}\\
  & \leq 8H^4 +  4H^3 \sum_{s\in\cS} d_{h+1}^{\pi^\star}(s) \sum_{k=1}^{K}  \left(V_{h+1}^{\star}(s) - V_{h+1}^{k}(s)\right) \nonumber\\
  &\overset{\mathrm{(vi)}}{\lesssim} H^3K +  H^8SC^\star \iota. \label{eq:lcb-adv-important2} 
\end{align}
Here, (iv) follows from the fact $0 \leq V_{h+1}^{\star}(s) - V^{\left(1 ,1\right)}_{h+1}(s) \leq H - 0 = H$; (v) holds since $V_{h+1}^\star \geq V_{h+1}^{(m+1,1)} = V_{h+1}^{(m,L_m)} \geq V_{h+1}^{(m,t)}$ for all $t\in [L_m]$ (using the monotonic increasing property of $V_{h+1}$ introduced in \eqref{equ:monotone-lcb-adv});
and (vi) follows from \eqref{equ:146-3}.
Putting \eqref{eq:lcb-adv-important2} and \eqref{eq:recursion-bound-3} together with \eqref{lcb-adv-bound-A-Y-2}, we arrive at
\begin{align}
 \left|\sum_{k=1}^K \left(A_{h,k} - Y_{h,k}\right)\right|  
   &\lesssim \sqrt{C^\star \log\frac{2H}{\delta}\left( H^3K +  H^8SC^\star \iota + \sqrt{H^8C^\star K \log\frac{2H}{\delta}} + H^4C^\star\log\frac{2H}{\delta} \right) } + C^\star H^2\log\frac{2H}{\delta} \nonumber \\
   &\lesssim \sqrt{H^3C^\star K\iota} + H^{4}\sqrt{S}C^\star \iota^2. \label{eq:lcb-adv-A-Y-bound2}
\end{align}

\paragraph{Step 3: controlling $\sum_{k=1}^K \left(Z_{h,k} - B_{h,k}\right)$.}
Similarly, we also invoke Lemma~\ref{lemma:martingale-union-recursion} to control $\sum_{k=1}^K \left(Z_{h,k} - B_{h,k}\right)$. 
Let's set
\begin{align}
  W_{h+1}^{k}(s,a) \coloneqq \left(V^{k}_{h+1} - \overline{V}^{k}_{h+1}\right)^2, \qquad C_{\mathrm{d}} \coloneqq \left(1+\frac{1}{H}\right)\leq 2,  \label{eq:lcb-adv-B-Z-w-2}
\end{align}
which satisfies
\begin{align*}
    \left\|W_{h+1}^{k}(s,a)\right\|_\infty \leq 4H^2 \eqqcolon C_{\mathrm{w}}.
\end{align*}
Applying Lemma~\ref{lemma:martingale-union-recursion} with \eqref{eq:lcb-adv-B-Z-w-2} yields that: with probability at least $1-\delta$, 
\begin{align}
   &\left|\sum_{k=1}^K \left(B_{h,k} - Z_{h,k}\right)\right| = \left|\sum_{k=1}^K \overline{X}_{h,k}\right| \nonumber\\
   &\leq  \sqrt{\sum_{k=1}^K 8C_{\mathrm{d}}^2 C^\star \sum_{(s,a)\in\cS\times\cA} d_{h}^{\pi_{\star}}(s,a) P_{h,s,a}\left[W_{h+1}^{k}(s,a)\right]^2 \log\frac{2H}{\delta}} + 2C_{\mathrm{d}} C^\star C_\mathrm{w}\log\frac{2H}{\delta} \nonumber \\
   & \lesssim  \sqrt{C^\star \log\frac{2H}{\delta}\sum_{k=1}^K  \sum_{(s,a)\in\cS\times\cA} d_{h}^{\pi_{\star}}(s,a) P_{h,s,a}\left[V^{k}_{h+1} - \overline{V}^{k}_{h+1}\right]^4 } + C^\star H^2\log\frac{2H}{\delta} \nonumber\\
   &\overset{\mathrm{(i)}}{\lesssim}  \sqrt{C^\star\log\frac{2H}{\delta} \left(H^3K +  H^8SC^\star \iota \right)} + C^\star H^2\log\frac{2H}{\delta} \lesssim   \sqrt{H^3C^\star K\iota} + H^{4}\sqrt{S}C^\star \iota^2, \label{lcb-adv-bound-A-Z-2}
\end{align}
where (i) follows from \eqref{eq:lcb-adv-important2-inter} and \eqref{eq:lcb-adv-important2}.

\paragraph{Step 4: combining the results.}
Inserting \eqref{lcb-adv-bound-A-Z-2} and \eqref{eq:lcb-adv-A-Y-bound2} back into \eqref{eq:lcb-adv-decompose-A-2}, we can conclude that
\begin{align}
&\sum_{k = 1}^K \sum_{h=1}^H \sum_{(s,a)\in \cS \times \cA} d_h^{\pi^\star}(s, a) \sum_{n = 1}^{N^{k}_h(s, a)} \eta^{N^{k}_h(s, a)}_n P_h^{k^n}\left(V^{k^n}_{h+1} - \overline{V}^{k^n}_{h+1}\right)^2 = \sum_{h=1}^H A_h\nonumber\\
& \leq \sum_{h=1}^H \sum_{k=1}^K B_{h,k} + \sum_{h=1}^H\sum_{k=1}^K \left(Z_{h,k} - B_{h,k}\right) + \sum_{h=1}^H\sum_{k=1}^K \left(A_{h,k} - Y_{h,k}\right) \nonumber\\
   &\leq \sum_{h=1}^H \sum_{k=1}^K \left(1+\frac{1}{H}\right) \sum_{s\in \cS} d_{h+1}^{\pi^\star}(s) \left(V^{k}_{h+1}(s) - \overline{V}^{k}_{h+1}(s)\right)^2 + \sum_{h=1}^H\left|\sum_{k=1}^K \left(Z_{h,k} - B_{h,k}\right)\right| + \sum_{h=1}^H\left|\sum_{k=1}^K \left(A_{h,k} - Y_{h,k}\right)\right| \nonumber\\
   & \leq H \sum_{h=1}^H \sum_{k=1}^K \left(1+\frac{1}{H}\right) \sum_{s\in \cS} d_{h+1}^{\pi^\star}(s) \left(V^{k}_{h+1}(s) - \overline{V}^{k}_{h+1}(s)\right) + \sqrt{H^5C^\star K\iota} + H^{5}\sqrt{S}C^\star \iota^2 \nonumber \\
   & \overset{\mathrm{(i)}}{\lesssim} H \sum_{h=1}^H \sum_{k=1}^K \sum_{s\in \cS} d_{h+1}^{\pi^\star}(s) \left(V^{\star}_{h+1}(s) - V^{k}_{h+1}(s)\right) + K + H^5\sqrt{S}C^\star\iota^2 \nonumber \\
   & \lesssim H^2 \max_{h\in [H] } \sum_{k=1}^K \sum_{s\in \cS} d_{h}^{\pi^\star}(s) \left(V^{\star}_{h}(s) - V^{k}_{h}(s)\right) + K + H^5\sqrt{S}C^\star\iota^2,  
\end{align}
where (i) follows from the same routine to obtain \eqref{eq:lcb-adv-important2-inter} and the Cauchy-Schwarz inequality.

\subsubsection{Proof of inequality~\eqref{adv-J-22-L}}\label{proof:adv-J-22-L}

\paragraph{Step 1: decomposing the error in \eqref{adv-J-22-L}.}
 The term in \eqref{adv-J-22-L} obeys
     \begin{align}
      &\sum_{h = 1}^H \sum_{(s,a)\in \cS \times \cA} d_h^{\pi^\star}(s, a)\sum_{k = 1}^K\sqrt{\frac{1}{N^{k}_h(s, a) \vee 1}} L_{h,k} \nonumber\\
      &= \sum_{h = 1}^H \sum_{(s,a)\in \cS \times \cA} d_h^{\pi^\star}(s, a)\sum_{k = 1}^K \sqrt{\frac{1}{N^{k}_h(s, a) \vee 1}} \sqrt{\frac{\sum_{n = 1}^{N^{k}_h(s,a)}2H\left(V^{\star}_{h+1}(s^{k^n}_{h+1}) - \overline{V}^{\nnext,k^n}_{h+1}(s^{k^n}_{h+1})\right)}{N^{k}_h(s,a) \vee 1} } \nonumber\\
      & \overset{\mathrm{(i)}}{\lesssim} \sqrt{H\iota} \sum_{h = 1}^H \sum_{(s,a)\in \cS \times \cA} \sum_{k = 1}^K \sqrt{\frac{d_h^{\pi^\star }(s, a)}{k d_h^{\mu}(s,a)}} \sqrt{\frac{d_h^{\pi^\star}(s, a)\iota \sum_{n = 1}^{N^{k}_h(s,a)}\left(V^{\star}_{h+1}(s^{k^n}_{h+1}) - \overline{V}^{\nnext,k^n}_{h+1}(s^{k^n}_{h+1})\right)}{k d_h^{\mu}(s,a)} } \nonumber\\
      & \overset{\mathrm{(ii)}}{\lesssim} \sqrt{HC^\star \iota^2} \sum_{h = 1}^H \sum_{(s,a)\in \cS \times \cA} \sum_{k = 1}^K \sqrt{\frac{\ind(a=\pi^\star(s))}{k }} \sqrt{\frac{d_h^{\pi^\star}(s, a)\sum_{n = 1}^{N^{k}_h(s,a)}\left(V^{\star}_{h+1}(s^{k^n}_{h+1}) - \overline{V}^{\nnext,k^n}_{h+1}(s^{k^n}_{h+1})\right)}{k d_h^{\mu}(s,a)} } \nonumber\\
      & \overset{\mathrm{(iii)}}{\lesssim} \sqrt{HC^\star \iota^2} \sqrt{\sum_{h = 1}^H \sum_{(s,a)\in \cS \times \cA} \sum_{k = 1}^K \frac{ d_h^{\pi^\star}(s, a) \sum_{n = 1}^{N^{k}_h(s,a)}\left(V^{\star}_{h+1}(s^{k^n}_{h+1}) - \overline{V}^{\nnext,k^n}_{h+1}(s^{k^n}_{h+1})\right)}{k d_h^{\mu}(s,a)} } \sqrt{\sum_{h = 1}^H \sum_{(s,a)\in \cS\times \cA} \sum_{k = 1}^K \frac{\ind(a=\pi^\star(s))}{k}} \nonumber  \\
      &\lesssim \sqrt{H^2SC^\star \iota^3} \sqrt{\sum_{h = 1}^H \sum_{(s,a)\in \cS \times \cA}\frac{d_h^{\pi^\star}(s, a)}{d_h^{\mu}(s,a)} \sum_{k = 1}^K \frac{1}{k}\sum_{n = 1}^{N^{k}_h(s,a)}\left(V^{\star}_{h+1}(s^{k^n(s,a)}_{h+1}) - \overline{V}^{\nnext,k^n}_{h+1}(s^{k^n(s,a)}_{h+1})\right) } \nonumber\\
      &\overset{\mathrm{(iv)}}{=} \sqrt{H^2SC^\star \iota^3} \sqrt{\sum_{h = 1}^H \sum_{k = 1}^K \frac{d_h^{\pi^\star}(s_h^k,a_h^k) } {d_h^{\mu}(s_h^k,a_h^k)} P_h^k \sum_{k'=k}^K \frac{1}{k'}(V^{\star}_{h+1} - \overline{V}^{\nnext,k}_{h+1}) } \nonumber\\
      &\lesssim \sqrt{H^2SC^\star \iota^4} \sqrt{\sum_{h = 1}^H \sum_{k = 1}^K \frac{d_h^{\pi^\star}(s_h^k,a_h^k) P_h^k} {d_h^{\mu}(s_h^k,a_h^k)} (V^{\star}_{h+1} - \overline{V}^{\nnext,k}_{h+1}) }. 
    \end{align}
   Here, (i) follows from the fact $\frac{1}{N_h^k(s,a) \vee 1} \leq \frac{8\iota}{k d_h^{\mu}(s,a)}$ (cf.~Lemma~\ref{lem:binomial}); (ii) follows from the definition of $C^\star$ in Assumption~\ref{assumption}; (iii) invokes the Cauchy-Schwarz inequality; (iv) can be obtained by regrouping the terms 
   (the terms involving $(V^{\star}_{h+1} - \overline{V}^{\nnext,k}_{h+1})$ associated with index $k$ will only been added during episodes $k' = k, k+1,\cdots, K$).

With this upper bound in hand, we further decompose 
\begin{align}
&\sum_{h = 1}^H \sum_{(s,a)\in \cS \times \cA} d_h^{\pi^\star}(s, a)\sum_{k = 1}^K\sqrt{\frac{1}{N^{k}_h(s, a) \vee 1}} L_{h,k} \lesssim \sqrt{H^2SC^\star \iota^4} \sqrt{\sum_{h = 1}^H \sum_{k = 1}^K \frac{d_h^{\pi^\star}(s_h^k,a_h^k) P_h^k} {d_h^{\mu}(s_h^k,a_h^k)} (V^{\star}_{h+1} - \overline{V}^{\nnext,k}_{h+1}) } \nonumber\\
&\overset{\mathrm{(i)}}{\lesssim} \sqrt{H^2SC^\star \iota^4} \sqrt{\sum_{h = 1}^H \sum_{k = 1}^K \frac{d_h^{\pi^\star}(s_h^k,a_h^k) } {d_h^{\mu}(s_h^k,a_h^k)} P_h^k \left(V_{h+1}^\star - \overline{V}_{h+1}^k\right)}\nonumber\\
  &\overset{\mathrm{(ii)}}{\lesssim} \sqrt{H^2SC^\star \iota^4} \sqrt{\sum_{h = 1}^H \sum_{k = 1}^K  \sum_{(s,a)\in \cS\times \cA}d_h^{\pi^\star}(s,a) P_{h, s,a} \left(V^{\star}_{h+1} - \overline{V}^{k}_{h+1}\right) } \nonumber \\
      & \qquad + \sqrt{H^2SC^\star \iota^4}\sqrt{\left|\sum_{h = 1}^H \sum_{k = 1}^K \left( \sum_{(s,a)\in \cS\times \cA}d_h^{\pi^\star}(s,a) P_{h, s,a} - \frac{d_h^{\pi^\star}(s_h^k,a_h^k) } {d_h^{\mu}(s_h^k,a_h^k)}P_h^k\right)\left(V^{\star}_{h+1} - \overline{V}^{k}_{h+1}\right)\right|}. 
	\label{equ:adv-L-decompose}
\end{align}
Here (i) holds due to the following observation:  denoting by $m$  the index of the epoch in which episode $k$ occurs, we have
\begin{align}
  &\overline{V}_{h+1}^{\nnext,k}=V_{h+1}^{(m,1)} \geq V_{h+1}^{((m-1 \vee 1),1)} = \overline{V}_{h+1}^{k},
\end{align}
which invokes the monotonicity of $V_{h+1}^k$ in \eqref{equ:monotone-lcb-adv}. In addition,  (ii) arises from the Cauchy-Schwarz inequality.

\paragraph{Step 2: controlling the first term in \eqref{equ:adv-L-decompose}.}
    The first term in \eqref{equ:adv-L-decompose} satisfies
    \begin{align}
   \sum_{h = 1}^H \sum_{k = 1}^K  \sum_{(s,a)\in \cS\times \cA}d_h^{\pi^\star}(s,a) P_{h, s,a} \left(V^{\star}_{h+1} - \overline{V}^{k}_{h+1}\right) 
   &= \sum_{h = 1}^H \sum_{k = 1}^K  \sum_{(s,a)\in \cS\times \cA}d_h^{\pi^\star}(s,a)  \big\langle P_h(\cdot \mymid s,a), V^{\star}_{h+1} - \overline{V}^{k}_{h+1}  \big\rangle  \nonumber \\
 &   \overset{\mathrm{(i)}}{=} \sum_{h = 1}^H \sum_{k = 1}^K \sum_{s' \in \cS} d_{h+1}^{\pi^\star}(s') \left(V^{\star}_{h+1}(s') - \overline{V}^{k}_{h+1}(s') \right)  \nonumber\\
   &\overset{\mathrm{(ii)}}{\lesssim} H^2 + \sum_{h = 1}^H \sum_{k = 1}^K \sum_{s \in \cS} d_{h+1}^{\pi^\star}(s) \left(V^{\star}_{h+1}(s) - V^{k}_{h+1}(s)\right) \nonumber\\
   &\overset{\mathrm{(iii)}}{\lesssim} HK + H^6 S C^{\star}\iota,\label{equ:adv-L-1-true}
    \end{align}
  where (i) holds due to the fact $\sum_{(s,a)\in \cS\times \cA}d_h^{\pi^\star}(s,a) P_h(\cdot \mymid s,a) = d_{h+1}^{\pi^\star}(\cdot)$, 
  (ii) comes from the same argument employed to establish \eqref{eq:lcb-adv-important2-inter}, and (iii) follows from \eqref{equ:146-3}.

\paragraph{Step 3: controlling the second  term in \eqref{equ:adv-L-decompose}.}
We shall invoke Lemma~\ref{lemma:martingale-union-recursion} for this purpose.
To proceed, let
\begin{align}
  W_{h+1}^{k}(s,a) \coloneqq V^{\star}_{h+1} - \overline{V}^{k}_{h+1}, \qquad C_{\mathrm{d}} \eqqcolon 1, \label{eq:lcb-adv-216-w}
\end{align}
which satisfies
\begin{align*}
    \left\|W_{h+1}^{k}(s,a)\right\|_\infty \leq H \eqqcolon C_{\mathrm{w}}.
\end{align*}
Applying Lemma~\ref{lemma:martingale-union-recursion} with \eqref{eq:lcb-adv-216-w} yields, for all $h\in[H]$, with probability at least $1-\delta$
\begin{align}
   \Bigg|\sum_{k = 1}^K &\left( \sum_{(s,a)\in \cS\times \cA}d_h^{\pi^\star}(s,a) P_{h, s,a} - \frac{d_h^{\pi^\star}(s_h^k,a_h^k) } {d_h^{\mu}(s_h^k,a_h^k)}P_h^k\right)\left(V^{\star}_{h+1} - \overline{V}^{k}_{h+1}\right) \Bigg|= \left|\sum_{k=1}^K \overline{X}_{h,k}\right| \nonumber\\
   &\leq  \sqrt{\sum_{k=1}^K 8C_{\mathrm{d}}^2 C^\star \sum_{(s,a)\in\cS\times\cA} d_{h}^{\pi_{\star}}(s,a) P_{h,s,a}\left[W_{h+1}^{k}(s,a)\right]^2 \log\frac{2H}{\delta}} + 2C_{\mathrm{d}} C^\star C_\mathrm{w}\log\frac{2H}{\delta} \nonumber \\
    &\lesssim \sqrt{C^\star\log\frac{2H}{\delta}\sum_{k=1}^K  \sum_{(s,a)\in\cS\times\cA} d_{h}^{\pi_{\star}}(s,a) P_{h,s,a}\left(V^{\star}_{h+1} - \overline{V}^{k}_{h+1}\right)^2 } + HC^\star\log\frac{2H}{\delta}\nonumber \\
    &\overset{\mathrm{(i)}}{\lesssim} \sqrt{C^\star\log\frac{2H}{\delta} \left(H^2 + \sum_{k=1}^K  \sum_{(s,a)\in\cS\times\cA} d_{h}^{\pi_{\star}}(s,a) P_{h,s,a}\left(V^{\star}_{h+1} - V^k_{h+1}\right)^2 \right)} + HC^\star\log\frac{2H}{\delta}\nonumber \\
   &\overset{\mathrm{(ii)}}{\lesssim} \sqrt{C^\star\log\frac{2H}{\delta}\left(HK +H^6SC^\star \iota\right)} + HC^\star\log\frac{2H}{\delta} \nonumber\\
   &
   \lesssim \sqrt{HC^\star K \iota} + H^3\sqrt{S}C^\star \iota.
\end{align}
Here (i) follows from the same routine to arrive at \eqref{eq:lcb-adv-important2-inter}, and (ii) comes from \eqref{equ:146-3}. As a result, the second term in \eqref{equ:adv-L-decompose} satisfies, with probability at least $1-\delta$,
\begin{align}
&\left| \sum_{h = 1}^H \sum_{k = 1}^K \left( \sum_{(s,a)\in \cS\times \cA}d_h^{\pi^\star}(s,a) P_{h, s,a} - \frac{d_h^{\pi^\star}(s_h^k,a_h^k) P_h^k} {d_h^{\mu}(s_h^k,a_h^k)}\right)\left(V^{\star}_{h+1} - \overline{V}^{k}_{h+1}\right)\right| \nonumber\\
&\leq \sum_{h=1}^H \left|\sum_{k = 1}^K \left( \sum_{(s,a)\in \cS\times \cA}d_h^{\pi^\star}(s,a) P_{h, s,a} - \frac{d_h^{\pi^\star}(s_h^k,a_h^k) P_h^k} {d_h^{\mu}(s_h^k,a_h^k)}\right)\left(V^{\star}_{h+1} - \overline{V}^{k}_{h+1}\right)\right|   \lesssim \sqrt{H^3C^\star K \iota} + H^4\sqrt{S}C^\star \iota.
  \label{equ:adv-L-2}
\end{align}

\paragraph{Step 4: combining the results.}
Finally, inserting \eqref{equ:adv-L-1-true} and \eqref{equ:adv-L-2} into \eqref{equ:adv-L-decompose}, we arrive at
\begin{align}
  &\sum_{h = 1}^H \sum_{(s,a)\in \cS \times \cA} d_h^{\pi^\star}(s, a)\sum_{k = 1}^K\sqrt{\frac{1}{N^{k}_h(s, a)}} L_{h,k} \nonumber\\
  &\lesssim \sqrt{H^2SC^\star \iota^4}
  \sqrt{HK + H^6 S C^{\star}\iota}  + \sqrt{H^2SC^\star \iota^4} \sqrt{\sqrt{H^3C^\star K \iota} + H^4\sqrt{S}C^\star \iota} \nonumber\\
  &\lesssim \sqrt{H^3S C^\star K \iota^4} + H^{4} S C^\star \iota^3 + \sqrt{H^2SC^\star \iota^4} \sqrt{HK + H^4\sqrt{S}C^\star \iota} \lesssim \sqrt{H^3S C^\star K \iota^4} + H^{4} S C^\star \iota^3,
\end{align}
where the last two inequalities follow from the Cauchy-Schwarz inequality.

\subsubsection{Proof of inequality~\eqref{adv-J-22-G}}\label{proof:adv-J-22-G}

Recall the expression of $G_{h,k}$ in \eqref{equ:adv-J22-terms} as
\begin{align}
G_{h,k}^2 &= \frac{\sum_{n = 1}^{N^{k}_h(s,a)}\left(V^{\star}_{h+1}(s^{k^n}_{h+1})\right)^2}{N^{k}_h(s,a) \vee 1} - \Big(\frac{\sum_{n = 1}^{N^{k}_h(s,a) }V^{\star}_{h+1}(s^{k^n}_{h+1})}{N^{k}_h(s,a) \vee 1}\Big)^2 \nonumber \\
& = \frac{\sum_{n = 1}^{N^{k}_h(s,a)} P_h^{k^n}\left(V^{\star}_{h+1}\right)^2}{N^{k}_h(s,a) \vee 1} - \Big(\frac{\sum_{n = 1}^{N^{k}_h(s,a) } P_h^{k^n} V^{\star}_{h+1}}{N^{k}_h(s,a) \vee 1}\Big)^2.
\end{align}
To continue, we make the following observation
\begin{align}  
G_{h,k} 
	&\leq \left\{ \left|G_{h,k}^2 - \Var_{h, s,a}(V^{\star}_{h+1}) \right| + \Var_{h, s,a}(V^{\star}_{h+1}) \right\}^{1/2} \notag\\
&\leq \left|G_{h,k}^2 - \Var_{h, s,a}(V^{\star}_{h+1}) \right|^{1/2} + \sqrt{\Var_{h, s,a}(V^{\star}_{h+1}) } 
\end{align}
due to the elementary inequality $\sqrt{a^2 + b^2} \leq a + b$ for any $a, b\geq0$.
Here, we remind the reader that $\Var_{h, s,a}(V^{\star}_{h+1}) = P_{h, s,a}(V^{\star}_{h+1})^2 -  (P_{h, s,a} V^{\star}_{h+1})^2$ (cf.~\eqref{lemma1:equ2}). This allows us to rewrite
\begin{align}\label{eq:control_Ghk}
& \sum_{h = 1}^H \sum_{(s,a)\in \cS \times \cA} d_h^{\pi^\star}(s, a)\sum_{k = 1}^K\sqrt{\frac{1}{N^{k}_h(s, a) \vee 1}} G_{h,k}  \nonumber \\
& \leq \sum_{h = 1}^H \sum_{(s,a)\in \cS \times \cA} d_h^{\pi^\star}(s, a)\sum_{k = 1}^K\sqrt{\frac{ \left|G_{h,k}^2 - \Var_{h, s,a}(V^{\star}_{h+1}) \right|}{N^{k}_h(s, a) \vee 1}}    + \sum_{h = 1}^H \sum_{(s,a)\in \cS \times \cA} d_h^{\pi^\star}(s, a)\sum_{k = 1}^K\sqrt{\frac{ \Var_{h, s,a}(V^{\star}_{h+1})}{N^{k}_h(s, a) \vee 1}}, 
 \end{align}
leaving us with two terms to cope with.

\paragraph{Step 1: controlling the first term of \eqref{eq:control_Ghk}.}
By definition, we have
\begin{align} 
&  \left|G_{h,k}^2 - \Var_{h, s,a}(V^{\star}_{h+1}) \right|  
 = \left|\frac{\sum_{n = 1}^{N^{k}_h(s,a)} P_h^{k^n}\left(V^{\star}_{h+1}\right)^2}{N^{k}_h(s,a) \vee 1} - \Big(\frac{\sum_{n = 1}^{N^{k}_h(s,a) } P_h^{k^n} V^{\star}_{h+1}}{N^{k}_h(s,a) \vee 1}\Big)^2 -P_{h, s,a}(V^{\star}_{h+1})^2 + \left(P_{h, s,a} V^{\star}_{h+1}\right)^2 \right| \nonumber \\
&\leq \left|\frac{\sum_{n = 1}^{N^{k}_h(s,a)} P_h^{k^n}\left(V^{\star}_{h+1}\right)^2}{N^{k}_h(s,a) \vee 1}  - P_{h, s,a}(V^{\star}_{h+1})^2  \right| + \left| \Big(\frac{\sum_{n = 1}^{N^{k}_h(s,a) } P_h^{k^n} V^{\star}_{h+1}}{N^{k}_h(s,a) \vee 1}\Big)^2 - \left(P_{h, s,a}V^{\star}_{h+1}\right)^2 \right| \nonumber \\
& \leq \left|\frac{\sum_{n = 1}^{N^{k}_h(s,a)} P_h^{k^n}\left(V^{\star}_{h+1}\right)^2}{N^{k}_h(s,a) \vee 1}  - P_{h, s,a}(V^{\star}_{h+1})^2  \right|  + 2H \left| \frac{\sum_{n = 1}^{N^{k}_h(s,a) } P_h^{k^n} V^{\star}_{h+1}}{N^{k}_h(s,a) \vee 1} - P_{h, s,a}V^{\star}_{h+1}\right|, \label{equ:Ghk-var-diff}
\end{align}
where the last inequality holds due to
\begin{align*}
\left| \Big(\frac{\sum_{n = 1}^{N^{k}_h(s,a) } P_h^{k^n} V^{\star}_{h+1}}{N^{k}_h(s,a) \vee 1}\Big)^2 - \left(P_{h, s,a}V^{\star}_{h+1}\right)^2 \right|  &=   \left| \frac{\sum_{n = 1}^{N^{k}_h(s,a) } P_h^{k^n} V^{\star}_{h+1}}{N^{k}_h(s,a) \vee 1} - P_{h, s,a}V^{\star}_{h+1}\right| \cdot\left| \frac{\sum_{n = 1}^{N^{k}_h(s,a) } P_h^{k^n} V^{\star}_{h+1}}{N^{k}_h(s,a) \vee 1} + P_{h, s,a}V^{\star}_{h+1}\right|   \\
& \leq  2H \left| \frac{\sum_{n = 1}^{N^{k}_h(s,a) } P_h^{k^n} V^{\star}_{h+1}}{N^{k}_h(s,a) \vee 1} - P_{h, s,a}V^{\star}_{h+1}\right| .
\end{align*}

We now control the two terms in \eqref{equ:Ghk-var-diff} separately by invoking Lemma~\ref{lemma:martingale-union-all}. 
For the first term in \eqref{equ:Ghk-var-diff}, let us set
    \begin{align}\label{eq:261-w}
      W_{h+1}^{i} \coloneqq \left(V^\star_{h+1}\right)^2, \qquad \text{and} \qquad u_h^i(s,a,N) \coloneqq \frac{1}{N \vee 1} \coloneqq C_{\mathrm{u}},
    \end{align}
  which indicates that
  \begin{align}
    \|W_{h+1}^{i}\|_\infty \leq  H^2 \eqqcolon C_{\mathrm{w}}, \label{eq:261-cw}
  \end{align}
Applying Lemma~\ref{lemma:martingale-union-all} with \eqref{eq:261-w} and $N = N_h^k = N_h^{k}(s,a)$, with probability at least $1-\frac{\delta}{2}$, we arrive at 
\begin{align}
  &\left|\frac{1}{N_h^k(s,a) \vee 1}\sum_{n = 1}^{{N_h^k}} (P^{k^n}_h - P_{h, s,a})(V^{\star}_{h+1})^2\right|  = \left|\sum_{i=1}^k X_i\left(s,a, h, N_h^k\right)\right| \nonumber\\
  &\lesssim \sqrt{C_{\mathrm{u}} \log^2\frac{SAT}{\delta}}\sqrt{\sum_{n = 1}^{N_h^k} u_h^{k^n}(s,a, N_h^k) \Var_{h,s,a} \big(W_{h+1}^{k^n}  \big)} + \left(C_{\mathrm{u}} C_{\mathrm{w}} + \sqrt{\frac{C_{\mathrm{u}}}{N_h^k \vee 1}} C_{\mathrm{w}}\right) \log^2\frac{SAT}{\delta} \nonumber \\
  & \asymp \sqrt{\frac{\iota^2}{N_h^k \vee 1}}\sqrt{\sum_{n = 1}^{N_h^k} \frac{1}{N_h^k \vee 1}\|W_{h+1}^{k^n}\|_\infty^2 } + \frac{H^2\iota^2}{N_h^k \vee 1}  \lesssim H^2\iota^2\sqrt{\frac{1}{N_h^k \vee 1}}. \label{eq:261-result1}
\end{align}
Similarly, for the second term in \eqref{equ:Ghk-var-diff}, with $W_{h+1}^{i} \coloneqq V^\star_{h+1}$, we have with probability at least $1-\frac{\delta}{2}$,
\begin{align}
 \frac{1}{N_h^k(s,a) \vee 1}\sum_{n = 1}^{{N_h^k}} \left(P^{k^n}_h - P_{h, s,a}\right)V^{\star}_{h+1} \lesssim H\iota^2\sqrt{\frac{1}{N_h^k(s,a) \vee 1}}. \label{eq:261-result2}
\end{align}
Inserting \eqref{eq:261-result1} and \eqref{eq:261-result2} back into  \eqref{equ:Ghk-var-diff} yields
\begin{align}
  \left|G_{h,k}^2 - \Var_{h, s,a}(V^{\star}_{h+1}) \right|  \lesssim H^2\iota^2\sqrt{\frac{1}{N_h^k(s,a) \vee 1}}. \label{eq:var-Vstar}
\end{align}
Consequently, the first term in \eqref{eq:control_Ghk} can be controlled as
\begin{align}
\sum_{h = 1}^H \sum_{(s,a)\in \cS \times \cA} d_h^{\pi^\star}(s, a)\sum_{k = 1}^K\sqrt{\frac{ \left|G_{h,k}^2 - \Var_{h, s,a}(V^{\star}_{h+1}) \right|}{N^{k}_h(s, a) \vee 1}}  &  \lesssim      H\iota \sum_{h = 1}^H \sum_{(s,a)\in \cS \times \cA} d_h^{\pi^\star}(s, a)\sum_{k = 1}^K \frac{1}{\left(N_h^k(s,a)\right)^{\frac{3}{4}} \vee 1} \nonumber \\
& \lesssim H^2(SC^\star)^{\frac{3}{4}}  K^{\frac{1}{4}}\iota^2 , \label{equ:adv-J-22-G-1}
 \end{align}
 where the last inequality holds due to \eqref{equ:adv-J-12-result}.

\paragraph{Step 2: controlling the second term of \eqref{eq:control_Ghk}.}
The second term can be decomposed as
\begin{align}
  &\sum_{h = 1}^H \sum_{(s,a)\in \cS \times \cA} d_h^{\pi^\star}(s, a)\sum_{k = 1}^K\sqrt{\frac{\Var_{h, s,a}(V^{\star}_{h+1})}{N_h^k(s,a) \vee 1}} \nonumber\\
  & \overset{\mathrm{(i)}}{\lesssim} \sum_{h = 1}^H \sum_{(s,a)\in \cS \times \cA} \sum_{k = 1}^K \sqrt{\frac{C^\star \iota d_h^{\pi^\star}(s, a) \Var_{h, s,a}(V^{\star}_{h+1})}{k} \ind\left(a = \pi^\star_h(s)\right)} \nonumber \\
  & \overset{\mathrm{(ii)}}{\lesssim} \sqrt{C^\star \iota} \sqrt{\sum_{h = 1}^H \sum_{(s,a)\in \cS \times \cA} d_h^{\pi^\star}(s, a)\sum_{k = 1}^K  \Var_{h, s,a}(V^{\star}_{h+1}) } \sqrt{\sum_{h = 1}^H \sum_{(s,a) \in \cS\times \cA} \sum_{k = 1}^K  \frac{1}{k}\ind\left(a = \pi^\star_h(s)\right) } \nonumber\\
  &\lesssim \sqrt{H S C^\star K\iota^2} \sqrt{\sum_{h = 1}^H  \sum_{(s,a)\in \cS \times \cA} d_h^{\pi^\star}(s, a)\Var_{h, s,a}(V^{\star}_{h+1}) }, \label{equ:adv-J-22-G-2}
\end{align}
where (i) follows from the facts $\frac{1}{N_h^k(s,a) \vee 1} \leq \frac{8\iota}{k d_h^{\mu}(s,a)}$ by Lemma~\ref{lem:binomial} and the definition of $C^\star$ in Assumption~\ref{assumption}, (ii) holds by the Cauchy-Schwarz inequality, and the final inequality comes from the fact that $\pi^\star$ is deterministic.

We are then left with bounding $\sum_{h = 1}^H  \sum_{(s,a)\in \cS \times \cA} d_h^{\pi^\star}(s, a)\Var_{h, s,a}(V^{\star}_{h+1})$. Note that
\begin{align}
& \sum_{h = 1}^H  \sum_{(s,a)\in \cS \times \cA} d_h^{\pi^\star}(s, a)\Var_{h, s,a}(V^{\star}_{h+1}) =  \mathbb{E}_{s_1 \sim \rho, s_{h+1}\sim P_{h,s_{h},\pi_h^{\star}(s_h)}}\left[\sum_{h = 1}^H  \Var_{h, s_h,\pi_h^{\star}(s_h)}(V^{\star}_{h+1})\right] \nonumber \\
 &  \overset{\mathrm{(i)}}{=}  \mathbb{E}_{s_1 \sim \rho, s_{h+1}\sim P_{h,s_{h},\pi_h^{\star}(s_h)}} \left[\sum_{h = 1}^H   \left( r_h\left(s_h, \pi^\star_h(s_h)\right) + V^{\star}_{h+1}(s_{h+1}) - V_h^\star(s_h) \right)^2\right] \nonumber \\
 & \overset{\mathrm{(ii)}}{=} \mathbb{E}_{s_1 \sim \rho, s_{h+1}\sim P_{h,s_{h},\pi_h^{\star}(s_h)}} \left[\sum_{h = 1}^H   \left( r_h(s_h, \pi^\star_h(s_h)) + V^{\star}_{h+1}(s_{h+1}) - V_h^\star(s_h) \right)\right]^2 \nonumber\\
 & \overset{\mathrm{(iii)}}{=} \mathbb{E}_{s_1 \sim \rho, s_{h+1}\sim P_{h,s_{h},\pi_h^{\star}(s_h)}} \left[ \left(\sum_{h = 1}^H r_h(s_h, \pi^\star_h(s_h)) \right)- V_1^\star(s_1) \right]^2 \overset{\mathrm{(iv)}}{\leq} H^2 , \label{eq:var-Vstar-sum}
 \end{align}
 where (i) follows from Bellman's optimality equation, (ii) follows from the Markov property, 
 (iii) holds due to the fact that $V_{H+1}^\star(s) = 0$ for all $s\in \cS$,
 and (iv) arises from the fact $r_h(s,a) \leq 1$ for all $(s,a,h)\in\cS\times\cA\times[H]$.
Substituting \eqref{eq:var-Vstar-sum} back into \eqref{equ:adv-J-22-G-2}, we get
 \begin{align} \label{eq:second_term_Ghk}
\sum_{h = 1}^H \sum_{(s,a)\in \cS \times \cA} d_h^{\pi^\star}(s, a)\sum_{k = 1}^K\sqrt{\frac{\Var_{h, s,a}(V^{\star}_{h+1})}{N_h^k(s,a) \vee 1}}   &\lesssim \sqrt{H^3 S C^\star K\iota^2}.
 \end{align}
 
\paragraph{Step 4: combing the results.}
Combining \eqref{equ:adv-J-22-G-1} and \eqref{eq:second_term_Ghk} with \eqref{eq:control_Ghk} yields 
\begin{align}
 \sum_{h = 1}^H \sum_{(s,a)\in \cS \times \cA} d_h^{\pi^\star}(s, a)\sum_{k = 1}^K\sqrt{\frac{1}{N^{k}_h(s, a) \vee 1}} G_{h,k}  
  & \lesssim H^2(SC^\star)^{\frac{3}{4}}  K^{\frac{1}{4}}\iota^2 +  \sqrt{H^3 S C^\star K\iota^2}  \notag \\&\lesssim \sqrt{H^3 S C^\star K\iota^2} + H^{2.5}SC^\star\iota^3.
\end{align}

\subsection{Proof of Lemma~\ref{lem:lcb-adv-bonus-upper}} \label{proof:lem:lcb-adv-bonus-upper}

In view of \eqref{eq:defn-xi-kh-123}, we can decompose the term of interest into 
\begin{align*}
  \bigg| \sum_{n = 1}^{N_h^k(s,a)} \eta^{N_h^k(s,a)}_n \xi_h^{k^n}\bigg|  &\leq |U_1| +|U_2|,
\end{align*}
where
\begin{subequations}
\begin{align}
U_1 &\defn  \sum_{n=1}^{N_h^k} \eta_n^{N_h^k} \big(P^{k^n}_{h} - P_{h, s,a} \big)\big(\overline{V}^{ k^n}_{h+1} - V^{k^n}_{h+1}  \big), \label{eq:defn-U1-1}\\  
U_2 &\defn \sum_{n=1}^{N_h^k} \eta_n^{N_h^k}  \left(P_{h,s,a} - \frac{\sum_{i= N_h^{(m^n-1,1)} + 1 }^{ N_h^{(m^n,1)}} P_h^{k^i}}{\widehat{N}_h^{\epo, m^n-1}(s,a) \vee 1} \right)\overline{V}^{ k^n}_{h+1}. \label{eq:defn-U1-2}
\end{align}
\end{subequations}
Next, we turn to controlling these two terms separately with the assistance of Lemma~\ref{lemma:martingale-union-all}.

\paragraph{Step 1: controlling $U_1$.} In the following, we invoke Lemma~\ref{lemma:martingale-union-all} to control $U_1$ in \eqref{eq:defn-U1-1}. 
    Let us set
    \begin{align*} 
      W_{h+1}^{i} \coloneqq \overline{V}_{h+1}^i - V_{h+1}^i, \qquad \text{and} \qquad u_h^i(s,a,N) \coloneqq \eta_{N_h^i(s,a)}^N \geq 0,
    \end{align*}
  which indicates that
  \begin{align*}
    \|W_{h+1}^{i}\|_\infty \leq \|\overline{V}_{h+1}^i \|_{\infty} + \|V_{h+1}^i\|_{\infty} \leq 2H \eqqcolon C_{\mathrm{w}}, 
  \end{align*}
  and 
  \begin{align}
    \max_{N, h, s, a\in \big(\{0\} \cup [K]\big) \times[H] \times \cS\times \cA} \eta_{N_{h}^{i}(s,a)}^{N} \leq \frac{2H}{N \vee 1} 
    \eqqcolon C_{\mathrm{u}}. 
    \label{U1-cu}
  \end{align}
Here, the last inequality follows since (according to Lemma~\ref{lemma:property of learning rate} and the definition in \eqref{equ:learning rate notation}) 
\begin{align*}
  \eta_{N_{h}^{i}(s,a)}^{N} & \leq\frac{2H}{N \vee 1},\qquad\text{if } 0\le N_{h}^{i}(s,a)\le N;\\
  \eta_{N_{h}^{i}(s,a)}^{N} & =0,\qquad\quad ~\text{if }N_{h}^{i}(s,a)>N.
\end{align*}
To continue, it can be seen from \eqref{eq:sum-eta-n-N} that
\begin{align}
  0 \leq \sum_{n=1}^N u_h^{k_h^n(s,a)}(s,a, N) = \sum_{n=1}^N \eta_{n}^N \leq 1 \label{eq:sum-uh-k-s-a-123}
\end{align}
holds for all $(N,s,a) \in [K] \times \cS\times \cA$.
Therefore, choosing $N = N_h^k(s,a) = N_h^k$ for any $(s,a)$ and applying Lemma~\ref{lemma:martingale-union-all} with the above quantities, we arrive at  
\begin{align}
  |U_1| &= \left|\sum_{n=1}^{N_h^k} \eta_n^{N_h^k} \big(P^{k^n}_{h} - P_{h, s,a} \big)\big(\overline{V}^{ k^n}_{h+1} - V^{k^n}_{h+1}  \big)\right| = \left|\sum_{i=1}^k X_i\left(s,a, h, N_h^k\right)\right| \nonumber\\
  &\lesssim \sqrt{C_{\mathrm{u}} \log^2\frac{SAT}{\delta}}\sqrt{\sum_{n = 1}^{N_h^k} u_h^{k^n}(s,a, N_h^k) \Var_{h,s,a} \big(W_{h+1}^{k^n}  \big)} + \left(C_{\mathrm{u}} C_{\mathrm{w}} + \sqrt{\frac{C_{\mathrm{u}}}{N \vee 1}} C_{\mathrm{w}}\right) \log^2\frac{SAT}{\delta} \nonumber \\
  & \asymp \sqrt{\frac{H\iota^2}{N_h^k \vee 1}}\sqrt{\sum_{n = 1}^{N_h^k} \eta^{N_h^k}_n \Var_{h, s,a} \big(\overline{V}^{k^n}_{h+1}-V^{k^n}_{h+1}  \big) } + \frac{H^2\iota^2}{N_h^k \vee 1} \label{lemma1:equ4-sub} \\
  &  \lesssim \sqrt{\frac{H\iota^2}{N_h^k \vee 1}}\sqrt{\sigma^{\adv, k^{N_h^k}+1}_h(s,a) - \big(\mu^{\adv, k^{N_h^k}+1}_h(s,a)\big)^2} + \frac{H^{7/4}\iota^2}{(N_h^k \vee 1)^{3/4}} +  \frac{H^2\iota^2}{N_h^k \vee 1}.
  \label{lemma1:equ4}
\end{align}
with probability at least $1-\delta$. 
Here,
the proof of the inequality~\eqref{lemma1:equ4} is postponed to Appendix~\ref{sec:proof:lemma1:equ4} in order
to streamline the presentation of the analysis.

\paragraph{Step 2: bounding $U_2$.} Making use of the result in \eqref{equ:definition-ref-refmean}, we arrive at
    \begin{align*}
    \frac{\sum_{i= N_h^{(m^n-1,1)} + 1 }^{ N_h^{(m^n,1)}} P_h^{k^i}}{\widehat{N}_h^{\epo, m^n-1}(s,a) \vee 1} \overline{V}^{ k^n}_{h+1} =  \frac{\sum_{i= N_h^{(m^n-1,1)} + 1 }^{ N_h^{(m^n,1)}} P_h^{k^i} \overline{V}^{ \nnext,k^i}_{h+1}}{\widehat{N}_h^{\epo, m^n-1}(s,a) \vee 1} .
    \end{align*} 
To continue, for any $(s,a)\in \cS\times \cA$, we rewrite and rearrange $U_2$ (cf.~\eqref{eq:defn-U1-2}) as follows:
\begin{align*}
U_2 
& = \sum_{n=1}^{N_h^k} \eta_n^{N_h^k}  \left(P_{h,s,a} - \frac{\sum_{i= N_h^{(m^n-1,1)} + 1 }^{ N_h^{(m^n,1)}} P_h^{k^i}}{\widehat{N}_h^{\epo,m^n-1}(s,a) \vee 1} \right)\overline{V}^{ k^n}_{h+1} \\
	&= \sum_{n=1}^{N_h^k} \eta_n^{N_h^k}  \left(P_{h,s,a}\overline{V}^{k^n}_{h+1} - \frac{\sum_{i= N_h^{(m^n-1,1)} + 1 }^{ N_h^{(m^n,1)}} P_h^{k^i}\overline{V}^{ \nnext,k^i}_{h+1}}{\widehat{N}_h^{\epo, m^n-1}(s,a) \vee 1} \right)    \\
& \overset{\mathrm{(i)}}{=} \sum_{n=1}^{N_h^k} \eta_n^{N_h^k}  \left(\frac{\sum_{i= N_h^{(m^n-1,1)} + 1 }^{ N_h^{(m^n,1)}} P_{h,s,a}}{\widehat{N}_h^{\epo,m^n-1}(s,a) \vee 1} \overline{V}^{k^n}_{h+1} - \frac{\sum_{i= N_h^{(m^n-1,1)} + 1 }^{ N_h^{(m^n,1)}} P_h^{k^i}\overline{V}^{ \nnext,k^i}_{h+1}}{\widehat{N}_h^{\epo, m^n-1}(s,a) \vee 1} \right)   \\
&= \sum_{n=1}^{N_h^k} \frac{\eta_n^{N_h^k}}{\widehat{N}_h^{\epo, m^n-1}(s,a) \vee 1}  \sum_{i= N_h^{(m^n-1,1)} + 1 }^{ N_h^{(m^n,1)}}\left(P_{h,s,a} - P_h^{k^i}\right) \overline{V}^{ \nnext,k^i}_{h+1}\\
& \overset{\mathrm{(ii)}}{=} \sum_{i=1}^{N_h^k} \left(\sum_{n= N_h^{(m^i+1,1)}+1}^{N_h^{(m^i+2,1)} \wedge N_h^{k}} \frac{\eta_n^{N_h^k}}{\widehat{N}_h^{\epo,  m^n-1}(s,a) \vee 1}\right)\left(P_{h,s,a} - P_h^{k^i}\right) \overline{V}^{ \nnext,k^i}_{h+1}\\
& = \sum_{i=1}^{N_h^k} \left(\sum_{n= N_h^{(m^i+1,1)}+1}^{N_h^{(m^i+2,1)} \wedge N_h^{k}} \frac{\eta_n^{N_h^k}}{\widehat{N}_h^{\epo, m^i} \vee 1}\right)\left(P_{h,s,a} - P_h^{k^i}\right) \overline{V}^{ \nnext,k^i}_{h+1},
\end{align*}
where (i) follows from the fact that $  N_h^{(m^n,1)} - N_h^{(m^n-1,1)} = \widehat{N}_h^{\epo, m^n-1}(s,a)$, and (ii) is obtained by rearranging terms with respect to  $i$ (the terms with respect to $\overline{V}^{ \nnext,k^i}_{h+1}$ will only be added during the epoch $m^i+1$), and the last equality holds since $m^n -1 = m^i$ for all $n= N_h^{(m^i+1,1)}+1, N_h^{(m^i+1,1)}+2, N_h^{(m^i+2,1)}$. 

With the above relation in mind, we are ready to invoke Lemma~\ref{lemma:martingale-union-all} to control $U_2$. 
To continue,  for any episode $j\leq k$, let us denote by $m(j)$ the index of the epoch in which episode $j$ happens (with slight abuse of notation). Let us set 
\begin{align*}
  W_{h+1}^j \coloneqq \overline{V}_{h+1}^{\nnext,j}, \qquad \text{and} \qquad u_h^j(s,a,N) \coloneqq \sum_{n=N_h^{(m(j)+1,1)}+1}^{N_h^{(m(j)+2,1)} \wedge N} \frac{\eta_n^N}{\widehat{N}_h^{\epo, m(j)}(s,a) \vee 1}. 
\end{align*}
As a result, we see that
\begin{align*}
  \|W_{h+1}^{j}\|_\infty \leq \|\overline{V}^{\nnext,j}_{h+1}\|_\infty \leq H \eqqcolon C_{\mathrm{w}}
\end{align*}
and the following fact (which will be established in Appendix~\ref{sec:proof:equ:stepsize-bound})
\begin{align}
  0 \leq u_h^j(s,a,N) = \sum_{n=N_h^{(m(j)+1,1)}+1}^{N_h^{(m(j)+2,1)} \wedge N} \frac{\eta_n^N}{\widehat{N}_h^{\epo, m(j)}(s,a) \vee 1} \leq \frac{64 e^2 \iota }{N \vee 1} \eqqcolon C_u\label{equ:stepsize-bound}
\end{align}
holds for all $(j,h,s,a) \in [K]\times [H]\times \cS\times \cA$ with probability at least $1-\delta$.

Given that $N = N_h^k(s,a) = N_h^k$, applying Lemma~\ref{lemma:martingale-union-all} with the above quantities, we can show that for any state-action pair $(s,a)\in\cS\times \cA$,
\begin{align}
  \left|U_2\right| &= \left|\sum_{i=1}^{N_h^k} \left(\sum_{n= N_h^{(m^i+1,1)}+1}^{N_h^{(m^i+2,1)} \wedge N_h^{k}} \frac{\eta_n^{N_h^k}}{\widehat{N}_h^{\epo,m^i} \vee 1}\right)\left(P_{h,s,a} - P_h^{k^i}\right) \overline{V}^{ \nnext,k^i}_{h+1}\right| = \left|\sum_{j=1}^k X_j\left(s,a, h, N_h^k\right)\right| \nonumber \\
  &\lesssim \sqrt{C_{\mathrm{u}} \log^2\frac{SAT}{\delta}}\sqrt{\sum_{i = 1}^{N_h^k(s,a)} u_h^{k_h^i(s,a)}(s,a,N) \Var_{h, s,a} \big(W_{h+1}^{k_h^i(s,a)}  \big)} + \left(C_{\mathrm{u}} C_{\mathrm{w}} + \sqrt{\frac{C_{\mathrm{u}}}{N \vee 1}} C_{\mathrm{w}}\right) \log^2\frac{SAT}{\delta} \nonumber\\
  &  \lesssim \sqrt{\frac{\iota^3}{N_h^k \vee 1}}\sqrt{\frac{1}{N_h^k \vee 1}\sum_{i= 1}^{N_h^k} \Var_{h, s,a}\big(\overline{V}^{ \nnext ,k^i}_{h+1}\big)} + \frac{H \iota^3 }{N_h^k \vee 1}\nonumber\\
  &  \lesssim \sqrt{\frac{\iota^3}{N_h^k \vee 1}}\sqrt{\sigma^{\re , k^{N_h^k}+1}_h(s,a) - \big(\mu^{\re , k^{N_h^k}+1}_h(s,a) \big)^2} + \frac{H \iota^3}{(N_h^k \vee 1)^{3/4}} .
  \label{lemma1:equ5}
\end{align}
To streamline the presentation of the analysis, we shall postpone the proof of \eqref{lemma1:equ5} to Appendix~\ref{sec:proof:eq:var-Vref}.

\paragraph{Step 3: summing up.} Combining the bounds in \eqref{lemma1:equ4} and \eqref{lemma1:equ5} yields that: for any state-action pair $(s,a)\in\cS\times \cA$,
\begin{align}
  \bigg| \sum_{n = 1}^{N_h^k(s,a)} \eta^{N_h^k(s,a)}_n \xi_h^{k^n}\bigg|  &\leq |U_1| +|U_2| \nonumber\\
  & \lesssim~\sqrt{\frac{H\iota^2}{N_h^k \vee 1}}\sqrt{\sigma^{\adv, k^{N_h^k}+1}_h(s,a) - \big(\mu^{\adv, k^{N_h^k}+1}_h(s,a) \big)^2} \nonumber \\
  &\quad + \sqrt{\frac{\iota^3}{N_h^k \vee 1}}\sqrt{\sigma^{\re , k^{N_h^k}+1}_h(s,a) - \big(\mu^{\re , k^{N_h^k}+1}_h(s,a) \big)^2} + \cb \frac{H^{7/4}\iota^2}{(N_h^k \vee 1)^{3/4}} + \cb\frac{H^2\iota^2}{N_h^k \vee 1}  \nonumber\\
  & \leq \overline{\sumb}^{k^{N_h^k}+1}_h(s,a) + \cb \frac{H^{7/4}\iota^2}{(N_h^k \vee 1)^{3/4}} + \cb\frac{H^2\iota^2}{N_h^k \vee 1} 
  \label{lemma1:equ6}
\end{align}
holds for some sufficiently large constant $\cb>0$, 
where the last line follows from the definition of $\overline{\sumb}^{k^{N_h^k}+1}_h(s,a)$ in line~\ref{eq:line-number-19} of Algorithm~\ref{algo:subroutine}.
As a consequence of the inequality~\eqref{lemma1:equ6}, for any $(s,a)\in\cS \times \cA$, one has 
\begin{align*}
   \bigg| \sum_{n = 1}^{N_h^k} \eta^{N_h^k}_n \xi_h^{k^n} \bigg| \le \overline{\sumb}^{k^{N_h^k}+1}_h(s,a) + \cb \frac{H^{7/4}\iota^2}{(N_h^k \vee 1)^{3/4}} + \cb\frac{H^2\iota^2}{N_h^k \vee 1} 
  \leq \sum_{n = 1}^{N_h^k} \eta^{N_h^k}_n \overline{b}^{k^n+1}_h,
\end{align*}
where the last inequality holds due to \eqref{lemma1:equ10}.
We have thus concluded the proof of Lemma~\ref{lem:lcb-adv-bonus-upper}. 

\subsubsection{Proof of inequality~\eqref{lemma1:equ4}}
\label{sec:proof:lemma1:equ4} 

To establish the inequality~\eqref{lemma1:equ4}, it is sufficient to consider the difference  
$$W_1 := \sum_{n = 1}^{N_h^k} \eta^{N_h^k}_n\Var_{h, s,a}(V^{k^n}_{h+1} - \overline{V}^{k^n}_{h+1}) - {\sigma}^{\adv, k^{N_h^k}+1}_h(s,a) + ({\mu}^{\adv, k^{N_h^k}+1}_h(s,a))^2. $$
Before continuing, it is easily verified that if $N_h^k= N_h^k(s,a) = 0$, the basic fact $\sum_{n = 1}^{{N_h^k}} \eta_n^{N_h^k} = 0$ leads to $W_1= 0$, and therefore, \eqref{lemma1:equ4} holds directly. The remainder of the proof is thus dedicated to controlling $W_1$ when $N_h^k= N_h^k(s,a) \geq 1$.
Recalling the definition in \eqref{lemma1:equ2}
\begin{equation}
\Var_{h, s,a}(V_{h+1}^{k^n} - \overline{V}_{h+1}^{k^n}) = P_{h, s,a}(V_{h+1}^{k^n} - \overline{V}_{h+1}^{k^n})^{2} - \left(P_{h, s,a}(V_{h+1}^{k^n} - \overline{V}_{h+1}^{k^n}) \right)^2, \label{eq:variance-nonnegative}
\end{equation}
we can take this result together with \eqref{eq:recursion_mu_sigma_adv} to yield
\begin{align}
W_1 & = \sum_{n = 1}^{N_h^k} \eta^{N_h^k}_n P_{h, s,a}(V_{h+1}^{k^n} - \overline{V}_{h+1}^{k^n})^{2} - \sum_{n = 1}^{N_h^k} \eta^{N_h^k}_n \left(P_{h, s,a}(V_{h+1}^{k^n} - \overline{V}_{h+1}^{k^n}) \right)^2  \nonumber \\
&\qquad \qquad+ \Big(\sum_{n = 1}^{{N_h^k}} \eta_n^{N_h^k} P^{k^n}_{h}(V^{k^n}_{h+1} - \overline{V}^{k^n}_{h+1})\Big)^2  - \sum_{n = 1}^{{N_h^k}} \eta_n^{N_h^k} P_h^{k^n}(V_{h+1}^{k^n} - \overline{V}_{h+1}^{k^n})^2 \nonumber\\
& \le  \underbrace{ \Bigg|\sum_{n = 1}^{{N_h^k}} \eta_n^{N_h^k}(P^{k^n}_{h}-P_{h, s,a})(V^{k^n}_{h+1} - \overline{V}^{ k^n}_{h+1})^{ 2} \Bigg| }_{=: W_{1}^1}  \nonumber \\
&\qquad \qquad+ \underbrace{ \Big(\sum_{n = 1}^{{N_h^k}} \eta_n^{N_h^k} P^{k^n}_{h}(V^{k^n}_{h+1} - \overline{V}^{ k^n}_{h+1})\Big)^2  - \sum_{n = 1}^{{N_h^k}} \eta_n^{N_h^k} \Big(P_{h, s,a}(V^{k^n}_{h+1} - \overline{V}^{ k^n}_{h+1})\Big)^2}_{=: W_{1}^2}. \label{equ:lemma4 vr 2}
\end{align}
It then boils down to control the above two terms in \eqref{equ:lemma4 vr 2} separately when $N_h^k = N_h^k(s,a) \geq 1$.

\paragraph{Step 1: controlling $W_1^1$.} 
To control $W_1^1$, we shall invoke Lemma~\ref{lemma:martingale-union-all} by setting 
    \begin{align*} 
      W_{h+1}^{i} \coloneqq (V_{h+1}^i - \overline{V}_{h+1}^i)^2, \qquad \text{and} \qquad u_h^i(s,a,N) \coloneqq \eta_{N_h^i(s,a)}^N \geq 0,
    \end{align*}
  which obey
  \begin{align*}
    \|W_{h+1}^{i}\|_\infty \leq \|\overline{V}_{h+1}^i \|_\infty^2 + \|V_{h+1}^i\|_\infty^2 \leq 2H^2 \eqqcolon C_\mathrm{w}.  
  \end{align*}
Invoking the facts in \eqref{U1-cu} and \eqref{eq:sum-uh-k-s-a-123}, we arrive at
\begin{align*}
  \frac{2H}{N \vee 1} &\eqqcolon C_{\mathrm{u}}
\end{align*}
and
\begin{align*}
  0 &\leq \sum_{n=1}^N u_h^{k_h^n(s,a)}(s,a, N) \leq 1, \qquad \forall (N,s,a) \in [K] \times \cS\times \cA.
\end{align*}

Therefore, choosing $N = N_h^k(s,a) = N_h^k$ for any $(s,a)$ and applying Lemma~\ref{lemma:martingale-union-all} with the above quantities, we arrive at,  with probability at least $1-\delta$,
\begin{align}
  |W_1^1| &= \Bigg|\sum_{n = 1}^{{N_h^k}} \eta_n^{N_h^k}(P^{k^n}_{h}-P_{h, s,a})(V^{k^n}_{h+1} - \overline{V}^{ k^n}_{h+1})^{ 2} \Bigg| = \left|\sum_{i=1}^k X_i\left(s,a, h, N_h^k\right)\right| \nonumber\\
  &\lesssim \sqrt{C_{\mathrm{u}} \log^2\frac{SAT}{\delta}}\sqrt{\sum_{n = 1}^{N_h^k} u_h^{k^n}(s,a, N_h^k) \Var_{h,s,a} \big(W_{h+1}^{k^n}  \big)} + \left(C_{\mathrm{u}} C_{\mathrm{w}} + \sqrt{\frac{C_{\mathrm{u}}}{N \vee 1}} C_{\mathrm{w}}\right) \log^2\frac{SAT}{\delta} \nonumber \\
  & \lesssim \sqrt{\frac{H}{N_h^k \vee 1}\iota^2}\sqrt{\sum_{n = 1}^{N_h^k} \eta^{N_h^k}_n \|W_{h+1}^{k^n}\|_\infty^2 } + \frac{{H^3}\iota^2}{N_h^k \vee 1}   \lesssim \sqrt{\frac{H^5}{N_h^k \vee 1}\iota^2} + \frac{H^3\iota^2}{N_h^k \vee 1}.\label{eq:W11-result}
\end{align}

\paragraph{Step 2: controlling $W_1^2$.} Observe that Jensen's inequality gives 
\begin{align}
\Big(\sum_{n = 1}^{{N_h^k}} \eta_n^{N_h^k} P_{h, s,a}(V^{k^n}_{h+1} - \overline{V}^{ k^n}_{h+1})\Big)^2  &\leq \sum_{n = 1}^{{N_h^k}} \eta_n^{N_h^k} \Big(P_{h, s,a}(V^{k^n}_{h+1} - \overline{V}^{ k^n}_{h+1})\Big)^2, \label{eq:extra-result1}
\end{align}
due to the fact $\sum_{n = 1}^{N_h^k} \eta_n^{N_h^k} = 1$ (see \eqref{eq:sum-eta-n-N} and \eqref{equ:learning rate notation}).
Plugging the above relation into \eqref{equ:lemma4 vr 2} gives 
\begin{align}
W_1^2 &\le\Big(\sum_{n = 1}^{{N_h^k}} \eta_n^{N_h^k} P^{k^n}_{h}(V^{k^n}_{h+1} - \overline{V}^{ k^n}_{h+1})\Big)^2 - \Big(\sum_{n = 1}^{{N_h^k}} \eta_n^{N_h^k} P_{h, s,a}(V^{k^n}_{h+1} - \overline{V}^{ k^n}_{h+1})\Big)^2 \nonumber\\
& = \Big(\sum_{n = 1}^{{N_h^k}} \eta_n^{N_h^k} (P^{k^n}_{h}-P_{h, s,a}) (V^{k^n}_{h+1} - \overline{V}^{ k^n}_{h+1})\Big)  \cdot \Big(\sum_{n = 1}^{{N_h^k}} \eta_n^{N_h^k} (P^{k^n}_{h}+P_{h, s,a}) (V^{k^n}_{h+1} - \overline{V}^{ k^n}_{h+1})\Big)  . \label{equ:lemma4vr4}
\end{align}
Note that the first term in \eqref{equ:lemma4vr4} is exactly $|U_1|$ defined in \eqref{eq:defn-U1-1}, which can be controlled by invoking \eqref{lemma1:equ4-sub} to achieve that, with probability at least $1-\delta$,
\begin{align}
& \Big| \sum_{n = 1}^{{N_h^k}} \eta_n^{N_h^k} (P^{k^n}_{h}-P_{h, s,a}) (V^{k^n}_{h+1} - \overline{V}^{ k^n}_{h+1})\Big| \nonumber \\
  & \lesssim \sqrt{\frac{H}{N_h^k \vee 1}\iota^2}\sqrt{\sum_{n = 1}^{N_h^k} \eta^{N_h^k}_n \Var_{h, s,a} \big(V^{k^n}_{h+1} - \overline{V}^{k^n}_{h+1} \big) } + \frac{H^2\iota^2}{N_h^k \vee 1}  \lesssim  \sqrt{\frac{H^3 \iota^2}{N_h^k \vee 1}} + \frac{H^2}{N_h^k \vee 1}\iota^2,
\end{align} 
where the final inequality holds since $\Var_{h, s,a} \big(V^{k^n}_{h+1} - \overline{V}^{k^n}_{h+1} \big)\lesssim H^2$ and the fact in \eqref{eq:sum-eta-n-N}. In addition, the second term in \eqref{equ:lemma4vr4} can be controlled straightforwardly by
\begin{equation*}
\Big|\sum_{n = 1}^{{N_h^k}} \eta_n^{N_h^k} \left(P^{k^n}_{h}+P_{h, s,a}\right) \left(V^{k^n}_{h+1} - \overline{V}^{ k^n}_{h+1}\right)\Big| \le \sum_{n = 1}^{{N_h^k}} \eta_n^{N_h^k} \left( \big\|P^{k^n}_{h}\big\|_1 +\big\|P_{h, s,a}\big\|_1 \right) \big\|V^{k^n}_{h+1} - \overline{V}^{ k^n}_{h+1} \big\|_{\infty} \leq 2H,
\end{equation*}
where we have used the fact in \eqref{eq:sum-eta-n-N}, $\big\|V^{k^n}_{h+1} -\overline{V}^{k^n}_{h+1}\big\|_{\infty}\leq H$ and $\big\|P^{k^n}_{h}\big\|_1 =\big\| P_{h, s,a}\big\|_1=1$.

Taking the above two facts collectively with \eqref{equ:lemma4vr4} yields
\begin{equation} \label{eq:bound_I32}
W_1^2 \lesssim  \sqrt{\frac{H^5 \iota^2}{N_h^k \vee 1}} + \frac{H^3 \iota^2}{N_h^k \vee 1}.
\end{equation}

\paragraph{Step 3: summing up.}
Plugging the results in \eqref{eq:W11-result} and \eqref{eq:bound_I32} back into \eqref{equ:lemma4 vr 2}, we have
\begin{align*} 
W_1 \leq W_1^1 + W_1^2 \lesssim \sqrt{\frac{H^5 \iota^2}{N_h^k \vee 1}} + \frac{H^3 \iota^2}{N_h^k \vee 1},
\end{align*}
which leads to the desired result \eqref{lemma1:equ4} directly.

\subsubsection{Proof of inequality~\eqref{equ:stepsize-bound}}\label{sec:proof:equ:stepsize-bound} 

To begin with, let us recall two pieces of notation that shall be used throughout this proof: 
\begin{enumerate}
  \item $m(j)$: the index of the epoch in which the $j$-th episode occurs.  
  \item $\widehat{N}_h^{\epo, m}(s,a)$: the value of  $\widehat{N}_h^{(m, L_m+1)}(s,a)$, representing the number of visits to $(s,a)$ in the entire $m$-th epoch with length $L_m = 2^m$.
\end{enumerate}
Applying \eqref{equ:binomial} and taking the union bound over $\left( m(j), h, s,a\right) \in [M]\times[H]\times\cS\times\cA$ yield
\begin{align}
  \widehat{N}_h^{\epo, m(j)}(s,a) \vee 1 
  \geq \frac{2^{m(j)}  d_h^{\mu}(s,a)}{8\log\left(\frac{SAT}{\delta}\right)} \label{eq:N-hat-mj}
\end{align}
 with probability at least $1-\delta/2$.

For any epoch $m$, if we denote by $k_{\mathrm{last}}(m)$  the index of the last episode  in the $m$-th epoch, we can immediately see that
\begin{align}
  k_{\mathrm{last}}(m) = \sum_{i=1}^m L_i = \sum_{i=1}^m 2^i = 2^{m+1}-2 \leq 2^{m+1}. \label{eq:k-last-def}
\end{align}
Applying \eqref{equ:binomial} again and taking the union bound over $\left( m(j), h, s,a\right) \in [M]\times[H]\times\cS\times\cA$, 
one can guarantee that for every $n \in [N_h^{(m(j)+1,1)}, N_h^{(m(j)+2,1)}]$,  with probability at least $1-\delta/2$,
\begin{align}
N_h^{(m(j)+1,1)} \leq n \leq N_h^{(m(j)+2,1)} &= N_h^{k_{\mathrm{last}}(m(j)+1)} \nonumber\\ 
&\leq N_h^{2^{m(j)+2}}  
\leq \begin{cases}
 e^2 2^{m(j)+2} d_h^{\mu}(s,a)  & \text{ if } 2^{m(j)+2} d_h^{\mu}(s,a) \geq \log\left(\frac{SAT}{\delta}\right)\\ 
2e^2 \log\left(\frac{SAT}{\delta}\right)& \text{ if } 2^{m(j)+2} d_h^{\mu}(s,a) \leq 2\log\left(\frac{SAT}{\delta}\right)
\end{cases} .\label{eq:N-mj-plus-2}
\end{align}

Combine the above results to yield
\begin{align}
  \begin{cases} \widehat{N}_h^{\epo, m(j)}(s,a) \vee 1 \overset{\mathrm{(i)}}{\geq} \frac{2^{m(j)}  d_h^{\mu}(s,a)}{8\log\left(\frac{SAT}{\delta}\right)}  \overset{\mathrm{(ii)}}{\geq} \frac{1}{32 e^2 \log\left(\frac{SAT}{\delta}\right)} n, & \text{ if } 2^{m(j)+2} \cdot d_h^{\mu}(s,a) \geq \log\left(\frac{SAT}{\delta}\right),\\
  \widehat{N}_h^{\epo,m(j)}(s,a) \vee 1  \geq 1 \overset{\mathrm{(iii)}}{\geq} \frac{1}{2e^2 \log\left(\frac{SAT}{\delta}\right)} n & \text{ if } 2^{m(j)+2}\cdot d_h^{\mu}(s,a) \leq 2\log\left(\frac{SAT}{\delta}\right),
  \end{cases}
\end{align}
where (i) follows from \eqref{eq:N-hat-mj}, (ii) and (iii) hold due to \eqref{eq:N-mj-plus-2}.
As a result, we arrive at
\begin{align*}
  \sum_{n=N_h^{(m(j)+1,1)}+1}^{N_h^{(m(j)+2,1)} \wedge N} \frac{\eta_n^N}{\widehat{N}_h^{\epo,m(j)}(s,a) \vee 1} &\leq \sum_{n=N_h^{(m(j)+1,1)}+1}^{N_h^{(m(j)+2,1)} \wedge N} \frac{32e^2 \log\left(\frac{SAT}{\delta}\right) \eta_n^N}{n} \\
  &\leq \sum_{n=N_h^{(m(j)+1,1)}+1}^{N} \frac{32e^2 \log\left(\frac{SAT}{\delta}\right) \eta_n^N}{n} \leq \frac{64 e^2 \log\left(\frac{SAT}{\delta}\right) }{N \vee 1},
\end{align*}
where the last inequality holds since $\sum_{i = 1}^{N} \frac{\eta^{N}_i}{i} \le \frac{2}{N \vee 1}$ (see Lemma~\ref{lemma:property of learning rate}).

\subsubsection{Proof of inequality~\eqref{lemma1:equ5}}\label{sec:proof:eq:var-Vref} 
In this subsection, we intend to control the following term
$$W_2 :=  \frac{1}{N_h^k \vee 1}\sum_{n = 1}^{{N_h^k}} \Var_{h, s,a}\left(\overline{V}^{\nnext, k^n}_{h+1}\right) - \left( {\sigma}^{\re, k^{N_h^k}+1}_h(s,a) -\left({\mu}^{\re, k^{N_h^k}+1}_h(s,a) \right)^2 \right) $$
for all $(s,a) \in \cS\times \cA$. 
First, it is easily seen that if $N_h^k= 0$, then we have $W_2 =0$ and thus \eqref{lemma1:equ5} is satisfied. Therefore, the remainder of the proof is devoted to verifying \eqref{lemma1:equ5}  when $N_h^k=N_h^k(s,a) \geq 1$. 

Combining the expression \eqref{eq:recursion_mu_sigma_ref} with the following definition
\begin{equation*}
\Var_{h, s,a}\left(\overline{V}_{h+1}^{\nnext, k^n} \right) = P_{h, s,a}\left(\overline{V}_{h+1}^{\nnext, k^n}\right)^{2} - \left(P_{h, s,a}\overline{V}_{h+1}^{\nnext, k^n}\right)^2,
\end{equation*} 
we arrive at
\begin{align}
W_2 &=  \frac{1}{{N_h^k} \vee 1}\sum_{n=1}^{N_h^k} \left(  P_{h, s,a}\left(\overline{V}_{h+1}^{\nnext, k^n}\right)^{2} - \left(P_{h, s,a}\overline{V}_{h+1}^{\nnext, k^n}\right)^2 \right) \nonumber \\  
& \qquad\qquad - \frac{1}{N_h^k \vee 1}\sum_{n = 1}^{{N_h^k}} P_h^{k^n}\left(\overline{V}_{h+1}^{\nnext, k^n}\right)^2 + \left(\frac{1}{N_h^k \vee 1}\sum_{n=1}^{N_h^k} P_h^{k^n} \overline{V}_{h+1}^{\nnext, k^n}\right)^2 \nonumber \\
& =  \underbrace{ \frac{1}{N_h^k \vee 1}\sum_{n = 1}^{{N_h^k}}\left(P_{h, s,a} - P^{k^n}_{h} \right)\left(\overline{V}^{\nnext, k^n}_{h+1}\right)^2 }_{=: W_2^1} + \underbrace{ \left(\frac{1}{N_h^k \vee 1}\sum_{n=1}^{N_h^k} P_h^{k^n} \overline{V}_{h+1}^{\nnext, k^n}\right)^2 - \frac{1}{N_h^k \vee 1}\sum_{n = 1}^{{N_h^k}} \left(P_{h, s,a}\overline{V}^{\nnext, k^n}_{h+1}\right)^2}_{=:W_2^2} . \label{equ:lemma4 vr 6}
\end{align}
In the sequel, we intend to control the terms in \eqref{equ:lemma4 vr 6} separately. 

\paragraph{Step 1: controlling $W_2^1$.}
The first term $W_2^1$ can be controlled by invoking Lemma~\ref{lemma:martingale-union-all} and set
\begin{align*}
  W_{h+1}^i \coloneqq \left(\overline{V}_{h+1}^{\nnext,i}\right)^2, \qquad \text{and} \qquad u_h^i(s,a,N) \coloneqq \frac{1}{N } \eqqcolon C_u. 
\end{align*}
To proceeding, with the fact
\begin{align*}
  \left\|W_{h+1}^{i}\right\|_\infty \leq \left\|\overline{V}^{\nnext,i}_{h+1}\right\|_\infty^2 \leq H^2 \eqqcolon C_{\mathrm{w}} 
\end{align*}
and $N = N_h^k(s,a) = N_h^k$, applying Lemma~\ref{lemma:martingale-union-all} with the above quantities, we have for all state-action pair $(s,a)\in\cS\times \cA$,
\begin{align}
  \left|W_2^1\right| &= \left|\frac{1}{N_h^k}\sum_{n = 1}^{{N_h^k}}\left(P_{h, s,a} - P^{k^n}_{h} \right)\left(\overline{V}^{\nnext, k^n}_{h+1}\right)^2 \right| = \left|\sum_{i=1}^k X_i\left(s,a, h, N_h^k\right)\right| \nonumber \\
  &\lesssim \sqrt{C_{\mathrm{u}} \log^2\frac{SAT}{\delta}}\sqrt{\sum_{n = 1}^{N_h^k(s,a)} u_h^{k_h^n(s,a)}(s,a,N) \Var_{h, s,a} \left(W_{h+1}^{k_h^n(s,a)}  \right)} + \left(C_{\mathrm{u}} C_{\mathrm{w}} + \sqrt{\frac{C_{\mathrm{u}}}{N }} C_{\mathrm{w}}\right) \log^2\frac{SAT}{\delta} \nonumber\\
  &  \lesssim \sqrt{\frac{\iota^2}{N_h^k }}\sqrt{\|W_{h+1}^{i}\|_\infty^2} + \frac{H^2 \iota^2 }{N_h^k }  \lesssim \sqrt{\frac{H^4 \iota^2 }{N_h^k}} + \frac{H^2\iota^2}{N_h^k}. \label{eq:W21-final-result}
\end{align}

\paragraph{Step 2: controlling $W_2^2$.}
Towards controlling $W_2^2$ in \eqref{equ:lemma4 vr 6}, we observe that by Jensen's inequality,
\begin{align*}
  \Big(\frac{1}{{N_h^k }}\sum_{n=1}^{N_h^k} P_{h, s,a} \overline{V}_{h+1}^{\nnext, k^n}\Big)^2  
&\leq \frac{1}{N_h^k }\sum_{n = 1}^{{N_h^k}} \left(P_{h, s,a}\overline{V}^{\nnext,k^n}_{h+1}\right)^2.
\end{align*}
Equipped with this relation, $W_2^2$ satisfies 
\begin{align}
W_2^2&\leq \left(\frac{1}{N_h^k }\sum_{n=1}^{N_h^k} P_h^{k^n} \overline{V}_{h+1}^{\nnext,k^n}\right)^2  - \left(\frac{1}{N_h^k }\sum_{n=1}^{N_h^k} P_{h, s,a} \overline{V}_{h+1}^{\nnext,k^n}\right)^2 \nonumber \\
&= \Big(\frac{1}{N_h^k }\sum_{n=1}^{N_h^k} \left( P_h^{k^n} -P_{h, s,a}\right)  \overline{V}_{h+1}^{\nnext,k^n}\Big) \cdot \Big(\frac{1}{N_h^k}\sum_{n=1}^{N_h^k} \left( P_h^{k^n} +P_{h, s,a}\right)  \overline{V}_{h+1}^{\nnext,k^n}\Big). \label{equ:lemma4vr5}
\end{align}

As for the first term in \eqref{equ:lemma4vr5}, let us set 
\begin{align*}
  W_{h+1}^i \coloneqq \overline{V}_{h+1}^{\nnext,i}, \qquad \text{and} \qquad u_h^i(s,a,N) \coloneqq \frac{1}{N } \eqqcolon C_u, 
\end{align*}
which satisfy 
\begin{align*}
  \big\|W_{h+1}^{i}\big\|_\infty \leq \big\|\overline{V}^{\nnext,i}_{h+1}\big\|_\infty \leq H \eqqcolon C_{\mathrm{w}}.
\end{align*}
For any $(s,a)$, Lemma~\ref{lemma:martingale-union-all} together with the above quantities and $N = N_h^k = N_h^k(s,a)$ gives 
\begin{align*}
  &\left|\frac{1}{N_h^k }\sum_{n=1}^{N_h^k} \left( P_h^{k^n} -P_{h, s,a}\right)  \overline{V}_{h+1}^{\nnext,k^n}\right| \\
  &\lesssim \sqrt{C_{\mathrm{u}} \log^2\frac{SAT}{\delta}}\sqrt{\sum_{n = 1}^{N_h^k(s,a)} u_h^{k_h^n(s,a)}(s,a,N) \Var_{h, s,a} \big(W_{h+1}^{k_h^n(s,a)}  \big)} + \left(C_{\mathrm{u}} C_{\mathrm{w}} + \sqrt{\frac{C_{\mathrm{u}}}{N }} C_{\mathrm{w}}\right) \log^2\frac{SAT}{\delta}\\
  &  \lesssim \sqrt{\frac{\iota^2}{N_h^k }}\sqrt{\left\|W_{h+1}^{k_h^n(s,a)}\right\|_\infty^2} + \frac{H \iota^2 }{N_h^k }  \lesssim \sqrt{\frac{H^2 \iota^2 }{N_h^k}} + \frac{H\iota^2}{N_h^k}
\end{align*}
with probability at least $1-\delta$. 
In addition, the second term can be bounded straightforwardly by
\begin{equation*}
\Big|\frac{1}{N_h^k }\sum_{n=1}^{N_h^k} \big( P_h^{k^n} +P_{h, s,a}\big)  \overline{V}_{h+1}^{\nnext,k^n}\Big| \le \frac{1}{N_h^k}\sum_{n = 1}^{{N_h^k}}  \left( \big\|P^{k^n}_{h}\big\|_1 +\big\|P_{h, s,a}\big\|_1 \right) \big\|\overline{V}_{h+1}^{\nnext,k^n}\big\|_{\infty} \leq 2H,
\end{equation*}
where the last inequality is valid since $\big\|\overline{V}_{h+1}^{\nnext,k^n}\big\|_{\infty}\leq H$ and $\big\|P^{k^n}_{h}\big\|_1 =\big\| P_{h, s,a}\big\|_1=1$. Substitution of the above two observations back into \eqref{equ:lemma4vr5} yields
\begin{equation} 
\label{eq:bound_I42}
W_2^2 \lesssim   \sqrt{\frac{H^4}{N_h^k \vee 1}\iota^2} + \frac{H^2}{N_h^k \vee 1} \iota^2.  
\end{equation}

\paragraph{Step 3: combining the above results.} Plugging the results in \eqref{eq:W21-final-result} and \eqref{eq:bound_I42} into \eqref{equ:lemma4 vr 6}, we reach
\begin{align*} 
W_2 \leq W_2^1 + W_2^2 \lesssim  \sqrt{\frac{H^4}{N_h^k \vee 1}\iota^2} + \frac{H^2}{N_h^k \vee 1}\iota^2,
\end{align*}
thus establishing the desired inequality~\eqref{lemma1:equ5}.

\bibliography{bibfileRL}
\bibliographystyle{apalike} 

\end{document}